\documentclass{article}

\usepackage[preprint]{neurips_2026}

\usepackage[utf8]{inputenc}
\usepackage[T1]{fontenc}
\usepackage[table]{xcolor}
\usepackage[colorlinks=true,allcolors=blue,bookmarks=true]{hyperref}
\usepackage{url}
\usepackage{booktabs}
\usepackage{amsfonts}
\usepackage{amsmath}
\usepackage{amssymb}
\usepackage{amsthm}
\usepackage{microtype}
\usepackage{graphicx}
\usepackage{xspace}
\usepackage[nameinlink,capitalize]{cleveref}
\usepackage{enumitem}
\usepackage{placeins}
\usepackage{aliascnt}

\crefname{section}{Sec.}{Secs.}
\Crefname{section}{Sec.}{Secs.}
\crefname{subsection}{Sec.}{Secs.}
\Crefname{subsection}{Sec.}{Secs.}
\crefname{figure}{Fig.}{Figs.}
\Crefname{figure}{Fig.}{Figs.}
\crefname{table}{Tab.}{Tabs.}
\Crefname{table}{Tab.}{Tabs.}
\crefname{equation}{Eq.}{Eqs.}
\Crefname{equation}{Eq.}{Eqs.}
\crefname{appendix}{App.}{Apps.}
\Crefname{appendix}{App.}{Apps.}
\crefname{lemma}{Lemma}{Lemmas}
\Crefname{lemma}{Lemma}{Lemmas}
\crefname{theorem}{Theorem}{Theorems}
\Crefname{theorem}{Theorem}{Theorems}
\crefname{proposition}{Proposition}{Propositions}
\Crefname{proposition}{Proposition}{Propositions}
\crefname{corollary}{Corollary}{Corollaries}
\Crefname{corollary}{Corollary}{Corollaries}
\crefname{definition}{Definition}{Definitions}
\Crefname{definition}{Definition}{Definitions}
\crefname{remark}{Remark}{Remarks}
\Crefname{remark}{Remark}{Remarks}

\newaliascnt{lemma}{theorem}
\newtheorem{lemma}[lemma]{Lemma}
\aliascntresetthe{lemma}
\newaliascnt{corollary}{theorem}
\newtheorem{corollary}[corollary]{Corollary}
\aliascntresetthe{corollary}
\newaliascnt{proposition}{theorem}
\newtheorem{proposition}[proposition]{Proposition}
\aliascntresetthe{proposition}
\newtheorem{definition}{Definition}
\newtheorem{remark}{Remark}

\newcommand{\X}{\mathcal{X}}
\newcommand{\Y}{\mathcal{Y}}
\newcommand{\Zx}{\mathcal{Z}_x}
\newcommand{\Zy}{\mathcal{Z}_y}
\newcommand{\R}{\mathbb{R}}
\newcommand{\E}{\mathbb{E}}
\newcommand{\Dtrain}{\mathcal{D}_{\mathrm{train}}}
\newcommand{\Dtest}{\mathcal{D}_{\mathrm{test}}}
\newcommand{\eps}{\varepsilon}
\newcommand{\norm}[1]{\left\|#1\right\|}
\newcommand{\ie}{i.e.,\xspace}
\newcommand{\eg}{e.g.,\xspace}

\title{Does Your Neural Network Extrapolate? \\
Feature Engineering as Identifiability Bias for OOD Generalization\thanks{Code and results: \url{https://github.com/leaguilar/DoesYourNeuralNetworkExtrapolate}, archived at \url{https://doi.org/10.5281/zenodo.24019953}.}}

\author{
Leonel Aguilar \\
\small Chair of Cognitive Science, ETH Zürich \\
\texttt{aleonel@ethz.ch} \\
\And
Jan Nagler \\
\small Centre for Human and Machine Intelligence, \\
Frankfurt School \\
\texttt{j.nagler@fs.de} \\
\And
Christoph Hoelscher \\
\small Chair of Cognitive Science, ETH Zürich \\
\texttt{choelsch@ethz.ch} \\
\And
Nino Antulov-Fantulin \\
\small Aisot Technologies AG, \\
\small D-GESS, ETH Zürich \\
\texttt{nino@aisot.ch} \\
}

\begin{document}

\maketitle

\begin{abstract}
Deep neural networks can learn salient features from an in-distribution (ID) training window without identifying the data-generating structure needed for out-of-distribution (OOD) generalization.
Understanding this limitation requires examining the identifiability of the data-generating process (DGP).
From a single training window, the DGP that yields the correct OOD extrapolation is not identifiable: infinitely many DGPs are observationally equivalent within a fixed tolerance on that window yet diverge arbitrarily beyond it, and no ID criterion alone can reliably distinguish them.
A structural commitment—the feature map, label map, and model class $(\varphi, \psi, \mathcal{M})$—specifies the assumed DGP and governs OOD generalization while leaving ID performance essentially unchanged. When architecture, pretraining, data augmentation, input formats, or domain knowledge implicitly supply the correct commitment, the model generalizes beyond the training window. When it cannot infer OOD-relevant structure from ID evidence, it fails to generalize.
When the commitment is correct {\em and} identifiable, OOD error vanishes. 
We demonstrate this mechanism in models trained from scratch and in foundation models across four natural-science settings (chemistry, physics, biology, and histopathology), as well as in a positional-encoding study spanning Transformer, Mamba, and S4D architectures.
Finally, a controlled study shows that correct features are necessary but not sufficient for OOD generalization.
\end{abstract}

\section{Introduction}
\label{sec:intro}

\begin{figure}[t]
  \centering
  \includegraphics[width=0.9\linewidth]{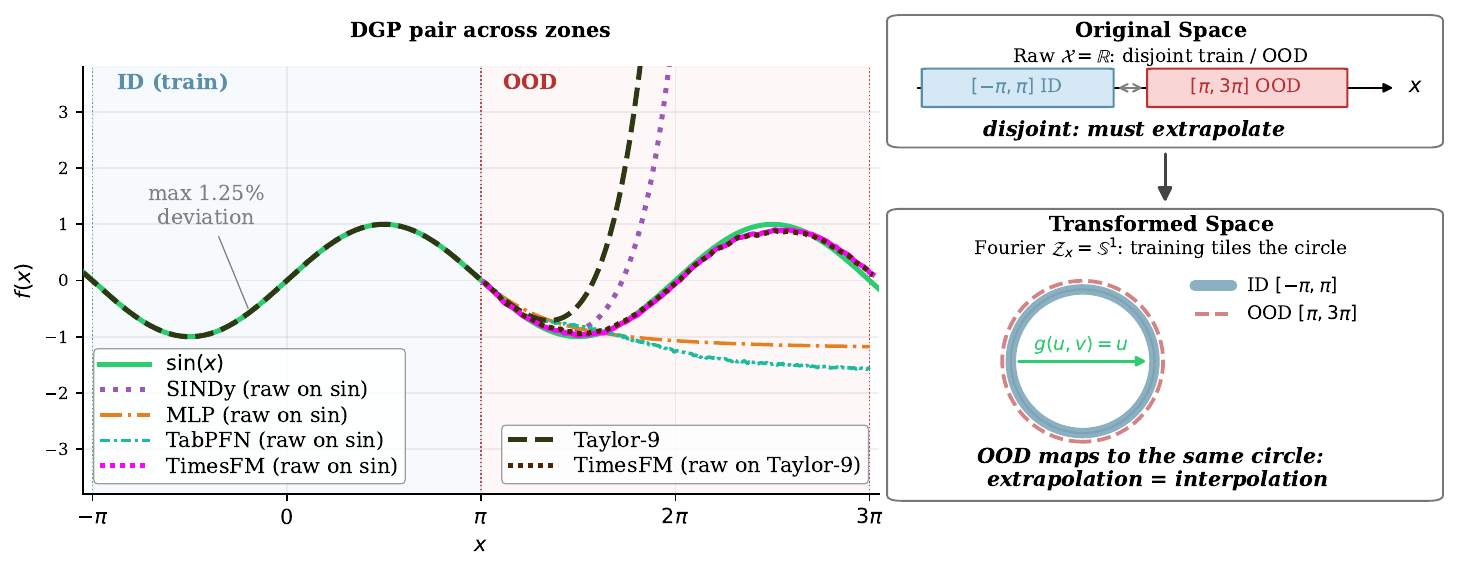}
  \caption{\emph{Left:} $\sin(x)$ (green) and its degree-9 Taylor polynomial (black) can be observationally equivalent on the training window (depending on noise and amount of samples) and diverge on $[\pi, 3\pi]$. Models (MLP, SINDy, TabPFN, TimesFM) trained/evaluated on the Original (raw) space without noise are shown. ``raw'' corresponds to ``None'' feature map in the tables (untransformed inputs). ``On sin'' and ``on Taylor-9'' identify the source of the training or forecasting-context data.
    \emph{Right:} in Fourier coordinates $\varphi(x) = (\sin x, \cos x)$, both (ID, OOD) map to the same circle $\mathbb{S}^1$, turning OOD extrapolation into interpolation.\label{fig:exp11_concept}}
\end{figure}

While the Universal Approximation Theorem guarantees that neural networks are exceptional interpolators within the support of their training data, they struggle with out-of-distribution extrapolation. Frameworks from Statistical Learning Theory provide guarantees for in-distribution generalization, however, achieving mathematically guaranteed extrapolation in deep learning remains an unsolved problem.
To understand what neural networks learn, we must ask: when can they extrapolate? We start with the good old $\sin(x)$ function.
A neural network, $5$-hidden-layer MLP with $256$ Tanh units, trained on raw inputs $x$ to predict $f(x) = \sin(x)$ on $[-\pi, \pi]$ (\Cref{fig:exp11_concept}) extrapolates as a smooth, eventually-levelling curve with high OOD relative error while an identical architecture trained on $\varphi(x) = (\sin x, \cos x)$ reaches OOD relative error below $1\%$ (\Cref{tab:exp11}). The two have indistinguishable training loss; they differ only in the feature map applied before the model. Two assumptions, periodic and polynomial-like (Taylor is a valid hypothesis on $[-\pi, \pi]$), can be observationally equivalent on the training window and diverge outside it.

\textbf{Indistinguishability:} 
For any finite training window $W$, infinitely many processes agree on $W$ to any fixed tolerance and disagree arbitrarily outside it. Throughout, \emph{identifiability} refers to the continuation of the data-generating process beyond the training support, not to the recovery of latent factors studied for deep generative models (\Cref{sec:related_work}).
The claim rests on two results. 1) When noise dominates, no in-distribution test separates the candidates (\Cref{prop:erm}). 2) When it does not, the multiplicity persists as the Rashomon set of near-equivalent hypotheses on $W$, and cross-validation still fails to resolve it (\Cref{sec:exp31}). The modeller's choice of representation space and model class breaks the tie.

\textbf{\textit{Correctness} is relative to the model class.}
$(\sin x, \cos x)$ is the correct $\varphi$ for $\sin(x)$ \emph{only because} the downstream class favours linear combinations: it renders the target linearly expressible. A Fourier-series regressor would collapse $\varphi$ to the identity. Correctness is a joint property of $(\varphi,\psi, \mathcal{M})$, which is why we treat them as a single commitment (\Cref{sec:theory}). Concrete failures of correct $\varphi$ paired with non-matching $\mathcal{M}$ appear in \Cref{sec:controlled}. Three conditions are jointly necessary: (a) a representation $(\varphi, \psi)$ aligned with the DGP, (b) a model class that expresses the target between the input/output representations, $\Zx$ and $\Zy$, and (c) coverage of $\Zx$ sufficient to identify the model. (a)+(b) constitute the DGP commitment; (c) is its evidential prerequisite (\Cref{sec:scope}). The headline comparisons hold (b) and (c) fixed, isolating (a).

\textbf{Contributions:}
(i) We formalise \textbf{observational equivalence} on a finite window (\Cref{def:obseq}), identify $(\varphi, \psi, \mathcal{M})$ as the DGP commitment, prove empirical risk minimization (ERM) on $W$ is structurally blind to the choice (\Cref{prop:erm}), and characterise the near-boundary data that breaks the blindness, both in finite samples through a single detectability scale (\Cref{prop:strip}) and geometrically at width $\delta_V \ge \sqrt{3}\,\eps_W/r$ (\Cref{cor:nearboundary}).
(ii) We isolate $\varphi$, $\psi$, and $\mathcal{M}$ in \textbf{controlled experiments} (\Cref{sec:controlled}) on minimal pairs and on pretrained foundation models (TimesFM~\citep{das2024decoder}, TabPFN~\citep{hollmann2025accurate}, governed by the same mechanism via input format).
(iii) We \textbf{validate cross-domain} (\Cref{sec:cross_domain}) on chemistry (mass-action kinetics), physics (Kepler exoplanets), biology (cross-species coding-DNA-sequence (CDS) detection on six organisms), and multi-hospital histopathology (Camelyon17), same prediction, all confirmed.
(iv) We deliver \textbf{three practitioner artefacts} (\Cref{sec:tools}): 1) near-boundary selection, 2) the SINDy (sparse identification of nonlinear dynamics)~\citep{brunton2016sindy} $\delta_{\mathrm{OOD}}$ diagnostic, and 3) a content/position-embedding (CE/PE) reframing verified architecture-invariant on Transformer~\citep{vaswani2017attention} / Mamba~\citep{gu2023mamba} / S4D~\citep{gu2022s4d}.
(v) We map \textbf{scope and limitations} on a periodic $d$-torus benchmark (\Cref{sec:scope}): correct $(\varphi, \psi, \mathcal{M})$ is necessary but not sufficient.
Extended tables are in the appendix.

\section{Related Work}
\label{sec:related_work}

Multi-environment domain generalisation, formalised by IRM~\citep{arjovsky2019irm} and catalogued by \citet{kaur2022modeling}, breaks ties through invariance penalties across training distributions. Our single-window setting offers no such auxiliary environments, leaving the right DGP non-identifiable from data alone, and the commitment $(\varphi, \psi, \mathcal{M})$ supplies the structural information that multi-environment access supplies elsewhere. The Rashomon set~\citep{breiman2001statistical, xin2022exploring} captures the same multiplicity at the model side: members agree on training data but disagree off-distribution, and $(\varphi, \psi)$ acts as a canonicaliser that selects a Rashomon element whose extension matches the DGP.
We interpret \emph{identifiability} as learning which of the alternative DGPs provides the correct continuation beyond the training support.
This is not the sense used for identifiable deep generative models~\citep{hyvarinen2016unsupervised, khemakhem2020variational, locatello2019challenging}, where a model fits the data, and the question is whether its latent factors are recoverable up to a stated equivalence class. The equivalence classes and the objects to be recovered differ, and we make no claim about latent-factor recovery.
The notion that some inductive bias is necessary comes from~\citep{mitchell1980need, wolpert1996lack, wolpert1997nofree}. No Free Lunch needs an unrestricted family and is dissolved by assumptions such as smoothness or sparsity, whereas \Cref{prop:erm} needs only two smooth, realisable candidates and yields only to data beyond $W$ (\Cref{app:nfl}), in line with \citet{sterkenburg2021nofree}.
Engression~\citep{shen2025engression} approaches extrapolation through a structural commitment to monotonicity and a pre-additive noise DGP.

A second body of work supplies the commitment implicitly via architecture or input format. PINNs~\citep{raissi2019physics} and Fourier Neural Operators~\citep{li2020fourier} bake in a differential operator (constraint on $\psi$) or a Fourier basis (a $\varphi$ that turns extrapolation into spectral interpolation), helping precisely when the assumed structure matches the DGP and degrading otherwise (\Cref{sec:scope}). Sinusoidal PE~\citep{vaswani2017attention}, RoPE~\citep{su2021roformer}, and ALiBi~\citep{press2022train} are partial periodic feature maps at the token-position level. \Cref{sec:exp33} evaluates state-of-the-art architectures: Mamba~\citep{gu2023mamba} and S4D~\citep{gu2022s4d}. On the activation side, \citet{ziyin2020neural} addresses the periodic-extrapolation failures by crafting an activation function with explicit periodic inductive bias, an $\mathcal{M}$-side commitment. In our controlled experiments we probe two frozen foundation models, TabPFN~\citep{hollmann2025accurate} and TimesFM~\citep{das2024decoder}, and find that their input format carries a commitment: with identical weights and no gradient step, OOD error on Exp.~1.1 moves from $226.2\%$ to $0.1\%$ when the user supplies a Fourier transformation (\Cref{sec:fm_probes}). Finally, we adopt SINDy~\citep{brunton2016sindy} as a diagnostic: $\delta_{\mathrm{OOD}}$, the OOD derivative residual after sparse fitting, correlates with the canonically-correct $\varphi$ across our toy battery and four real-data natural science experiments (\Cref{sec:exp32}).

\section{Observational Equivalence and the DGP Commitment}
\label{sec:theory}

\textbf{Setup:}
We build our mathematical framework under the assumption of noise-dominated regimes, evaluate it with varying degrees of noise in our experiments, and finally, in \Cref{sec:scope}, test the empirical limits with respect to coverage, dimensionality, target complexity order, and noise.

Target $f: \X \to \R$; observations $y_i = f(x_i) + \eta_i$ with $x_i \stackrel{\mathrm{iid}}{\sim} \Dtrain$ supported on a window $W \subseteq \X$ and $\eta_i \stackrel{\mathrm{iid}}{\sim} \mathcal{N}(0, \sigma^2)$.
A \emph{feature map} $\varphi: \X \to \Zx$ and an invertible \emph{label map} $\psi: \Y \to \Zy$ recast learning as finding $h: \Zx \to \R$ with $h(\varphi(x)) \approx \psi(f(x))$; the predictor on the original scale is $\hat{f} = \psi^{-1} \circ h \circ \varphi$.
The pair $(\varphi, \psi)$ together with a model class $\mathcal{M}$ over $\Zx$ jointly form the \emph{DGP commitment} $(\varphi, \psi, \mathcal{M})$.
These are alternative routes for supplying one commitment rather than three independently necessary pillars: an expressive $\mathcal{M}$ can absorb a monotone $\psi$, which routes the commitment through $\mathcal{M}$ instead of avoiding it, while a frozen $\mathcal{M}$ leaves $(\varphi, \psi)$ as the only channels.
For a test distribution $\Dtest$, the OOD risk is the population functional $\eps_{\mathrm{OOD}}(h) = \E_{x \sim \Dtest}\bigl[\bigl(\hat{f}(x) - f(x)\bigr)^2\bigr]$. We do not require $\Dtest$ to be i.i.d., and several experiments use deterministic grids, permitting worst-case $W_{\mathrm{OOD}}$ under the distribution shift $\Dtrain \ne \Dtest$.
Reported regression percentages use normalised MAE, $100\,\mathrm{mean}|\hat{f}-f|/(\mathrm{mean}|f|+10^{-12})$; other task metrics are specified separately (the squared form above is the canonical definition). The $\delta_{\mathrm{OOD}}$ used in \Cref{sec:exp32} is a separate coordinate residual in $\Zx$, not an instance of $\eps_{\mathrm{OOD}}$.

\begin{definition}[$\eps_W$-observational equivalence]
\label{def:obseq}
$P_1, P_2$ are \emph{$\eps_W$-observationally equivalent} on $W$ under $\Dtrain$ if $\norm{P_1 - P_2}_{L^2(\Dtrain)} := \bigl(\E_{x \sim \Dtrain}[(P_1(x) - P_2(x))^2]\bigr)^{1/2} \le \eps_W$.
\end{definition}

\begin{proposition}[ERM is blind to observationally-equivalent processes]
\label{prop:erm}
Let $D_n = \{(x_i, y_i)\}_{i=1}^n$ with $x_i \stackrel{\mathrm{iid}}{\sim} \Dtrain$ on $W$ and $y_i = P_k(x_i) + \eta_i$, $\eta_i \stackrel{\mathrm{iid}}{\sim} \mathcal{N}(0, \sigma^2)$ independently of the inputs, with $\sigma>0$, for unknown $k \in \{1, 2\}$ where $P_1, P_2$ are $\eps_W$-observationally equivalent on $W$ (\Cref{def:obseq}), and write $t_W := \sqrt{n}\,\eps_W/\sigma$.
For any (deterministic or randomised) test $T: D_n \to \{1, 2\}$,
\begin{equation}
\label{eq:prop1-l2}
  \min_{k \in \{1,2\}} \Pr_{D_n \sim \mathcal{D}_k^n}[T(D_n) = k]
  \;\le\; \Phi(t_W/2),
\end{equation}
with $\Phi$ the standard normal CDF.
In particular, in the regime $\eps_W\sqrt{n} \ll \sigma$, no $W$-supported test reliably distinguishes $P_1$ from $P_2$.
Selection criteria such as cross-validation, training loss, or held-out loss on $W$ are special cases.
\end{proposition}

Relaxing $\Phi$ to its tangent at $0$ turns \eqref{eq:prop1-l2} into the linear form $\tfrac{1}{2} + \eps_W\sqrt{n}/(2\sqrt{2\pi}\,\sigma)$ (\Cref{app:proof_prop_erm}), which is what \Cref{sec:scope} uses to locate the two regimes.
The feature map determines which process the model assumes: a $\varphi_i$ for which $P_i = g_i \circ \varphi_i$ with globally simple $g_i$ commits the model to $P_i$. Both $\varphi_1$ and $\varphi_2$ reach near-identical training error when $P_1, P_2$ agree on $W$; only the matching commitment extrapolates. The pair $(\varphi, \psi)$ thus canonicalises a subset of the Rashomon set $\mathcal{R}(\tau) = \{h: \mathcal{L}(h) \le \mathcal{L}^* + \tau\}$~\citep{breiman2001statistical} that can be unique on $W_{\mathrm{OOD}}$ even when many elements agree on $W$.

\begin{proposition}[Window--strip separation]
\label{prop:strip}
Let $P_1, P_2$ be $\eps_W$-observationally equivalent on $W$ under $\Dtrain$ (\Cref{def:obseq}), observed through $n$ samples, and write $\Delta := P_1 - P_2$.
Let a near-boundary strip $V = [b, b + \delta_V]$ carry $n_V$ design points drawn from $\mathcal{D}_V$ under independent Gaussian noise at the same level $\sigma>0$, and set $t_V := \sqrt{n_V}\,\norm{\Delta}_{L^2(\mathcal{D}_V)}/\sigma$ (norm under the empirical design measure).
Conditional on the design, the likelihood-ratio test on $V$ (with fair randomisation at ties) is correct with probability exactly $\Phi(t_V/2)$.
\end{proposition}

\begin{corollary}[Geometric floor]
\label{cor:nearboundary}
Suppose in addition that $P_1, P_2$ diverge at rate $r > 0$ past the boundary, \ie $|\Delta(b + u)| \ge r u$ for $u \in (0, \delta_V]$, and that $\mathcal{D}_V$ is uniform on $V$.
Then $\norm{\Delta}_{L^2(\mathcal{D}_V)} \ge r\,\delta_V/\sqrt{3}$ (under this uniform $\mathcal{D}_V$), so
\begin{equation}
\label{eq:floor}
  \delta_V \;\ge\; \sqrt{3}\,\eps_W/r
  \qquad\Longrightarrow\qquad
  \norm{\Delta}_{L^2(\mathcal{D}_V)} \ge \eps_W,
  \quad\text{\ie}\quad
  t_V \;\ge\; t_W\sqrt{n_V/n}.
\end{equation}
Here $t_V$ uses the uniform population measure, rather than the empirical measure in \Cref{prop:strip}; the width guarantees a population gap of at least $\eps_W$, not finite-sample detection.
\end{corollary}
Proofs with additional interpretations are given in \Cref{app:proof_prop_erm,app:proof_prop_nearboundary}. \Cref{app:proof_prop_erm} restates \Cref{prop:strip} in its full four-part form.



\section{Controlled Experiments}
\label{sec:controlled}

This section examines how the choice of commitment $(\varphi, \psi, \mathcal{M})$ influences out-of-distribution (OOD) performance. Through synthetic experiments comparing observationally equivalent DGPs (\Cref{def:obseq}), we demonstrate how the components of $(\varphi, \psi, \mathcal{M})$ directly determine the OOD error. All experimental details are in \Cref{app:experiments}.

\paragraph{Exp.~1.1: $\sin(x)$ vs.\ Taylor-9 ($\varphi$ alone).}
\label{sec:exp11}
\noindent\textit{Same models, same training data, only $\varphi$ varies.}
We start with a simple toy model without noise. $f(x) = \sin(x)$ on $W = [-\pi, \pi]$, OOD $[\pi, 3\pi]$, and the degree-$9$ Taylor expansion can be observationally equivalent on $W$ and diverge on OOD. In this example, since there is no noise, the functions are fully identifiable.
We contrast raw $x$ with $\varphi(x) = (\sin x, \cos x)$ (biasing to a periodic DGP), see \Cref{fig:exp11_concept}. 
The main takeaway of \Cref{tab:exp11} is Fourier MLP vs. Raw MLP: identical training error, yet ${\sim}520\times$ in OOD ($0.3\%$ vs.\ $155.9\%$). OLS on Fourier features attains $0.0\%$, matching \Cref{lem:sin}. This mechanism is present in Foundation models as they inherit the representation via input format: TimesFM and TabPFN succeed on Fourier columns and fail on raw $x$, with no change in weights. TimesFM is biased towards $\sin(x)$ even when the DGP is in fact Taylor-$9$. For a similar nonlinear toy example, see \Cref{app:exp13}, \Cref{fig:exp13_concept}.

\begin{table}[ht]
  \centering
  \small
  \caption{Exp.~1.1: $P_1{=}\sin(x)$ vs.\ $P_2{=}\mathrm{Taylor}_9(x)$, train $[-\pi, \pi]$ (max rel.\ dev.\ ${<}1.1\%$), OOD $[\pi, 3\pi]$. The DGP column identifies the training-data source. OOD $P_1$ and OOD $P_2$ score the same predictions against $P_1$ and $P_2$ as ground truth, respectively. \colorbox{gray!12}{\strut Shaded rows} use a feature map. OOD relative error (\%), mean $\pm$ std over 3 seeds; best per column in bold.}
  \label{tab:exp11}
  \begin{tabular}{lllrrr}
    \toprule
    \textbf{DGP} & \textbf{Feature map} & \textbf{Model} & \textbf{ID (\%)} & \textbf{OOD $P_1$ (\%)} & \textbf{OOD $P_2$ (\%)} \\
    \midrule
    \rowcolor{gray!12}
    $P_1$: $\sin(x)$ & Fourier & OLS & 0.0 $\pm$ 0.0 & \textbf{0.0 $\pm$ 0.0} & 99.8 $\pm$ 0.0 \\
    \rowcolor{gray!12}
     & Fourier & SINDy & 0.0 $\pm$ 0.0 & \textbf{0.0 $\pm$ 0.0} & 99.8 $\pm$ 0.0 \\
    \rowcolor{gray!12}
     & Fourier & MLP & 0.3 $\pm$ 0.3 & 0.3 $\pm$ 0.3 & 99.8 $\pm$ 0.0 \\
    \rowcolor{gray!12}
     & Fourier & TabPFN & 0.1 $\pm$ 0.0 & 0.1 $\pm$ 0.0 & 99.8 $\pm$ 0.0 \\
    \rowcolor{gray!12}
     & Fourier & TimesFM (cov) & -- & 1.0 $\pm$ 0.0 & 99.8 $\pm$ 0.0 \\
    \cmidrule(lr){2-6}
     & None & SINDy & 0.0 $\pm$ 0.0 & $\gg 9{,}999$ & 42.5 $\pm$ 0.0 \\
     & None & MLP & 0.3 $\pm$ 0.1 & 155.9 $\pm$ 4.0 & 100.6 $\pm$ 0.0 \\
     & None & TabPFN & 0.2 $\pm$ 0.0 & 226.2 $\pm$ 0.0 & 101.0 $\pm$ 0.0 \\
     & None & TimesFM & -- & 8.9 $\pm$ 0.0 & 99.8 $\pm$ 0.0 \\
    \midrule
    $P_2$: Taylor-9 & None & SINDy & 0.0 $\pm$ 0.0 & $\gg 9{,}999$ & \textbf{0.0 $\pm$ 0.0} \\
     & None & MLP & 0.6 $\pm$ 0.5 & 154.3 $\pm$ 2.5 & 100.5 $\pm$ 0.0 \\
     & None & TabPFN & -- & 224.0 $\pm$ 0.0 & 101.0 $\pm$ 0.0 \\
     & None & TimesFM & -- & 9.5 $\pm$ 0.0 & 99.8 $\pm$ 0.0 \\
    \bottomrule
  \end{tabular}
\end{table}

\paragraph{Exp.~1.2: power law vs.\ exponential (joint $(\varphi, \psi)$).}
\label{sec:exp12}
\noindent\textit{Adds a nonlinear label map $\psi$ and noise ($\sigma=0.2$).}
$P_1(x) = x^2$ and $P_2(x) = e^{(\ln 4)(x-1)}$ match exactly at the endpoints $x{=}1, 2$ and are observationally equivalent across $W = [1, 2]$ at $\eps_W = 7.5\%$ of $\norm{P_1}_{L^2(\Dtrain)}$, then diverge on $[2, 10]$.
We pair $\psi = \log$ with $\varphi = \log$ (correct for $P_1$) or $\varphi = \mathrm{id}$ (correct for $P_2$). TabPFN applies $\phi$ to inputs and target and TimesFM uses the matching coordinate as a dynamic covariate to forecast $\log y$.

\begin{table}[t]
  \centering
  \small
  \caption{Exp.~1.2 (excerpt). $P_1 = x^2$ vs.\ $P_2 = e^{(\ln 4)(x-1)}$, training $[1, 2]$, OOD $[2, 10]$ (divergence $\sim 10^5$). The DGP column identifies the training-data source. OOD $P_1$ and OOD $P_2$ score the same predictions against $P_1$ and $P_2$ as ground truth, respectively. \colorbox{gray!12}{\strut Shaded rows} use the correct feature map for OLS. Mean $\pm$ std over 3 seeds; best per column in bold; full grid in App.~\ref{app:exp12_full}.}
  \label{tab:exp12_main}
  \begin{tabular}{lllrrr}
    \toprule
    \textbf{DGP} & \textbf{Feature map} & \textbf{Model} & \textbf{ID (\%)} & \textbf{OOD $P_1$ (\%)} & \textbf{OOD $P_2$ (\%)} \\
    \midrule
    $P_1$: $x^2$ & None & MLP & 1.2 $\pm$ 0.2 & 85.1 $\pm$ 1.0 & 100.0 $\pm$ 0.0 \\
    \rowcolor{gray!12}
     & log-log & OLS & 1.0 $\pm$ 0.5 & \textbf{3.2 $\pm$ 1.7} & 99.8 $\pm$ 0.0 \\
    \rowcolor{gray!12}
     & log-log & TabPFN & 1.3 $\pm$ 0.3 & 85.4 $\pm$ 1.0 & 100.0 $\pm$ 0.0 \\
    \rowcolor{gray!12}
     & log-log & TimesFM & -- & 28.3 $\pm$ 0.0 & 99.8 $\pm$ 0.0 \\
     & log-$y$ & OLS & 3.2 $\pm$ 0.3 & $\gg 9{,}999$ & 13.5 $\pm$ 3.1 \\
    \midrule
    $P_2$: $e^{\alpha(x-1)}$ & None & MLP & 1.4 $\pm$ 0.2 & 88.3 $\pm$ 1.2 & 100.0 $\pm$ 0.0 \\
     & log-log & OLS & 3.1 $\pm$ 0.2 & 5.6 $\pm$ 3.1 & 99.8 $\pm$ 0.0 \\
    \rowcolor{gray!12}
     & log-$y$ & OLS & 1.1 $\pm$ 0.5 & $\gg 9{,}999$ & \textbf{12.8 $\pm$ 8.8} \\
    \rowcolor{gray!12}
     & log-$y$ & TabPFN & 1.4 $\pm$ 0.3 & 80.6 $\pm$ 3.3 & 100.0 $\pm$ 0.0 \\
    \rowcolor{gray!12}
     & log-$y$ & TimesFM & -- & $\gg 9{,}999$ & 29.6 $\pm$ 0.0 \\
    \bottomrule
  \end{tabular}
\end{table}

The correct commitments in \Cref{tab:exp12_main}, achieve $3.2\%$ ($P_1$, log-log OLS) and $12.8\%$ ($P_2$, log-$y$ OLS) while the incorrect collapse ($\gg 9{,}999\%$ for log-$y$ OLS on $P_1$, $99.8\%$ for log-log OLS on $P_2$). The implementation uses $\log(\max(y,10^{-8}))$; coefficient errors in log space are amplified exponentially towards $x{=}10$. Two further failure modes are visible: wrong label map fails catastrophically on $P_1$ and a correct $(\varphi, \psi)$ with the wrong model class fails (log-log MLP reaches $83.3\%$ OOD vs.\ $3.2\%$ for OLS, because the MLP does not exploit linearity in $(\log x, \log y)$), full grid in \Cref{app:exp12_full}.

\subsection{Foundation models inherit DGP commitment via input format}
\label{sec:fm_probes}

In the frozen foundation models we tested, the selected continuation changes at inference with the supplied representation, with identical weights and no gradient step. Pretraining may itself encode a commitment in $\mathcal{M}$; once those weights are fixed, OOD-relevant structure has to arrive through inference-time conditioning, and the input format is the channel that carries it. On Exp.~1.1, TimesFM zero-shot reaches $8.9\%$ OOD on the $\sin(x)$ context (vs.\ $155.9\%$ for raw-$x$ MLP) and drops to $1.0\%$ with a Fourier covariate. TabPFN goes from $226.2\%$ on raw $x$ to $0.1\%$ on Fourier columns. The mechanism is geometric: the Fourier map wraps train and OOD onto the same compact manifold $S^1$, so OOD becomes interpolation in feature space and TabPFN (a tabular regressor that does not extrapolate beyond the training input distribution) succeeds. The $\log$ map compresses but does not bound (Exp.~1.2: $\log x_{\mathrm{train}} \in [0, 0.69]$, $\log x_{\mathrm{OOD}} \in [0.69, 2.30]$), so TabPFN with the correct $(\log,\log)$ map still clamps to the training $y$-range and fails ($85.4\%$ on $P_1$, comparable to its raw failure $80.3\%$). TimesFM with the matching covariate improves but lags ($28.3\%$ on $P_1$, $29.6\%$ on $P_2$). Linear OLS in the correct transformed space has lower OOD error on both DGPs. The right $\varphi$ is necessary but not sufficient: the commitment requires a matching model class that extrapolates linearly in $\varphi$-space, which OLS and SINDy do by construction.

\section{Cross-Domain Validation}
\label{sec:cross_domain}

Does the same commitment mechanism predict OOD behaviour on independent natural-science domains?

\begin{figure}[t]
  \centering
  \begin{minipage}[t]{0.64\linewidth}
    \vspace{0pt}\centering
    \includegraphics[width=\linewidth, trim=9pt 9pt 7pt 17pt, clip]{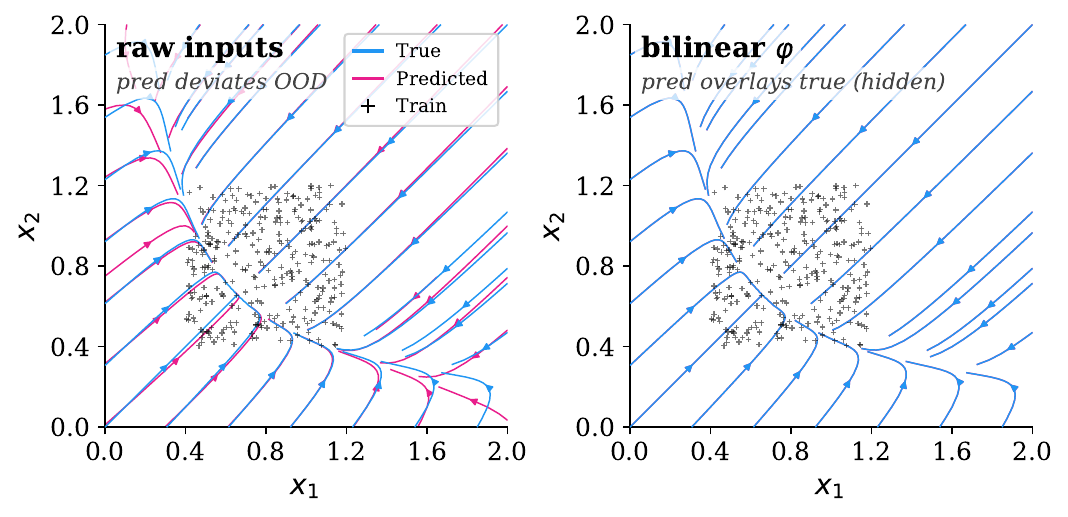}\\
    \scriptsize\textbf{(a) Chemistry: MAK}
  \end{minipage}
  \begin{minipage}[t]{0.3\linewidth}
    \vspace{0pt}\centering
    \includegraphics[width=\linewidth, trim=9pt 9pt 7pt 22pt, clip]{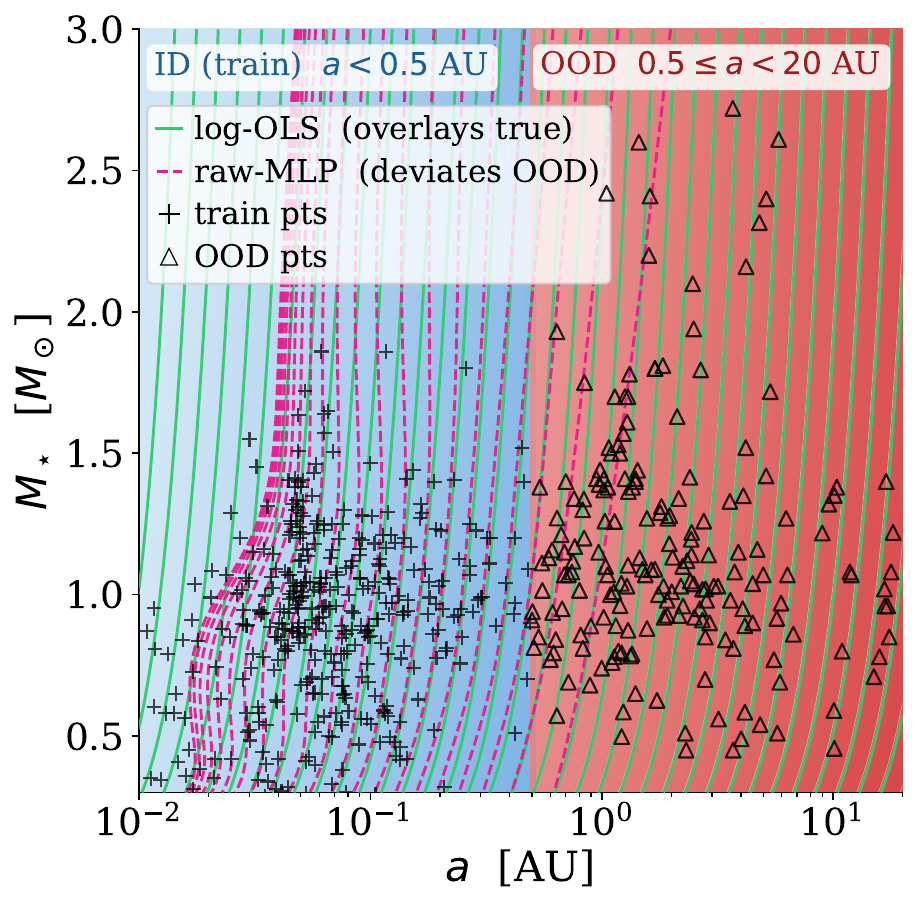}\\
    \scriptsize\textbf{(b) Physics: Kepler}
  \end{minipage}\\[0.4em]
  \begin{minipage}[t]{0.6\linewidth}
    \vspace{0pt}\centering
    \includegraphics[width=\linewidth]{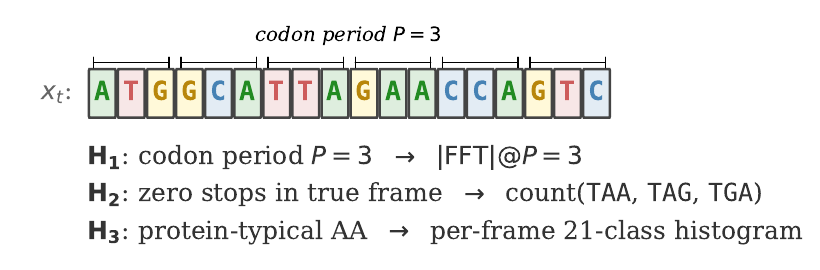}\\
    \scriptsize\textbf{(c) Biology: CDS DGP}
  \end{minipage}
  \begin{minipage}[t]{0.3\linewidth}
    \vspace{0pt}\centering
    \includegraphics[width=\linewidth]{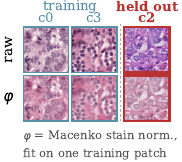}\\
    \scriptsize\textbf{(d) Histopathology: stain basis}
  \end{minipage}
  \caption{Concept figures for the four natural-science domains.
    \emph{(a) MAK}: predicted (pink) and true (blue) streamlines for raw inputs (left) and bilinear features (right).
    \emph{(b) Kepler}: iso-$T$ contours, ID shaded blue and OOD red. Log-OLS (green) overlays the true levels everywhere, the raw MLP (pink dashed) departs from them for $a \geq 0.5$~AU.
    \emph{(c) CDS}: three genetic-code feature maps that transfer zero-shot across organisms.
    \emph{(d) Histopathology}: training (c0, c3) and held-out (c2) patches before and after the stain map $\varphi$.\label{fig:cross_domain_concept}}
\end{figure}

\paragraph{Exp.~2.1: mass-action kinetics (chemistry).}
\label{sec:exp21}
When fitting to a Mass-action kinetics (MAK) problem on a synthetic 2-node graph $\dot{x}_i = F - B x_i - R x_i x_j$~\citep{vasiliauskaite2024generalization}, see \Cref{fig:cross_domain_concept}(a). Bilinear features $\varphi(x) = [1,\, x_i,\, x_i x_j]$ globally linearise the target. 
\Cref{tab:exp21_main} isolates the $(\varphi, \mathcal{M})$ interaction: \texttt{OLS} recovers $(F, -B, -R)$ to machine precision at $0.0\%$ OOD. The same $\varphi$ with an MLP head reaches $16.9\%$, as the model class does not exploit linearity, and raw inputs reach $24.4\%$.

\begin{table}[t]
  \centering
  \small
  \caption{Exp.~2.1 MAK (excerpt). 2-node mass-action system, training $(x_1, x_2) \in [0.4, 1.2]^2$, OOD $[0, 2.0]^2 \setminus [0.4, 1.2]^2$. \colorbox{gray!12}{\strut Shaded rows} use the canonically-correct bilinear feature map $[1, x_i, x_i x_j]$. Mean $\pm$ std over 3 seeds; best OOD in bold; full grid in App.~\ref{app:exp21_full}.}
  \label{tab:exp21_main}
  \begin{tabular}{llrr}
    \toprule
    \textbf{Feature map} & \textbf{Model} & \textbf{ID (\%)} & \textbf{OOD (\%)} \\
    \midrule
    None & MLP & 0.5 $\pm$ 0.0 & 24.4 $\pm$ 0.9 \\
    \cmidrule(lr){1-4}
    \rowcolor{gray!12}
    Bilinear & OLS & 0.0 $\pm$ 0.0 & \textbf{0.0 $\pm$ 0.0} \\
    \rowcolor{gray!12}
    Bilinear & MLP & 1.4 $\pm$ 1.2 & 16.9 $\pm$ 1.3 \\
    \bottomrule
  \end{tabular}
\end{table}

\paragraph{Exp.~2.2: Kepler's third law on the NASA Exoplanet Archive (physics).}
\label{sec:exp22}
Similarly, on the NASA Exoplanet Archive, see \Cref{fig:cross_domain_concept}(b), the correct $\varphi$ based on Kepler's third law $(\log a, \log M_{\mathrm{star}})$ + OLS approximately recovers the Kepler exponents ($\hat{b}_a = 1.493$, $\hat{b}_{M_{\mathrm{star}}} = -0.493$ vs.\ theory $1.5, -0.5$) at $3.8\%$ OOD, with SINDy, $2.1\%$ (\Cref{tab:exp22_main}). The same $(a, M_{\mathrm{star}})$ inputs fed to a flexible MLP without the log transform overfit ID at $6.9\%$ but break to $79.0\%$ OOD. In raw $(a, M_{\mathrm{star}})$ space, the predicted iso-$T$ contours bend the wrong way past the training boundary (\Cref{fig:cross_domain_concept}(b)).

\begin{table}[t]
  \centering
  \small
  \caption{Exp.~2.2 Kepler (excerpt). NASA Exoplanet Archive, $n_{\mathrm{train}}=1{,}881$ ($a < 0.5$ AU), $n_{\mathrm{OOD}}=481$. \colorbox{gray!12}{\strut Shaded rows} use the canonically-correct $(\log a, \log M_{\mathrm{star}})$ feature map (Kepler's law). Mean $\pm$ std over 3 seeds; best OOD in bold; full grid in App.~\ref{app:exp22_full}.}
  \label{tab:exp22_main}
  \begin{tabular}{lllrr}
    \toprule
    \textbf{Variables} & \textbf{Feature map} & \textbf{Model} & \textbf{ID (\%)} & \textbf{OOD (\%)} \\
    \midrule
    $a,\,M_{\mathrm{star}}$ & $a,\,M_{\mathrm{star}}$ & MLP & 6.9 $\pm$ 0.2 & 79.0 $\pm$ 2.1 \\
    \cmidrule(lr){1-5}
    \rowcolor{gray!12}
    $a,\,M_{\mathrm{star}}$ & $\log a,\,\log M_{\mathrm{star}}$ & OLS & 4.2 $\pm$ 0.0 & 3.8 $\pm$ 0.0 \\
    \rowcolor{gray!12}
    $a,\,M_{\mathrm{star}}$ & $\log a,\,\log M_{\mathrm{star}}$ & SINDy & 3.9 $\pm$ 0.0 & \textbf{2.1 $\pm$ 0.0} \\
    \rowcolor{gray!12}
    $a,\,M_{\mathrm{star}}$ & $\log a,\,\log M_{\mathrm{star}}$ & MLP & 4.4 $\pm$ 0.1 & 60.9 $\pm$ 2.5 \\
    \bottomrule
  \end{tabular}
\end{table}

\paragraph{Exp.~2.3: cross-species coding-DNA detection (biology).}
\label{sec:exp23}
In a DNA sequencing problem, we train a transformer-based architecture with a domain knowledge $\varphi$ integrated into the architecture as \textit{Feature Engineering} (FE) Heads. The selected problem is to classify $300$\,bp (base-pair) DNA reads as coding-DNA-sequence (CDS) vs.\ non-CDS. Genomes vary in GC content (\% guanine--cytosine bases). For our FE heads, we consider the following domain knowledge: coding regions exhibit a 3-nucleotide \emph{codon} period (H1), depletion of in-frame stop codons (H2), and a 21-class amino-acid (AA) composition (H3). We train on \textit{S.\,cerevisiae} chromosomes~I–IX and evaluate zero-shot on five organisms spanning $19$--$65\%$ GC (\Cref{fig:cross_domain_concept}(c)). The shared base model is a ${\sim}0.3$M-parameter random-init Transformer encoder ($2$~layers, $4$~heads, $d{=}128$). We evaluate two non-FE heads (learned absolute positional encoding (APE), as used in the DNABERT~\citep{ji2021dnabert} default; sinusoidal PE~\citep{vaswani2017attention}, denoted ``sin'') and two FE heads (the codon-Fourier magnitude $|\mathrm{FFT}(h)|@P{=}3$ at the codon period (codon), capturing H1, and the biology head, which adds H2 in-frame stop counts and H3 amino-acid composition. Both Non-FE heads collapse on most OOD species and fall below chance on \textit{M.\,tuberculosis} ($0.459$ APE, $0.445$ sin at $65\%$ GC). The biologically inspired heads transfer at area under the ROC curve (AUROC) $0.718$--$0.960$ zero-shot across the five OOD organisms. Stacking three genetic-code features (codon period, in-frame stop depletion, AA composition) lifts AUROC and wins on every organism (\Cref{tab:exp23_species_main}). \textit{Hardest case}: \textit{M.\,tuberculosis} (codon+biology $0.750$, the lowest across organisms) is the GC-skewed outlier ($65\%$ GC vs.\ yeast's $38\%$).

\begin{table}[ht]
  \centering
  \small
  \caption{Cross-species CDS-vs-nonCDS classification, AUROC (mean $\pm$ std over 3 seeds). Two standard heads (APE and sin) and two FE heads (codon, codon+biology). Best variant per organism in bold.}
  \label{tab:exp23_species_main}
  \begin{tabular}{l r rrrr}
    \toprule
    \textbf{Organism} & \textbf{GC} & \textbf{APE} & \textbf{sin} & \textbf{codon} & \textbf{codon\,+\,biology} \\
    \midrule
    \textit{S.\,cerevisiae} & 38\% & 0.745\,$\pm$\,0.012 & 0.776\,$\pm$\,0.009 & 0.903\,$\pm$\,0.006 & \textbf{0.962\,$\pm$\,0.005} \\
    \textit{E.\,coli} & 50\% & 0.599\,$\pm$\,0.011 & 0.606\,$\pm$\,0.019 & 0.783\,$\pm$\,0.030 & \textbf{0.839\,$\pm$\,0.017} \\
    \textit{B.\,subtilis} & 44\% & 0.667\,$\pm$\,0.023 & 0.697\,$\pm$\,0.006 & 0.891\,$\pm$\,0.007 & \textbf{0.926\,$\pm$\,0.006} \\
    \textit{S.\,pombe} & 36\% & 0.823\,$\pm$\,0.010 & 0.824\,$\pm$\,0.021 & 0.917\,$\pm$\,0.008 & \textbf{0.953\,$\pm$\,0.005} \\
    \textit{P.\,falciparum} & 19\% & 0.606\,$\pm$\,0.098 & 0.483\,$\pm$\,0.112 & 0.905\,$\pm$\,0.030 & \textbf{0.960\,$\pm$\,0.010} \\
    \textit{M.\,tuberculosis} & 65\% & 0.459\,$\pm$\,0.004 & 0.445\,$\pm$\,0.016 & 0.718\,$\pm$\,0.035 & \textbf{0.750\,$\pm$\,0.068} \\
    \bottomrule
  \end{tabular}
\end{table}

\paragraph{Exp.~2.4: multi-hospital histopathology (high-dimensional perception).}
\label{sec:exp24}

The three settings above are low-dimensional with analytic feature maps. Camelyon17-WILDS~\citep{bandi2018detection, koh2021wilds} tests the same prediction on images, with a physical model as the map: $96{\times}96$ H\&E-stained lymph-node patches, tumour vs.\ normal, from five hospitals with different staining protocols. A ResNet-18~\citep{he2016deep} trained \emph{from scratch} on three hospitals, with the epoch selected by ID accuracy, reaches ID AUC ${\approx}0.998$ but only $0.754 \pm 0.059$ on a held-out hospital ($5$ seeds). Under the Beer--Lambert model a stained image factorises as $\mathrm{OD} = WC$ with $W$ the stain vectors, so Macenko normalisation~\citep{macenko2009method}, re-rendering each image in one basis fit on a \emph{single} training patch with no OOD statistics, supplies $\varphi$ (\Cref{fig:cross_domain_concept}(d)). It cuts the held-out hospital's offset from the training colour cloud from $3.0$ to $0.4$ training standard deviations (\Cref{fig:histopath_concept}), lifts OOD AUC to $0.984 \pm 0.001$, and shrinks the seed standard deviation ${\sim}60\times$. A \emph{fixed} random RGB affine of the same form deepens the gap ($0.631 \pm 0.074$), while other maps such as a per-image channel standardisation also closes it ($0.987 \pm 0.002$, \Cref{tab:histopath}): adequate per-image adaptation $\varphi$ separates success from failure.

\noindent
Four mechanisms (mass-action kinetics, gravity, biological invariants, and dye absorption) and one prediction: correct $\varphi$ + matching $\mathcal{M}$ wins, wrong commitment fails, structurally rather than for want of samples.

\section{Practitioner Artefacts}
\label{sec:tools}

The framework yields three practitioner artefacts: two methodological tools and one conceptual reframing.

\paragraph{Exp.~3.1: near-boundary validation as a DGP-selection heuristic.}
\label{sec:exp31}
\Cref{cor:nearboundary} prescribes a selection heuristic: split into $D_{\mathrm{train}} \cup D_{\mathrm{val}}$ with $D_{\mathrm{val}}$ a near-boundary segment $V$ of width $\delta_V \ge \sqrt{3}\,\eps_W/r$ outside $D_{\mathrm{train}}$, train each candidate $(\varphi_i, \psi_i)$ on $D_{\mathrm{train}}$, score its validation error $\hat{R}_i$ on $D_{\mathrm{val}}$, pick the minimum. The val window exits the agreement zone of the candidate hypotheses, so the bias gap visible on $V$ reflects the OOD divergence the candidates would have past $V$. CV folds, drawn from $D_{\mathrm{train}}$ itself, so they cannot see this gap.

We compare on $100$ trials per seed over three seeds of $f(x) = \sin(x)$, library $\{$\texttt{fourier}, \texttt{poly7}, \texttt{poly9}, \texttt{raw}$\}$. \textbf{Near-boundary picks \texttt{fourier} $100\%$ of the time. CV picks it $81.3\%$} (full grid in \Cref{app:exp31_full}). CV's $56$ failures land on the observationally-near-equivalent polynomial confounds ($37$ \texttt{poly7}, $19$ \texttt{poly9}, $0$ \texttt{raw}). The CV failure mode is overfitting on the polynomial bases in the signal-dominated regime. A regime sweep over $\sigma$ from $10^{-3}$ to $1$ confirms the contrast is robust to noise level (near-boundary stays at $100\%$, CV stays at ${\sim}80\%$).

\paragraph{Exp.~3.2: the SINDy $\delta_{\mathrm{OOD}}$ correlation diagnostic.}
\label{sec:exp32}
In the \emph{SINDy regime}~\citep{brunton2016sindy}, where the DGP admits a sparse polynomial representation in some coordinate, $\delta_{\mathrm{OOD}}$, the OOD residual in the fitted target space after sparse fitting, flags coordinates that admit such a representation. It is useful when the candidate set of $\varphi$'s is too large for full near-boundary scoring. On a toy battery of $8$ univariate DGPs ($35$ pairs over Fourier and log-polynomial families; $G_1 = \sin(x)\cos(2x)$ and per-row data in \Cref{app:exp13}), Spearman $\rho = 0.82$ with the preassigned wrong-pair label, nominal $p = 1.3{\times}10^{-9}$. The diagnostic extends to every SINDy-regime experiment here: in Exp.~1.1, Exp.~1.2, Kepler, and MAK, the canonically-correct commitment yields $\delta_{\mathrm{OOD}} \le 0.012$ while wrong feature maps yield $\delta_{\mathrm{OOD}}$ from ${\sim}0.04$ up to $6.7{\times}10^{5}$ (\Cref{tab:sindy_delta_consolidated}). Failure mode: non-polynomial dynamics (\eg $G_2 = \tanh(\cdot)$) have no exact representation in the tested libraries, so $\delta_{\mathrm{OOD}} > 0$ throughout those libraries but cannot rank.

\paragraph{Exp.~3.3: CE/PE reframing and architectural invariance check.}
\label{sec:exp33}
In the genome example, Exp. 2.3, we show how a $\varphi$ that injects adequate domain knowledge helps sequence models. For practitioners, this motivates viewing the \emph{Content Embedding} (CE) and \emph{Position Embedding} (PE) as an integrated feature-engineering component. To show this is not an artefact of transformers, we compare with other architectures on a simplified toy problem. On a periodic-sequence task $y_t = f(t \bmod P)$ with random tokens, we sweep three causal backbones (Transformer, Mamba, S4D), varying periods $P$ and sequence lengths $L$.
Here $f:\{0,\ldots,P-1\}\to\{0,\ldots,255\}$ is a fixed random lookup table of class labels. For each $P$, its entries are sampled independently and uniformly, then held fixed across training, ID, OOD, and model conditions.
We show the same mechanism as in all previous experiments.
This toy design isolates the PE: random tokens make the CE label-independent, and a $t \bmod P$ label leaves position (hence the PE) as the only predictive channel.
State-space models (SSMs) do not natively use an explicit PE (recurrence and depthwise causal convolution already carry position), so for a controlled comparison we graft the PE onto all three backbones identically: summed into the token embedding before the first block, exactly as on Transformer. The added PE coexists with the SSM's implicit positional signal, isolating the contribution of the PE family from the backbone.
The exact-Fourier PE $\varphi(t) = (\sin 2\pi t/P, \cos 2\pi t/P)$ maps every OOD position to the same $\mathbb{S}^1$ point as some training position (\Cref{lem:sin}).

\begin{table}[ht]
  \centering
  \small
  \caption{Exp.~3.3 Architectural invariance (excerpt). ID/OOD error (\%) on the periodic annotation task across Transformer, Mamba, S4D and $P \in \{64, 256, 512\}$ ($L_\mathrm{train}{=}4P$, $L_\mathrm{OOD}{=}2P$; chance ${\approx}99.61\%$). \colorbox{gray!12}{\strut Shaded rows} use the correct PE; the OOD guarantee belongs to the feature map, not the backbone. Best OOD per (architecture, $P$) in bold. Mean over 3 seeds; 81 runs. Full $8$-PE $\times$ $4$-period grid in \Cref{tab:exp33_main}.}
  \label{tab:exp33_minitext}
  \resizebox{\textwidth}{!}{%
  \begin{tabular}{llrrrrrr}
    \toprule
    \textbf{Architecture} & \textbf{PE} & \multicolumn{2}{c}{$P{=}64$} & \multicolumn{2}{c}{$P{=}256$} & \multicolumn{2}{c}{$P{=}512$} \\
    \cmidrule(lr){3-4} \cmidrule(lr){5-6} \cmidrule(lr){7-8}
     & & ID & OOD & ID & OOD & ID & OOD \\
    \midrule
    \rowcolor{gray!12}
    Transformer & \textbf{Exact Fourier} & 0.0 $\pm$ 0.0 & \textbf{0.0 $\pm$ 0.0} & 0.0 $\pm$ 0.0 & \textbf{0.0 $\pm$ 0.0} & 0.6 $\pm$ 0.5 & \textbf{0.2 $\pm$ 0.3} \\
     & Sinusoidal & 0.0 $\pm$ 0.0 & 98.4 $\pm$ 0.6 & 0.0 $\pm$ 0.0 & 99.1 $\pm$ 0.4 & 0.0 $\pm$ 0.0 & 99.3 $\pm$ 0.1 \\
     & Learned APE & 0.0 $\pm$ 0.0 & 97.7 $\pm$ 1.2 & 0.0 $\pm$ 0.0 & 99.4 $\pm$ 0.3 & 0.0 $\pm$ 0.0 & 99.2 $\pm$ 0.2 \\
    \rowcolor{gray!12}
    Mamba & \textbf{Exact Fourier} & 0.0 $\pm$ 0.0 & \textbf{0.0 $\pm$ 0.0} & 0.0 $\pm$ 0.0 & \textbf{0.0 $\pm$ 0.0} & 0.0 $\pm$ 0.0 & \textbf{0.0 $\pm$ 0.0} \\
     & Sinusoidal & 0.0 $\pm$ 0.0 & 99.0 $\pm$ 1.0 & 0.0 $\pm$ 0.0 & 98.3 $\pm$ 0.8 & 0.0 $\pm$ 0.0 & 99.0 $\pm$ 0.2 \\
     & Learned APE & 0.0 $\pm$ 0.0 & 96.4 $\pm$ 1.1 & 0.0 $\pm$ 0.0 & 98.8 $\pm$ 0.3 & 0.0 $\pm$ 0.0 & 99.5 $\pm$ 0.1 \\
    \rowcolor{gray!12}
    S4D & \textbf{Exact Fourier} & 0.0 $\pm$ 0.0 & \textbf{0.0 $\pm$ 0.0} & 0.0 $\pm$ 0.0 & \textbf{0.0 $\pm$ 0.0} & 0.0 $\pm$ 0.0 & \textbf{0.0 $\pm$ 0.0} \\
     & Sinusoidal & 0.0 $\pm$ 0.0 & 97.9 $\pm$ 0.3 & 0.0 $\pm$ 0.0 & 99.4 $\pm$ 0.2 & 0.0 $\pm$ 0.0 & 99.2 $\pm$ 0.2 \\
     & Learned APE & 0.0 $\pm$ 0.0 & 98.4 $\pm$ 1.1 & 0.0 $\pm$ 0.0 & 98.7 $\pm$ 0.8 & 0.0 $\pm$ 0.0 & 99.4 $\pm$ 0.0 \\
    \bottomrule
  \end{tabular}
  }
\end{table}

\textbf{Exact-Fourier PE achieves $<0.22\%$ mean OOD error in every $(\text{model}, P)$ cell} ($0.0\%$ in 11 of 12, $0.2\%$ on Transformer at $P{=}512$); every other PE (sinusoidal, RoPE, learned APE, learnable sinusoidal, learned Fourier $k \in \{1, 4\}$, no-PE) sits between ${\sim}5.5\%$ and ${\sim}99\%$ OOD against the $99.61\%$ chance baseline (\Cref{tab:exp33_minitext}; full $8$-PE grid in \Cref{app:exp33_sweeps}). The OOD result is consistent across the tested backbones. Sinusoidal PE and RoPE do not attain the exact-Fourier result in this task. The learned-Fourier rows show that even when the correct DGP is parametrically expressible, SGD does not reliably select it from random init (the parametric-but-unreachable Rashomon phenomenon). Robustness sweeps over context length, OOD horizon, SSM state dimension, and PE initialisation appear in \Cref{app:exp33_sweeps}.

\section{Scope: Three Necessary Conditions and the Limits}
\label{sec:scope}

\begin{figure}[t]
  \centering
  \includegraphics[width=0.4\linewidth, trim=8pt 7pt 12pt 11pt, clip]{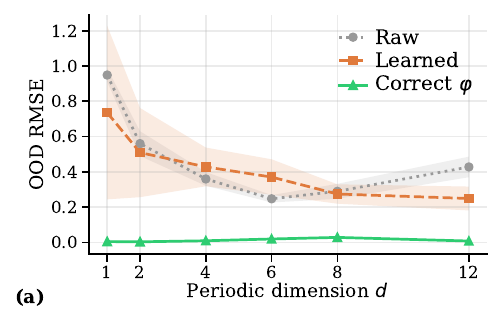}
  \includegraphics[width=0.42\linewidth, trim=8pt 7pt 12pt 11pt, clip]{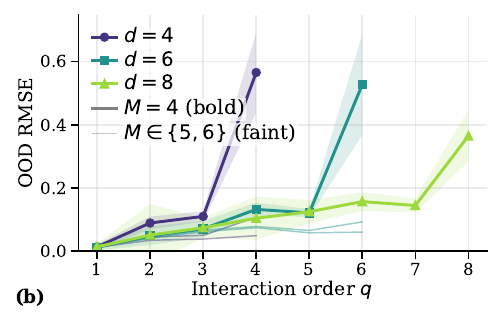}\\
  \includegraphics[width=0.4\linewidth, trim=8pt 7pt 12pt 11pt, clip]{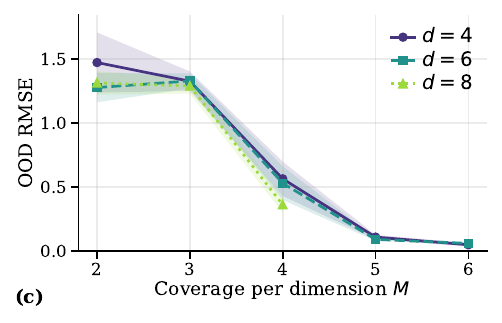}
  \includegraphics[width=0.44\linewidth, trim=8pt 7pt 12pt 11pt, clip]{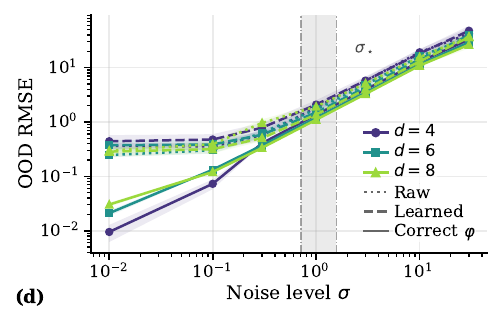}
  \caption{Four faces of the DGP commitment on the periodic $d$-torus with exact $\varphi(x){=}(\sin x_j,\cos x_j)_{j=1}^d$. Panels: (a) representation alignment, (b) model-class capacity, (c) coverage, (d) noise robustness. Bands: $\pm 1$ std over 100 seeds. The shaded vertical band in (d) is a plug-in estimate of the regime boundary $\sigma_\star = \sqrt{n}\,\eps_W/\sqrt{2\pi}$ of \Cref{prop:erm}, bracketed to $[0.71, 1.57]$ by the measured in-window gap.\label{fig:torus_scope}}
\end{figure}

\emph{When is the right representation not enough?} We construct a periodic $d$-torus benchmark in which the correct feature map is known exactly, then vary the model class $\mathcal{M}$, the sample budget, and noise levels (\Cref{fig:torus_scope}). Three conditions emerge as jointly necessary: \textbf{(a)} representation alignment, \textbf{(b)} a model class expressing the residual interaction order $q$ of the target in $\Zx$, and \textbf{(c)} coverage $N \asymp M_{\mathrm{crit}}^d$ with $M_{\mathrm{crit}} \approx 4$. A practitioner with correct $\varphi$ but no capacity for $q$ fails OOD; one with both but $N < M_{\mathrm{crit}}^d$ also fails. The torus benchmark localises where the commitment runs out and the rest of the OOD problem (identifiability, capacity, sampling) takes over.

\textbf{Noise $\sigma$ does not break the commitment (d).}
\label{sec:exp44}
\Cref{prop:erm} bounds the worst-case correct-decision probability of any $W$-supported test by $\Phi(t_W/2)$ with $t_W = \sqrt{n}\,\eps_W/\sigma$: informative when $\eps_W\sqrt{n}\ll\sigma$ and vacuous below, its linear form ceasing to constrain any test at $\sigma_\star = \sqrt{n}\,\eps_W/\sqrt{2\pi}$. At $N{=}4{,}096$ the empirical in-window $L^2$ gap between the raw and canonical commitments brackets $\eps_W$ to $[0.028, 0.062]$ (\Cref{app:experiments}), giving a plug-in estimate of $\sigma_\star$ between $0.71$ and $1.57$. Sweeping $\sigma \in \{0, 10^{-2}, 10^{-1}, 3{\cdot}10^{-1}, 1, 3, 10, 30\}$ at $d \in \{4, 6, 8\}$ spans this estimated crossover (\Cref{fig:torus_scope}d). The canonical $\varphi$ shows approximately linear scaling with training-noise level (log-log slopes $0.88$--$1.05$ over $\sigma>0$); raw and learned-Fourier baselines are bias-floored below this estimated crossover, then rise with $\sigma$ above it as the over-parameterised MLPs fit noise. Critically, the canonical commitment retains a clear OOD advantage in both regimes ($\sim 40\times$ lower OOD error at low $\sigma$, ${\sim}30\%$ smaller OOD error at high $\sigma$): the choice of $\varphi$ retains its advantage across the tested noise levels.

\textbf{Exp.~4.5 shows that a partly wrong commitment degrades gracefully in some directions and catastrophically in others.}
\label{sec:exp45}
A $10\%$ frequency error in the torus map, $\varphi_\omega = (\sin \omega x_j, \cos \omega x_j)$ with $\omega{=}0.90$, gives OOD-RMSE $0.258$, between the exact map ($0.008$) and no commitment ($0.554$), and additive feature noise degrades as smoothly (\Cref{tab:e3_misspec}, \Cref{app:misspec}). The exact map on only $m$ of $d{=}4$ coupled coordinates is \emph{worse} than no map ($0.503$ to $0.599$ against $0.353$), so a partly wrong commitment can be worse than none. At $\omega{=}0.90$ the mean ID error ($0.004$) is even below the exact map's ($0.008$) while OOD error is ${\sim}30\times$ higher, so ID fit does not rank commitments (\Cref{prop:erm}).

\section{Discussion and Conclusion}
\label{sec:conclusion}

\textbf{Conclusion.}
OOD risk on a single training window is non-identifiable from data alone (\Cref{prop:erm}); the structural commitment $(\varphi, \psi, \mathcal{M})$ is the only information the modeller supplies that resolves it.
A fixed architecture at identical ID loss can differ by ${\sim}520\times$ OOD when only $\varphi$ changes (\Cref{tab:exp11}), and the mechanism transfers, untouched, to chemistry ($0.0\%$ vs.\ $24.4\%$ OOD, \Cref{tab:exp21_main}), physics ($2.1\%$ vs.\ $79.0\%$, \Cref{tab:exp22_main}), biology (AUROC up to $0.96$ vs.\ below-chance phase-equivariant readouts, \Cref{tab:exp23_species_main}), and Transformer/Mamba/S4D over $264$ runs (\Cref{sec:exp33}).
This sharpens ``feature engineering is dead'': deep learning supersedes hand-crafted features ID once samples and noise permit, but $(\varphi, \psi, \mathcal{M})$ remains the only ID-invisible signal that selects one extrapolation from the observationally-equivalent Rashomon class.
Pretraining hides the commitment in the input format (\Cref{sec:fm_probes}) without removing it; near-boundary loss and the SINDy $\delta_{\mathrm{OOD}}$ recover it as operational selection signals (\Cref{sec:tools}).

\textbf{Limitations.}
\Cref{prop:erm} is informative only when $\eps_W\sqrt{n}\ll\sigma$. Exps.~1.1, 1.2 and 3.1 sit in the signal-dominated regime where it is vacuous, yet the multiplicity it describes still binds CV there (\Cref{sec:exp31}).
Near-boundary validation needs boundary samples, and the SINDy diagnostic flags non-polynomial dynamics such as $\tanh$ but cannot rank maps for them.
Correct $\varphi$ is necessary, not sufficient: at $N\lesssim M_{\mathrm{crit}}^d$ ($M_{\mathrm{crit}}{\approx}4$), or with a model class missing the interaction order $q$, OOD collapses (\Cref{sec:scope}).
The CDS head is weakest on the GC-skewed \textit{M.\,tuberculosis} (AUROC $0.750$).

\textbf{Future work.}
A finite-sample OOD bound in terms of sparse identifiability ($\delta_{\mathrm{OOD}}$, \Cref{sec:exp32}) would extend \Cref{lem:canon} beyond exact realisability, and near-boundary data could drive automatic discovery of $\varphi$.

\begin{ack}
We thank the anonymous reviewers for comments that improved the paper.
Computations were carried out on the Euler cluster at ETH Zurich.
This work received no specific grant from any funding agency.

Code, per-seed results and the Lean~4 proof library are available at
\url{https://github.com/leaguilar/DoesYourNeuralNetworkExtrapolate} and archived at \url{https://doi.org/10.5281/zenodo.24019953}.
\end{ack}

\bibliographystyle{plainnat}
\bibliography{references}

\appendix
\setcounter{theorem}{0}
\renewcommand{\theHproposition}{appendix.\arabic{proposition}}
\renewcommand{\theHcorollary}{appendix.\arabic{corollary}}
\section{ERM is blind to observationally-equivalent processes}
\label{app:proof_prop_erm}



\begin{proposition}[ERM is blind to observationally-equivalent processes, restated]
Let $D_n = \{(x_i, y_i)\}_{i=1}^n$ with $x_i \stackrel{\mathrm{iid}}{\sim} \Dtrain$ on $W$ and $y_i = P_k(x_i) + \eta_i$, $\eta_i \stackrel{\mathrm{iid}}{\sim} \mathcal{N}(0, \sigma^2)$ independently of the inputs, with $\sigma>0$, for unknown $k \in \{1, 2\}$ where $P_1, P_2$ are $\eps_W$-observationally equivalent on $W$ in the sense of \Cref{def:obseq}.
Write $t_W := \sqrt{n}\,\eps_W/\sigma$.
For any (deterministic or randomised) test $T: D_n \to \{1, 2\}$,
\begin{equation}
\label{eq:app-prop1-l2}
  \min_{k \in \{1,2\}} \Pr_{D_n \sim \mathcal{D}_k^n}[T(D_n) = k]
  \;\le\; \Phi(t_W/2),
\end{equation}
with $\Phi$ the standard normal CDF. The fixed-design analogue is attained by the likelihood-ratio test (with fair randomisation at ties) with $t_W$ replaced by $S_n/\sigma$ from \eqref{eq:gaussian-two-point}.
In particular, in the regime $\eps_W\sqrt{n} \ll \sigma$, no $W$-supported test reliably distinguishes $P_1$ from $P_2$.
Selection criteria such as cross-validation, training loss, or held-out loss on $W$ are special cases.
\end{proposition}

\begin{proof}
Let $\mathcal{D}_k$ denote the joint distribution of a single observation $(x, y)$ under hypothesis $P_k$, \ie $x \sim \Dtrain$ and $y \mid x \sim \mathcal{N}(P_k(x), \sigma^2)$, and let $\mathcal{D}_k^n$ be the product measure over $n$ i.i.d.\ samples.
Write $\Delta := P_1 - P_2$.

\emph{Step 1: reduction to the average error.}
For any test $T$,
\begin{equation}
  \min_{k} \Pr_k[T = k]
  \;\le\; \tfrac{1}{2}\bigl(\Pr_1[T = 1] + \Pr_2[T = 2]\bigr),
\end{equation}
so it suffices to bound the optimal average success probability under equal priors.

\emph{Step 2: the conditional problem is an exact Gaussian two-point problem.}
Condition on the design $X = (x_1, \dots, x_n)$.
Under hypothesis $k$ the response vector is $Y \sim \mathcal{N}(\mu_k, \sigma^2 I_n)$ with $\mu_k = (P_k(x_i))_{i=1}^n$, and the two covariances coincide.
For two simple Gaussian hypotheses with common covariance $\sigma^2 I_n$ and equal priors, the Neyman--Pearson test is optimal and its average success probability is exactly
\begin{equation}
\label{eq:gaussian-two-point}
  \Phi\!\left(\frac{S_n}{2\sigma}\right),
  \qquad
  S_n^2 \;:=\; \norm{\mu_1 - \mu_2}_2^2 \;=\; \sum_{i=1}^n \Delta(x_i)^2 .
\end{equation}
No test conditional on $X$ can do better.

\emph{Step 3: averaging over the design.}
Taking expectations over $X$,
\begin{equation}
  \tfrac{1}{2}\bigl(\Pr_1[T = 1] + \Pr_2[T = 2]\bigr)
  \;\le\; \E_X\!\left[\Phi\!\left(\frac{S_n}{2\sigma}\right)\right].
\end{equation}
Let $g(v) := \Phi(\sqrt{v}/(2\sigma))$ for $v \ge 0$. Differentiating twice,
\begin{equation}
  g''(v) \;=\; -\,\frac{\phi\bigl(\sqrt{v}/(2\sigma)\bigr)}{4\sigma}
  \left[\frac{1}{8\sigma^2 \sqrt{v}} + \frac{1}{2 v^{3/2}}\right] \;<\; 0 ,
\end{equation}
so $g$ is concave on $(0, \infty)$. Since the $x_i$ are i.i.d.\ from $\Dtrain$,
\begin{equation}
  \E_X\bigl[S_n^2\bigr] \;=\; n\,\E_{x \sim \Dtrain}\bigl[\Delta(x)^2\bigr]
  \;=\; n\,\norm{\Delta}_{L^2(\Dtrain)}^2 \;\le\; n\,\eps_W^2 ,
\end{equation}
and Jensen's inequality gives
\begin{equation}
  \E_X\bigl[g(S_n^2)\bigr] \;\le\; g\bigl(\E_X[S_n^2]\bigr)
  \;\le\; \Phi\!\left(\frac{\sqrt{n}\,\eps_W}{2\sigma}\right)
  \;=\; \Phi(t_W/2). \qedhere
\end{equation}
\end{proof}

\paragraph{Linear form.}
Under the hypotheses of the proposition,
\begin{equation}
\label{eq:app-prop1-linear}
  \min_{k} \Pr_k[T = k]
  \;\le\; \frac{1}{2} + \frac{\eps_W\sqrt{n}}{2\sqrt{2\pi}\,\sigma}
  \;=\; \frac{1}{2} + \frac{t_W}{2\sqrt{2\pi}}.
\end{equation}
Indeed the density of $\Phi$ is maximised at $0$, so $\Phi(u) \le \tfrac{1}{2} + u\,\phi(0) = \tfrac{1}{2} + u/\sqrt{2\pi}$ for $u \ge 0$; apply this to \eqref{eq:app-prop1-l2} with $u = t_W/2$.
This form is convenient for locating the boundary between the two regimes below: its right-hand side reaches $1$, and the bound therefore stops constraining any test, exactly at $\sigma = \sqrt{n}\,\eps_W/\sqrt{2\pi}$.

\paragraph{Relation to the Le Cam--Pinsker route.}
A bound of the same shape follows from the classical route, Le Cam's two-point method with Pinsker's inequality~\citep[Lemma 2.5]{tsybakov2009nonparametric}.
For the Gaussian shift model the per-sample Kullback--Leibler divergence is exactly $\norm{\Delta}_{L^2(\Dtrain)}^2/(2\sigma^2)$, so Pinsker gives $\mathrm{TV}(\mathcal{D}_1^n, \mathcal{D}_2^n) \le \sqrt{n}\,\eps_W/(2\sigma) = t_W/2$ and hence $\min_k \Pr_k[T=k] \le \tfrac{1}{2} + t_W/4$.
That route relaxes in Pinsker, which is loose for Gaussian shifts, whereas \eqref{eq:app-prop1-l2} is sharp: the constant improves from $1/4$ to $1/(2\sqrt{2\pi}) \approx 0.199$, and the exact form stays informative over a wider range of $(\eps_W, n, \sigma)$.
As an illustration, at $n = 200$, $\sigma = 0.5$ and $\eps_W = 0.071$ the clipped Pinsker bound returns $1.000$ (vacuous) while $\Phi(t_W/2) = 0.842$, which is the exact optimum when $|P_1-P_2|=\eps_W$ almost surely.

\paragraph{Remark on the regime.}
The quantity $t_W = \eps_W\sqrt{n}/\sigma$ is the effective signal-to-noise ratio after $n$ samples (signal scale $\eps_W$, noise scale $\sigma/\sqrt{n}$).
The bound is informative when $\eps_W\sqrt{n} \ll \sigma$ (the \emph{noise-dominated} regime, where the per-sample gap is drowned by noise; this is the proposition's content and the regime characteristic of finite-window OOD problems) and vacuous (trivially $\to 1$) when $\eps_W\sqrt{n} \gg \sigma$ (the \emph{signal-dominated} regime, where $n$ samples \emph{can} reveal the gap).
Selection criteria (cross-validation, training loss, held-out loss on $W$) are functions of the same $W$-supported sample, hence special cases of a $W$-supported test, and the bound applies to them by inclusion.

\paragraph{Provenance.}
The proof technique (Le Cam's two-point reduction, evaluated with the exact Gaussian testing error rather than Pinsker's inequality) is classical \citep{tsybakov2009nonparametric} and the multiplicity it characterises is the Rashomon effect of \citet{breiman2001statistical, xin2022exploring}. Our contribution is the application to single-window OOD identifiability under the $(\varphi, \psi, \mathcal{M})$ commitment, not the inequality.
The toy battery (Exp.~1.1, 1.2) and Exp.~3.1 sit in the signal-dominated regime where the bound is vacuous; the multiplicity it captures still binds in practice because the candidate hypothesis classes are observationally near-equivalent on $W$, and ERM/CV select among them with the residual ambiguity that follows (Exp.~3.1, \Cref{sec:exp31}).

\subsection{Window--strip separation: the finite-sample counterpart}
\label{app:prop_strip}

\Cref{cor:nearboundary} is a purely geometric floor: a near-boundary strip must be at least $\sqrt{3}\,\eps_W/r$ wide before the divergence rate alone guarantees $\norm{\Delta}_{L^2(\mathcal{D}_V)} \ge \eps_W$, and by the tightness construction no smaller width does.
It is silent about the two quantities a practitioner actually controls, namely how many points the strip carries and how much noise they carry.
The proposition restated here supplies that counterpart, placing the window and the strip on a single scale; \Cref{cor:nearboundary} follows from its part (c).

\setcounter{theorem}{1}
\begin{proposition}[Window--strip separation, restated]
Let $P_1, P_2$ be $\eps_W$-observationally equivalent on $W$ under $\Dtrain$ (\Cref{def:obseq}), observed through $n$ samples, and let $\Delta := P_1 - P_2$.
Let $V = [b, b+\delta_V]$ be a near-boundary segment carrying $n_V$ design points drawn from $\mathcal{D}_V$ under independent Gaussian noise at the same level $\sigma>0$.
Define the two signal-to-noise ratios
\begin{equation}
  t_W \;:=\; \frac{\sqrt{n}\,\eps_W}{\sigma},
  \qquad
  t_V \;:=\; \frac{\sqrt{n_V}\,\norm{\Delta}_{L^2(\mathcal{D}_V)}}{\sigma},
\end{equation}
where in (b) and (d) the norm defining $t_V$ is taken under the empirical measure of the realized design, and in (c) under the uniform distribution on $V$ (\Cref{rem:strip-measure}).
Then:
\begin{enumerate}[label=(\alph*),leftmargin=*,topsep=2pt,itemsep=2pt]
  \item every test computed on the window has worst-case correct-decision probability at most $\Phi(t_W/2)$;
  \item conditional on the design, the likelihood-ratio test on $V$ (with fair randomisation at ties) is correct with probability exactly $\Phi(t_V/2)$;
  \item if $|\Delta(b+u)| \ge r u$ for $u \in (0, \delta_V]$ and $\mathcal{D}_V$ is uniform on $V$, then
    \begin{equation}
      t_V \;\ge\; \frac{r\,\delta_V\,\sqrt{n_V}}{\sqrt{3}\,\sigma},
    \end{equation}
    with equality for $\Delta(x) = \max\bigl(0,\, r(x - b)\bigr)$;
  \item consequently, for any $\alpha \in (0, \tfrac{1}{4})$ the segment resolves what the window cannot, in the sense that the window bound is at most $\tfrac{1}{2} + \alpha$ while the segment is at least $1 - \alpha$, as soon as
    \begin{equation}
      t_W \;\le\; 2\Phi^{-1}\bigl(\tfrac{1}{2} + \alpha\bigr)
      \qquad\text{and}\qquad
      t_V \;\ge\; 2\Phi^{-1}(1 - \alpha).
    \end{equation}
\end{enumerate}
\end{proposition}

\begin{proof}
Part (a) is \eqref{eq:app-prop1-l2}.

For (b), condition on the $n_V$ design points of $V$.
Under hypothesis $k$ the responses on $V$ are $\mathcal{N}(\mu_k^V, \sigma^2 I_{n_V})$ with $\mu_k^V = (P_k(v_i))_i$, so by the same Gaussian two-point computation \eqref{eq:gaussian-two-point} the likelihood-ratio test (with fair randomisation at ties) is correct with probability exactly $\Phi\bigl(S_V/(2\sigma)\bigr)$, where
\begin{equation}
  S_V^2 \;=\; \sum_{i=1}^{n_V} \Delta(v_i)^2 \;=\; n_V \norm{\Delta}_{L^2(\mathcal{D}_V)}^2 ,
\end{equation}
the last equality being the definition of the empirical $L^2$ norm under $\mathcal{D}_V$. Hence $S_V/\sigma = t_V$ and the probability is $\Phi(t_V/2)$.
When $S_V>0$, the log-likelihood ratio is Gaussian with mean $\pm S_V^2/(2\sigma^2)$ and variance $S_V^2/\sigma^2$, and it errs exactly when it crosses zero.

For (c), uniformity of $\mathcal{D}_V$ on $[b, b+\delta_V]$ and the divergence hypothesis give
\begin{equation}
  \norm{\Delta}_{L^2(\mathcal{D}_V)}^2
  \;=\; \frac{1}{\delta_V}\int_0^{\delta_V} \Delta(b+u)^2 \, du
  \;\ge\; \frac{1}{\delta_V}\int_0^{\delta_V} r^2 u^2 \, du
  \;=\; \frac{r^2 \delta_V^2}{3},
\end{equation}
so $\norm{\Delta}_{L^2(\mathcal{D}_V)} \ge r \delta_V/\sqrt{3}$ and the claim follows on multiplying by $\sqrt{n_V}/\sigma$.
The extremal $\Delta(x) = \max(0, r(x-b))$ attains the inequality with equality and vanishes identically on $W$, so it also satisfies the hypothesis of (a) for every $\eps_W \ge 0$.

For (d), $\Phi$ is increasing, so $t_W \le 2\Phi^{-1}(\tfrac{1}{2}+\alpha)$ gives $\Phi(t_W/2) \le \tfrac{1}{2}+\alpha$ and $t_V \ge 2\Phi^{-1}(1-\alpha)$ gives $\Phi(t_V/2) \ge 1-\alpha$.
The window ceiling lies strictly below the segment floor exactly when $1 - \alpha > \tfrac{1}{2} + \alpha$, \ie $\alpha < \tfrac{1}{4}$.
\end{proof}

\begin{remark}[What each hypothesis buys]
\label{rem:strip-roles}
$\eps_W$ enters (a) and (d), so the observational-equivalence hypothesis is load-bearing rather than decorative: it is what forces the window side of the separation.
Width, point count and noise enter only through the single scalar $t_V$, which is the finite-sample governance \Cref{cor:nearboundary} lacks.
\Cref{cor:nearboundary} is recovered as the geometric content of (c): $t_V > 0$ requires $\delta_V > 0$, and the $\sqrt{3}\,\eps_W/r$ threshold is what guarantees $\norm{\Delta}_{L^2(\mathcal{D}_V)} \ge \eps_W$.
\end{remark}

\begin{remark}[Which measure the subscript names]
\label{rem:strip-measure}
In (b) and (d), $\norm{\Delta}_{L^2(\mathcal{D}_V)}$ is taken under the empirical measure of the realized design, which is what makes the conditional statements exact.
In (c) and \Cref{cor:nearboundary}, $\mathcal{D}_V$ is the uniform distribution on $V$, the population idealization of that design, whose squared norm is the design average of the empirical one: $\E[S_V^2] = n_V \norm{\Delta}^2_{L^2(\mathcal{D}_V)}$ for $n_V$ i.i.d.\ uniform points.
The empirical norm converges to the population norm for i.i.d.\ uniform points when $\Delta$ is square-integrable. For the fixed equally spaced design $v_i = b + i\,\delta_V/n_V$, the lower envelope $\Delta(v_i)^2 \ge r^2(i\delta_V/n_V)^2$ gives the lower bound in (c) for the empirical norm as well, since $n_V^{-1}\sum_{i=1}^{n_V}(i/n_V)^2 \ge 1/3$. This does not imply that the empirical norm dominates the population norm for arbitrary $\Delta$.
No such device is needed in (a): there the bound is an upper bound, so Jensen's inequality delivers it directly from the population norm.
\end{remark}

\begin{remark}[Beyond linear divergence]
\label{rem:strip-monomial}
The linear hypothesis in (c) is inherited from \Cref{cor:nearboundary} and is not essential.
If instead $|\Delta(b+u)| \ge r u^{p}$ on $(0, \delta_V]$ for some $p > 0$, the same integral gives
\begin{equation}
  \norm{\Delta}_{L^2(\mathcal{D}_V)} \;\ge\; \frac{r\,\delta_V^{p}}{\sqrt{2p+1}},
  \qquad\text{hence}\qquad
  t_V \;\ge\; \frac{r\,\delta_V^{p}\sqrt{n_V}}{\sqrt{2p+1}\,\sigma},
\end{equation}
with (c) the case $p = 1$.
This matters in practice because the gap between a periodic target and its truncated Taylor expansion grows with an exponent slightly above $1$ near the boundary, so the linear form holds on $[u_{\min}, \delta_V]$ rather than on an interval open at $b$.
\end{remark}

\paragraph{Numerical verification.}
Parts (a) and (b) were checked against Monte Carlo simulation of the corresponding two-point problems and agree to four decimal places.
\Cref{app:experiments} reports the empirical counterpart: a sweep over segment width, point count and noise in which selection accuracy is governed by $t_V$ alone, with the practical threshold at $t_V \approx 6$ for the relative-error selection rule used in Exp.~3.1.

\section{The geometric floor}
\label{app:proof_prop_nearboundary}
\setcounter{theorem}{2}

\begin{corollary}[Geometric floor, restated]
Let $P_1, P_2$ be $\eps_W$-observationally equivalent on $W$ (\Cref{def:obseq}) and diverge at rate $r > 0$ past the boundary $b$: $|\Delta(b + u)| \ge r u$ for $u \in (0, \delta_V]$, with $\Delta = P_1 - P_2$ and $\mathcal{D}_V$ uniform on $V = [b, b + \delta_V]$.
Then $\norm{\Delta}_{L^2(\mathcal{D}_V)} \ge r\,\delta_V/\sqrt{3}$, hence
\begin{equation}
  \delta_V \;\ge\; \sqrt{3}\,\eps_W/r
  \qquad\Longrightarrow\qquad
  \norm{\Delta}_{L^2(\mathcal{D}_V)} \;\ge\; \eps_W,
  \qquad\text{equivalently}\qquad
  t_V \;\ge\; t_W\sqrt{n_V/n}.
\end{equation}
The threshold is tight.
\end{corollary}

\begin{proof}
Uniformity of $\mathcal{D}_V$ and the divergence assumption give
\begin{equation}
  \norm{\Delta}_{L^2(\mathcal{D}_V)}^2
  \;=\; \frac{1}{\delta_V}\int_0^{\delta_V} \Delta(b+u)^2 \, du
  \;\ge\; \frac{1}{\delta_V}\int_0^{\delta_V} r^2 u^2 \, du
  \;=\; \frac{r^2\delta_V^2}{3},
\end{equation}
so $\norm{\Delta}_{L^2(\mathcal{D}_V)} \ge r\delta_V/\sqrt{3}$. Substituting $\delta_V \ge \sqrt{3}\,\eps_W/r$ gives $\norm{\Delta}_{L^2(\mathcal{D}_V)} \ge \eps_W$, and multiplying by $\sqrt{n_V}/\sigma$ turns this into $t_V \ge \sqrt{n_V}\,\eps_W/\sigma = t_W\sqrt{n_V/n}$.
\end{proof}

\paragraph{Tightness (existence).}
The threshold $\delta_V = \sqrt{3}\,\eps_W/r$ is sharp.
Take $P_1 \equiv 0$ on $\R$ and $P_2(x) = \max(0,\, r(x - b))$, which vanishes on $W$ and so is $\eps_W$-observationally equivalent for every $\eps_W \ge 0$, while saturating the divergence assumption with rate $r$.
For it, $\norm{\Delta}_{L^2(\mathcal{D}_V)} = r\delta_V/\sqrt{3}$ exactly, so any $\delta_V' < \sqrt{3}\,\eps_W/r$ gives $\norm{\Delta}_{L^2(\mathcal{D}_{V'})} < \eps_W$: no smaller width guarantees a population gap of at least $\eps_W$. This does not rule out detection below that width.

\textbf{Two regimes.} \Cref{prop:erm} governs the noisy regime $\sigma>0$. At $\sigma=0$, known candidates are indistinguishable only when they agree at the observed inputs. \Cref{cor:nearboundary} gives a sufficient population-gap width, while \Cref{prop:strip} quantifies finite-sample detection.

\textbf{Constructive existence.}
\emph{When $\mathcal{M}$ contains $g^\star$ with $g^\star \circ \varphi = f$ on all of $\X$ (not just $W$) and ERM identifies $g^\star$ from $n$ samples, $\eps_{\mathrm{OOD}}(\hat{g}_n) = 0$ on any $\Dtest$.}
The sinusoidal map $\varphi(x) = (\sin x, \cos x)$ instantiates this: ordinary least squares on $n \ge 2$ noiseless i.i.d.\ samples from an absolutely continuous input distribution recovers $\mathbf{w} = (1, 0)$ with probability $1$ and extrapolates exactly, as also shown for log-log and log-y transformations (\Cref{lem:canon,lem:sin,lem:exp,lem:powerlaw,cor:canon-recovery}, proofs in \Cref{app:proof_canon}).

\section{Realisable-case canonicalisation and three linearising instances}
\label{app:proof_canon}
\setcounter{theorem}{3}

These three lemmas instantiate the $(\varphi, \psi, \mathcal{M})$ commitment of \Cref{sec:theory} with classical examples: the math is textbook (linear OLS in transformed coordinates), and the contribution is the unified framing under realisability.

\begin{lemma}[Realisable-case canonicalisation]
\label{lem:canon}
Let $f: \X \to \R$, $\varphi: \X \to \Zx$ a feature map, $\psi: f(\X) \to \R$ a strictly monotonic (hence invertible) label map, and $\mathcal{M}$ a hypothesis class on $\Zx$. Suppose
\begin{enumerate}[label=(\roman*),leftmargin=*,topsep=2pt,itemsep=1pt]
  \item \emph{realisability}: there exists $g^\star \in \mathcal{M}$ such that
    \begin{equation}
      g^\star(\varphi(x)) = \psi(f(x)) \quad \forall\, x \in \X
      \quad \text{(not merely on $W$);}
    \end{equation}
  \item \emph{noise-free finite-sample identifiability}: ERM on $\mathcal{M}$ from $n$ i.i.d.\ samples $\{(\varphi(x_i), \psi(y_i))\}_{i=1}^n$ with $x_i \sim \Dtrain$ and $y_i = f(x_i)$ recovers $g^\star$ with probability $1$, $\hat{g}_n = g^\star$.
\end{enumerate}
Then the predictor $\hat{f}_n = \psi^{-1} \circ \hat{g}_n \circ \varphi$ satisfies
\begin{equation}
  \eps_{\mathrm{OOD}}(\hat{f}_n) = 0 \quad \text{for any test distribution } \Dtest \text{ on } \X.
\end{equation}
The case $\psi = \mathrm{id}$ recovers the standard formulation.
\end{lemma}

\begin{proof}
By realisability, $\psi \circ f = g^\star \circ \varphi$ on $\X$. Composing with $\psi^{-1}$ gives
\begin{equation}
  f = \psi^{-1} \circ g^\star \circ \varphi \quad \text{on } \X,
\end{equation}
so $\eps_{\mathrm{OOD}}(\psi^{-1} \circ g^\star \circ \varphi) = 0$ for every $\Dtest$.
By identifiability, $\hat{g}_n = g^\star$ with probability $1$, hence
\begin{equation}
  \hat{f}_n = \psi^{-1} \circ \hat{g}_n \circ \varphi = \psi^{-1} \circ g^\star \circ \varphi = f \quad \text{w.p.~$1$,}
\end{equation}
and $\eps_{\mathrm{OOD}}(\hat{f}_n) = 0$.
\end{proof}

\paragraph{Discussion.}
Realisability is the load-bearing assumption: absent it, $\hat{g}_n$ converges to the $\mathcal{M}$-best approximant on $\Dtrain$, which need not extrapolate.
Identifiability is a generic regularity condition; for the linear instances below it reduces to invertibility of the design Gram matrix.
The label map $\psi$ canonicalises targets that are not naturally linear in $\varphi(x)$ (\eg $\psi = \log$ for multiplicative families), at the cost of a noise-model caveat addressed in \Cref{rem:noise-log}.

\begin{lemma}[Sinusoidal instance: Fourier features]
\label{lem:sin}
Let $f(x) = \sin x$ on $\R$ and $\varphi(x) = (\sin x, \cos x): \R \to \mathbb{S}^1 \subset \R^2$.
The linear function $g(u, v) = \mathbf{w}^\top (u, v)$ with $\mathbf{w} = (1, 0)$ satisfies
\begin{equation}
  g(\varphi(x)) = \sin x = f(x) \quad \forall\, x \in \R,
\end{equation}
so \Cref{lem:canon} applies with $\psi = \mathrm{id}$, $\mathcal{M}$ the linear functions on $\R^2$, and $g^\star = (1, 0)^\top$; hence $\eps_{\mathrm{OOD}}(g^\star) = 0$ for any $\Dtest$ supported on $\R$.
\end{lemma}

\begin{proof}
Realisability is the displayed identity; \Cref{lem:canon}(i) is satisfied with $\psi = \mathrm{id}$.
\end{proof}

\begin{remark}[Phase-amplitude generalisation]
\label{rem:sin-general}
The same instance with $g^\star = (A\cos\phi, A\sin\phi)^\top$ realises any $f(x) = A\sin(x + \phi)$, since
\begin{equation}
  A\sin(x + \phi) = A\cos\phi \cdot \sin x + A\sin\phi \cdot \cos x = g^{\star\top}\varphi(x).
\end{equation}
The 2-parameter family $\{A\sin(x + \phi) : A \in \R, \phi \in [0, 2\pi)\}$ is exactly the linear span of $\{\sin x, \cos x\}$, so realisability is automatic for any $A, \phi$.
\end{remark}

\begin{lemma}[Power-law instance: log-log features]
\label{lem:powerlaw}
Let $f(x) = a x^\beta$ for $a > 0$, $\beta \in \R$, on $\X = \R_{>0}$, with $\varphi(x) = (1, \log x): \R_{>0} \to \R^2$ and $\psi(y) = \log y: \R_{>0} \to \R$.
The linear function $g^{\star\top}(u, v) = (\log a)\,u + \beta\,v$, with $g^\star = (\log a, \beta)^\top$, satisfies
\begin{equation}
  g^{\star\top}\varphi(x) = \log a + \beta \log x = \log(a x^\beta) = \psi(f(x)) \quad \forall\, x \in \R_{>0},
\end{equation}
so \Cref{lem:canon} applies with $\mathcal{M}$ the linear functions on $\R^2$; hence $\eps_{\mathrm{OOD}}(\psi^{-1} \circ g^\star \circ \varphi) = 0$ for any $\Dtest$ supported on $\R_{>0}$.
\end{lemma}

\begin{proof}
Realisability is the displayed identity; \Cref{lem:canon}(i) is satisfied with $\psi = \log$.
\end{proof}

\begin{lemma}[Exponential instance: log-$y$ features]
\label{lem:exp}
Let $f(x) = a e^{\beta x}$ for $a > 0$, $\beta \in \R$, on $\X = \R$, with $\varphi(x) = (1, x): \R \to \R^2$ and $\psi(y) = \log y: \R_{>0} \to \R$.
The linear function $g^{\star\top}(u, v) = (\log a)\,u + \beta\,v$, with $g^\star = (\log a, \beta)^\top$, satisfies
\begin{equation}
  g^{\star\top}\varphi(x) = \log a + \beta x = \log(a e^{\beta x}) = \psi(f(x)) \quad \forall\, x \in \R,
\end{equation}
so \Cref{lem:canon} applies with $\mathcal{M}$ the linear functions on $\R^2$; hence $\eps_{\mathrm{OOD}}(\psi^{-1} \circ g^\star \circ \varphi) = 0$ for any $\Dtest$ supported on $\R$.
\end{lemma}

\begin{proof}
Realisability is the displayed identity; \Cref{lem:canon}(i) is satisfied with $\psi = \log$.
\end{proof}

\begin{corollary}[Unified OLS recovery in canonical coordinates]
\label{cor:canon-recovery}
For each instance \Cref{lem:sin,lem:powerlaw,lem:exp}, let $\Dtrain$ be absolutely continuous with respect to Lebesgue measure on a window $W$ of positive Lebesgue measure (with $W \subseteq \R_{>0}$ for \Cref{lem:powerlaw} and $W \subseteq \R$ otherwise).
Assume noiseless labels $y_i=f(x_i)$. Then ordinary least squares on $n \ge 2$ i.i.d.\ samples $\{(\varphi(x_i), \psi(y_i))\}_{i=1}^n$ recovers $g^\star$ with probability $1$:
\begin{equation}
  \hat{g}_n = g^\star \;\;\text{w.p.~$1$},
  \qquad
  \eps_{\mathrm{OOD}}(\hat{f}_n) = 0 \quad \text{for any } \Dtest,
\end{equation}
where $\hat{f}_n = \psi^{-1} \circ \hat{g}_n \circ \varphi$.
\end{corollary}

\begin{proof}
We verify identifiability (\Cref{lem:canon}(ii)) for OLS on $n \ge 2$ samples in each instance.
The OLS estimator is
\begin{equation}
  \hat{g}_n = (\Phi^\top \Phi)^{-1} \Phi^\top \mathbf{z},
\end{equation}
with design matrix $\Phi \in \R^{n \times 2}$ having rows $\varphi(x_i)^\top$ and target $\mathbf{z}_i = \psi(y_i) = \psi(f(x_i))$.
By realisability, $\mathbf{z}_i = g^{\star\top}\varphi(x_i)$, so $g^\star$ is a zero-residual solution; uniqueness holds whenever $\Phi^\top \Phi$ is invertible, which fails only if the first two rows of $\Phi$ are collinear.
The collinearity event in each instance is:
\begin{itemize}[leftmargin=*,topsep=2pt,itemsep=1pt]
  \item \emph{Sinusoidal} (\Cref{lem:sin}): $\sin x_i \cos x_j - \sin x_j \cos x_i = \sin(x_i - x_j) = 0$, \ie $x_i \equiv x_j \pmod{\pi}$;
  \item \emph{Power-law} (\Cref{lem:powerlaw}): $\log x_i = \log x_j$, \ie $x_i = x_j$;
  \item \emph{Exponential} (\Cref{lem:exp}): $x_i = x_j$.
\end{itemize}
Each collinearity event is a countable union of measure-$0$ subsets of the joint sample space and so has probability $0$ under absolutely-continuous $\Dtrain$.
Hence $\Phi^\top\Phi$ is invertible w.p.~$1$, $\hat{g}_n = g^\star$, and \Cref{lem:canon} closes the argument.
\end{proof}

\begin{remark}[Reparameterisation between power-law and exponential]
\label{rem:powerlaw-exp}
\Cref{lem:powerlaw,lem:exp} are the same log-linear instance up to the choice of feature map ($\varphi(x) = (1, \log x)$ versus $\varphi(x) = (1, x)$); we present them separately because they correspond to two distinct practitioner recipes (log-log and log-$y$ regression) and to the two diagonal cells of \Cref{tab:exp12_main} in Exp.~1.2.
\end{remark}

\begin{remark}[Noise model under $\psi = \log$]
\label{rem:noise-log}
When $\psi = \log$, additive Gaussian noise on $y$ is non-Gaussian on $\psi(y)$, so OLS in canonical coordinates is the maximum-likelihood estimator only under a multiplicative log-normal noise model, $\log y_i = \log f(x_i) + \eta_i$ with $\eta_i \sim \mathcal{N}(0, \sigma^2)$.
\Cref{lem:canon} and \Cref{cor:canon-recovery} are noise-free realisable identifiability statements: they assert that OLS on the transformed sample has zero residual at $g^\star$ when the noise is zero, and that this remains true with probability $1$ under absolutely-continuous sampling.
The noisy version (rates of convergence under either Gaussian or log-Gaussian noise) is a standard exercise in linear regression and is omitted.
\end{remark}

\section{Delineation from No Free Lunch}
\label{app:nfl}

\Cref{sec:related_work} states the difference in brief. This appendix gives it in full.

\paragraph{Three regimes.}
It helps to separate three worlds.

\begin{enumerate}[label=(\roman*),leftmargin=*,topsep=2pt,itemsep=2pt]
  \item \emph{The NFL world.} Labels are independent across inputs and the target is an incompressible lookup table. Under a uniform prior over all labellings, what has been seen tells you nothing about what has not, and every algorithm has the same off-training-set risk~\citep{wolpert1996lack}. Nothing is learnable here, not even interpolation, and no real machine-learning problem lives in this world.
  \item \emph{The regular world, in distribution.} Smoothness, sparsity, symmetry or a manifold assumption each \emph{couple} inputs to one another, so the posterior separates from the prior and minimax rates exist. NFL is \emph{dissolved}, not evaded: the assumptions remove the free off-sample coordinates its proof requires.
  \item \emph{The regular world, out of distribution.} This is our setting. Regularity is \emph{local} and extrapolation is \emph{global}, so restoring interpolation does not restore extrapolation. The candidates remain distinguishable in principle and indistinguishable from the window.
\end{enumerate}

This work is scoped to regime (iii).

\paragraph{Four formal differences.}
\emph{What must be assumed.} NFL needs an unrestricted family $\mathcal{Y}^{\mathcal{X}}$, or at least a class of unbounded capacity. \Cref{prop:erm} needs exactly two functions.

\emph{Whether regularity helps.} Every standard cure for NFL is already in force in our instance. Take $P_1 = \sin$ and $P_2$ its degree-9 Taylor polynomial. Both are $C^\infty$ with bounded derivatives on $W$, both are sparse in a fixed basis, both are realisable, and neither is ruled out by a low-order prior. They agree on $W = [-\pi, \pi]$ to $\eps_W = 1.6\times10^{-3}$, about $0.2\%$ of $\norm{P_1}_{L^2(\Dtrain)}$, and separate to $\norm{P_1 - P_2}_{L^2} = 2.3\times10^{2}$ on the OOD design over $[\pi, 3\pi]$, roughly $10^{5}$ times as much. Assuming more regularity does not separate them.

\emph{Invariance to the hypothesis class.} \Cref{prop:erm} is a two-point testing bound and never references a hypothesis class, so it holds with the same constants for \emph{every} $\mathcal{H}$ with $\{P_1, P_2\} \subseteq \mathcal{H}$. This is the point at which the qualifier matters. A restriction that shrinks capacity while keeping both candidates, which is what any complexity restriction does here since both are smooth, sparse and low-order, leaves the bound verbatim. A restriction that \emph{excludes} one candidate is not a complexity restriction at all: in our framework the class is generated by the triplet, $\mathcal{H} = \{\psi^{-1} \circ g \circ \varphi : g \in \mathcal{M}\}$, so excluding a candidate means pinning $(\varphi, \psi, \mathcal{M})$. That is precisely the DGP-selecting side information \Cref{prop:erm} says the window cannot supply, and it is what we decline to assume, because which triplet is correct is the unidentifiable quantity.

\emph{What is bounded.} NFL bounds the off-training-set risk of a \emph{learner} under a uniform prior. \Cref{prop:erm} bounds how often any \emph{selection rule} computed on in-distribution data, cross-validation included, picks the right process, and it is quantitative in $n$: it says something while $n \ll \sigma^2/\eps_W^2$ and becomes vacuous above that, whereas NFL returns the same answer at every sample size. Exp.~3.1 is the empirical counterpart, where cross-validation picks the correct map $81.3\%$ of the time and a thin near-boundary segment picks it $100\%$ of the time.

\paragraph{Opposite interventions.}
The cleanest statement of the difference is that the two obstructions respond to opposite moves.

\begin{center}
\small
\begin{tabular}{lll}
  \toprule
  \textbf{Intervention} & \textbf{No Free Lunch} & \textbf{\Cref{prop:erm}} \\
  \midrule
  Assume regularity, restrict capacity & dissolved & unchanged \\
  Enlarge the support past $W$ by $\delta_V \ge \sqrt{3}\,\eps_W/r$ & unchanged & dissolved \\
  \bottomrule
\end{tabular}
\end{center}

\noindent
\citet{sterkenburg2021nofree} arrive at a compatible conclusion from the other direction: standard learning algorithms are model-dependent and admit only a model-relative justification. In our vocabulary, that model is the commitment.

\paragraph{What we do not claim.}
We do not claim that the necessity of an inductive bias is new; it is classical~\citep{mitchell1980need, wolpert1996lack, wolpert1997nofree}. We do not claim that NFL is wrong or inapplicable, only that it is a statement about hypothesis-class complexity while ours is a statement about support geometry, and that the second is not a corollary of the first. We also do not claim NFL is silent about cross-validation: \citet{wolpert1996lack} treats cross-validation explicitly. The objects differ. NFL scores cross-validation's off-training-set risk under a uniform prior; \Cref{prop:erm} bounds its ability to select between two candidate processes from a single window.

\FloatBarrier
\section{Exp.~1.3: Correct Transformation ($(\varphi,\psi)$ with Matching $\mathcal{M}$)}
\label{app:exp13}

\begin{figure}[t]
  \centering
  \includegraphics[width=0.8\linewidth]{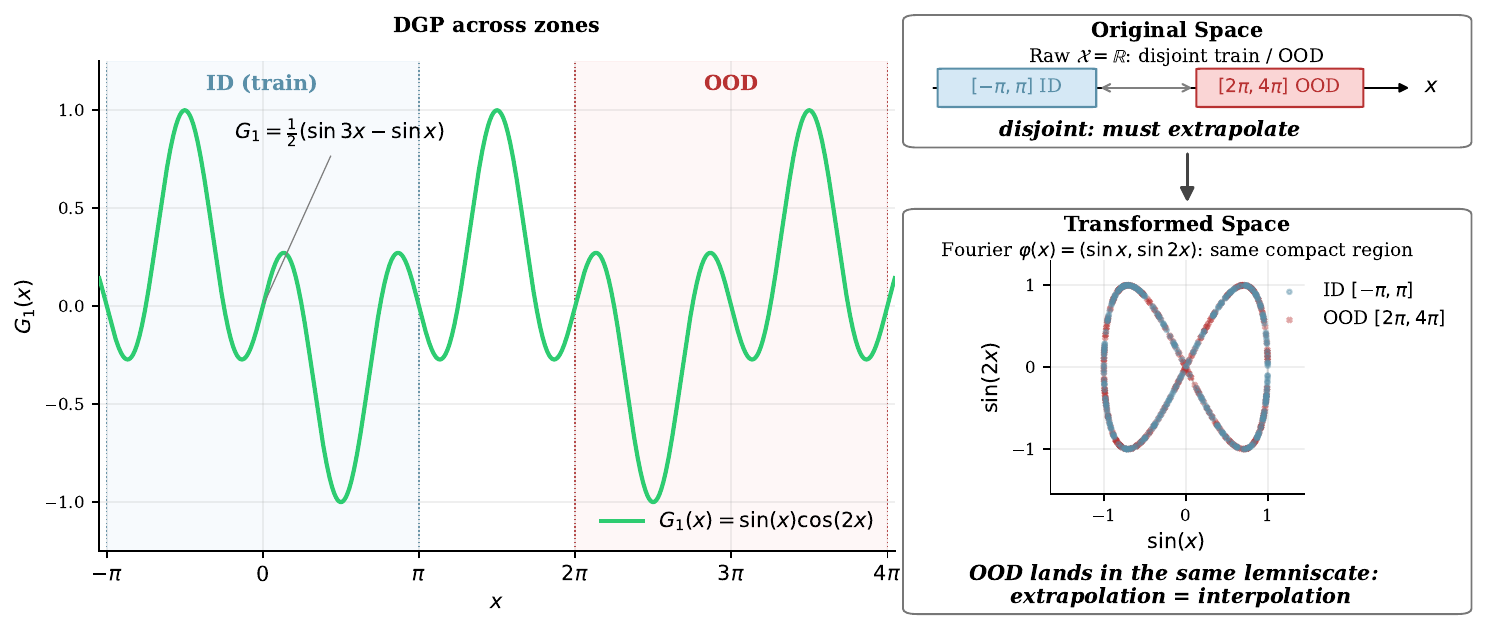}
  \caption{Exp.~1.3 Non linear DGP in the transformed space, mirroring \Cref{fig:exp11_concept}.
    \emph{Left:} $G_1(x) = \sin(x)\cos(2x) = \frac{1}{2}(\sin 3x - \sin x)$ on training $W = [-\pi, \pi]$ and OOD $[2\pi, 4\pi]$.
    \emph{Centre:} in raw coordinates, training and OOD occupy disjoint segments of $\mathbb{R}$.
    \emph{Right:} the Fourier feature map $\varphi(x) = (\sin x, \sin 2x)$ sends both regions onto the same lemniscate; OOD inputs land where training inputs already do, so extrapolation in $x$ becomes interpolation in $\varphi$.\label{fig:exp13_concept}}
\end{figure}

\noindent\textit{Example G1 tests when the target is \emph{nonlinear} in the transformed space.}
For the Fourier family $G_1(x) = \sin(x)\cos(2x)$ on $W = [-\pi, \pi]$, OOD $[2\pi, 4\pi]$ (\Cref{fig:exp13_concept}), apply the product-to-sum identity $\sin A \cos B = \tfrac{1}{2}[\sin(A{+}B) + \sin(A{-}B)]$ with $A = x$, $B = 2x$ to rewrite the target as
\begin{equation*}
  G_1(x) \;=\; \sin(x)\cos(2x) \;=\; \tfrac{1}{2}\bigl[\sin(3x) + \sin(-x)\bigr] \;=\; \tfrac{1}{2}\bigl[\sin 3x - \sin x\bigr].
\end{equation*}
The feature map $\varphi(x) = (\sin x, \cos x, \sin 2x, \cos 2x)$ collapses training $[-\pi,\pi]$ and OOD $[2\pi,4\pi]$ onto the same compact image (whose two-dimensional projection is the lemniscate in \Cref{fig:exp13_concept}, right), so OOD becomes interpolation in $\varphi$.
But the linear-head condition fails: the third-harmonic component $\sin 3x$ in the rewritten target is \emph{not} a coordinate of $\varphi$ and cannot be reached by any linear combination of $\{\sin x, \cos x, \sin 2x, \cos 2x\}$, so OLS on $\varphi$ is parametrically blocked from representing $G_1$ on $W$ in the first place, let alone OOD.
This makes the experiment a joint test of the correct feature map and the model class: the target lies outside the linear span and requires a nonlinear model.
Empirically: OLS fails at ${\sim}82\%$ OOD on the same features at which MLP achieves $2.5\%$, because the MLP can multiply basis entries to synthesise $\sin 3x = 3\sin x - 4\sin^3 x$ (or any other nonlinear route to the missing harmonic) and so closes the gap that the linear head cannot.
For the log-polynomial families ($\log y = 2u - u^2$ where $u = \log x$; $\log y = 1 - 2 e^{-x}$), the targets are respectively quadratic in $u$ and exponential in $x$; OLS fails on the same nonlinearity-in-$\varphi$ ground.

\section{Full SINDy $\delta_{\mathrm{OOD}}$ Diagnostic Battery}
\label{app:sindy_battery}

A degree-$2$ polynomial SINDy library in the \emph{transformed state} represents the governing dynamics exactly (\eg $\frac{dv}{du} = 2 - 2u$ for the first log-polynomial family, identified to ${<}0.01\%$ OOD).
The toy battery fits derivative targets; Exp.~1.1, Exp.~1.2, and Kepler instead fit static regression targets. Computing the OOD residual requires OOD targets or derivatives.
This motivates the diagnostic of \Cref{sec:exp32}: $\delta_{\mathrm{OOD}}$ as the OOD residual in the fitted target space after sparse identification.
\Cref{tab:sindy_delta_consolidated} consolidates the diagnostic across every SINDy-regime experiment in the paper: the toy battery of $35$ (DGP, feature map) pairs across the Fourier and log-polynomial families, Exp.~1.1 ($\sin(x)$ vs.\ Taylor-9 with Fourier and poly-9 libraries, \Cref{sec:exp11}), Exp.~1.2 (power law vs.\ exponential with log-log, log-y, raw maps under a degree-1 polynomial library, \Cref{sec:exp12}), Kepler's third law (\Cref{sec:exp22}, real exoplanet data, $8$ multivariate feature maps), and the bilinear MAK vector field (\Cref{sec:exp21}, synthetic chemistry ODE, $3$ feature maps).
Across the toy battery, $\delta_{\mathrm{OOD}} \le 4.3\times10^{-5}$ for every preassigned canonically-correct $(\text{DGP}, \varphi)$ pair and $\delta_{\mathrm{OOD}} \ge 0.0958$ for every pair labelled wrong: \textbf{Spearman $\rho(\delta_{\mathrm{OOD}}, \text{wrong-label}) = 0.82$ ($p = 1.3 \times 10^{-9}$, $n = 35$ pairs).}
The four additional experiments (Exp.~1.1, 1.2, Kepler, MAK) add test points where every canonical commitment achieves $\delta_{\mathrm{OOD}} \le 0.012$ and wrong feature maps reach $\delta_{\mathrm{OOD}}$ between $0.037$ and $6.7{\times}10^{5}$.
We treat $\delta_{\mathrm{OOD}}$ as \emph{shown to correlate} with the canonically-correct $\varphi$, not as a guaranteed selector. These are preassigned labels, not an exhaustive classification of representability: for example, $d(\sin 2x)/dx=2(\cos^2x-\sin^2x)$ lies in the degree-$2$ library on $(\sin x,\cos x)$, although that fitted row is labelled wrong. The reported $p$-value is nominal because pairs share DGPs and features. The augmented Kepler map containing $(\log a,\log M_{\mathrm{star}})$ also contains the canonical variables and is not a wrong-map control.
The diagnostic is silent on non-polynomial dynamics: for $\tanh(5\sin x \cos 2x)$ none of the tested polynomial libraries represents the dynamics exactly and $\delta_{\mathrm{OOD}} > 0$ throughout those libraries, signalling ``no exact fit in the tested libraries'' without ranking the candidates.

\begin{table}[ht]
  \centering
  \small
  \caption{Consolidated SINDy $\delta_\mathrm{OOD}$ across all SINDy-regime experiments \citep{brunton2016sindy}. $\delta_\mathrm{tr}$ and $\delta_\mathrm{OOD}$: squared-NMSE residual in the SINDy fit space. \colorbox{gray!12}{\strut Shaded rows} use a preassigned canonical coordinate; bolded $\delta_\mathrm{OOD} \approx 0$ for the designated canonical rows. $G_2 = \tanh(\cdot)$ has no exact representation in the tested libraries (diagnostic silent). The noiseless toy diagnostic is deterministic; repeated seeds give the same displayed values. $k$ counts nonzero coefficients. Spearman $\rho(\delta_\mathrm{OOD},\,\text{wrong-label}) = 0.82$ ($p=1.32\times 10^{-9}$, $n=35$ pairs in toy battery). Toy battery; additional experiments continue below.}
  \label{tab:sindy_delta_consolidated}
  \begin{tabular}{llrrr}
    \toprule
    \textbf{DGP / Experiment} & \textbf{Feature map} & $\boldsymbol{k}$ & $\boldsymbol{\delta_\mathrm{tr}}$ & $\boldsymbol{\delta_\mathrm{OOD}}$ \\
    \midrule
    $\sin(x)$ & Raw $x$ & 2 & 0.0006 & 0.9710 \\
    \rowcolor{gray!12}
    $\sin(x)$ & $(\sin x,\cos x)$ (deg.\ 2) & 4 & 0.0000 & \textbf{0.0000} \\
    $\sin(x)$ & $(\sin 2x,\cos 2x)$ (deg.\ 2) & 6 & 0.0000 & 3.62 \\
    $\sin(x)$ & Poly-2 & 2 & 0.0000 & 54.31 \\
    $\sin(x)$ & Poly-3 & 3 & 0.0000 & 11.61 \\
    \midrule
    $\sin(2x)$ & Raw $x$ & 2 & 0.0070 & 26.65 \\
    $\sin(2x)$ & $(\sin x,\cos x)$ (deg.\ 2) & 6 & 0.0000 & 0.0958 \\
    \rowcolor{gray!12}
    $\sin(2x)$ & $(\sin 2x,\cos 2x)$ (deg.\ 2) & 4 & 0.0000 & \textbf{0.0000} \\
    $\sin(2x)$ & Poly-2 & 3 & 0.0003 & 332 \\
    $\sin(2x)$ & Poly-3 & 4 & 0.0000 & 1285 \\
    \midrule
    \rowcolor{gray!12}
    $x^2$ & Raw $x$ & 1 & 0.0000 & \textbf{0.0000} \\
    $x^2$ & $(\sin x,\cos x)$ (deg.\ 2) & 6 & 0.0000 & 0.8548 \\
    \rowcolor{gray!12}
    $x^2$ & Poly-2 & 1 & 0.0000 & \textbf{0.0000} \\
    \rowcolor{gray!12}
    $x^2$ & Poly-3 & 1 & 0.0000 & \textbf{0.0000} \\
    \midrule
    $x^3+x$ & Raw $x$ & 2 & 0.0050 & 0.5606 \\
    $x^3+x$ & $(\sin x,\cos x)$ (deg.\ 2) & 6 & 0.0000 & 0.5641 \\
    \rowcolor{gray!12}
    $x^3+x$ & Poly-2 & 2 & 0.0000 & \textbf{0.0000} \\
    \rowcolor{gray!12}
    $x^3+x$ & Poly-3 & 2 & 0.0000 & \textbf{0.0000} \\
    \midrule
    $\sin(x)\cos(2x)$ & Raw $x$ & 0 & 1.0000 & 1.0000 \\
    \rowcolor{gray!12}
    $\sin(x)\cos(2x)$ & Fourier-4 (deg.\ 2) & 3 & 0.0000 & \textbf{0.0000} \\
    \rowcolor{gray!12}
    $\sin(x)\cos(2x)$ & Fourier-6 (deg.\ 2) & 4 & 0.0000 & \textbf{0.0000} \\
    $\sin(x)\cos(2x)$ & Poly-3 & 2 & 0.9638 & 26.46 \\
    \midrule
    $\tanh(5\sin(x)\cos(2x))$ & Raw $x$ & 0 & 1.0000 & 1.0000 \\
    $\tanh(5\sin(x)\cos(2x))$ & Fourier-4 (deg.\ 2) & 9 & 0.5186 & 0.5226 \\
    $\tanh(5\sin(x)\cos(2x))$ & Fourier-6 (deg.\ 2) & 16 & 0.1851 & 0.1894 \\
    $\tanh(5\sin(x)\cos(2x))$ & Poly-3 & 0 & 1.0000 & 1.0000 \\
    \midrule
    $\log y=2u-u^2$ & Raw $x{\to}y$ & 2 & 0.0359 & 4.80 \\
    $\log y=2u-u^2$ & $\log x$, raw $y$ & 3 & 0.0045 & 0.7504 \\
    \rowcolor{gray!12}
    $\log y=2u-u^2$ & Log-log & 2 & 0.0000 & \textbf{0.0000} \\
    $\log y=2u-u^2$ & Poly-2 $x{\to}y$ & 3 & 0.0041 & 9587 \\
    \midrule
    $\log y=1-2e^{-x}$ & Raw $x{\to}y$ & 2 & 0.0128 & 1.54 \\
    $\log y=1-2e^{-x}$ & $\log y$, feat.\ $x$ & 2 & 0.0147 & 1254 \\
    \rowcolor{gray!12}
    $\log y=1-2e^{-x}$ & $\log y$, state $v$ & 2 & 0.0000 & \textbf{0.0000} \\
    \rowcolor{gray!12}
    $\log y=1-2e^{-x}$ & $\log y$, feat.\ $e^{-x}$ & 1 & 0.0000 & \textbf{0.0000} \\
    $\log y=1-2e^{-x}$ & Poly-2 $x{\to}y$ & 3 & 0.0008 & 1660 \\
    \bottomrule
  \end{tabular}
\end{table}

\begingroup
\renewcommand{\theHtable}{\arabic{table}.continued}
\setlength{\tabcolsep}{3pt}
\addtocounter{table}{-1}
\begin{table}[ht]
  \centering
  \small
  \caption[Consolidated SINDy $\delta_\mathrm{OOD}$ diagnostic (continued).]{Consolidated SINDy $\delta_\mathrm{OOD}$ (continued): additional synthetic and real-data experiments. \colorbox{gray!12}{\strut Shaded rows} use a preassigned canonical coordinate; bold $\delta_\mathrm{OOD}\approx 0$ for the designated canonical rows. Residuals use each experiment's fitted target space; seed counts vary by experiment.}
  \begin{tabular}{llrrr}
    \toprule
    \textbf{DGP / Experiment} & \textbf{Feature map} & $\boldsymbol{k}$ & $\boldsymbol{\delta_\mathrm{tr}}$ & $\boldsymbol{\delta_\mathrm{OOD}}$ \\
    \midrule
    \rowcolor{gray!12}
    $\sin(x)$ \;(Exp.~1.1) & Fourier $[\sin x,\cos x,1]$ & 15 & 2.85e-30 & \textbf{2.83e-30} \\
    $\sin(x)$ \;(Exp.~1.1) & Poly-$9$ $[1,x,\dots,x^9]$ & 5 & 3.09e-11 & $3.7{\times}10^{4}$ \\
    \midrule
    \rowcolor{gray!12}
    Taylor-$9(x)$ \;(Exp.~1.1) & Poly-$9$ $[1,x,\dots,x^9]$ & 5 & 6.60e-26 & \textbf{2.71e-22} \\
    Taylor-$9(x)$ \;(Exp.~1.1) & Fourier $[\sin x,\cos x,1]$ & 25 & 2.73e-06 & 0.9979 \\
    \midrule
    \rowcolor{gray!12}
    $y = x^2$ power law \;(Exp.~1.2) & $(\log x) \to \log y$ & 1 & 0.0134 & \textbf{1.15e-04} \\
    $y = x^2$ power law \;(Exp.~1.2) & $x \to \log y$ & 2 & 0.0148 & 1.43 \\
    $y = x^2$ power law \;(Exp.~1.2) & $x \to y$ \;(raw) & 2 & 0.0072 & 0.4440 \\
    \midrule
    $y = e^{\alpha(x-1)}$ exp \;(Exp.~1.2) & $(\log x) \to \log y$ & 2 & 0.0204 & 0.3146 \\
    \rowcolor{gray!12}
    $y = e^{\alpha(x-1)}$ exp \;(Exp.~1.2) & $x \to \log y$ & 2 & 0.0181 & \textbf{1.40e-04} \\
    $y = e^{\alpha(x-1)}$ exp \;(Exp.~1.2) & $x \to y$ \;(raw) & 2 & 0.0117 & 0.9996 \\
    \midrule
    \rowcolor{gray!12}
    Kepler (real, $n_\mathrm{train}=1{,}881$, $n_\mathrm{OOD}=481$) & $(\log a,\log M_{\mathrm{star}})$ & 2 & 7.12e-04 & \textbf{0.0120} \\
    Kepler (real, $n_\mathrm{train}=1{,}881$, $n_\mathrm{OOD}=481$) & $(\log a,\log M_{\mathrm{star}},\log R_{\mathrm{star}},\log T_{\mathrm{eff}})$ & 8 & 7.01e-04 & 0.0105 \\
    Kepler (real, $n_\mathrm{train}=1{,}881$, $n_\mathrm{OOD}=481$) & $(\log a,\log T_{\mathrm{eff}})$ & 5 & 0.0011 & 0.0370 \\
    Kepler (real, $n_\mathrm{train}=1{,}881$, $n_\mathrm{OOD}=481$) & $(\log a,\log R_{\mathrm{star}})$ & 2 & 0.0012 & 0.1530 \\
    Kepler (real, $n_\mathrm{train}=1{,}881$, $n_\mathrm{OOD}=481$) & $(\log a)$ & 3 & 0.0041 & 0.1496 \\
    Kepler (real, $n_\mathrm{train}=1{,}881$, $n_\mathrm{OOD}=481$) & $(\log a,M_{\mathrm{star}})$ & 5 & 8.27e-04 & 2.74 \\
    Kepler (real, $n_\mathrm{train}=1{,}881$, $n_\mathrm{OOD}=481$) & $(a,M_{\mathrm{star}})$ & 3 & 0.1113 & 9.53 \\
    Kepler (real, $n_\mathrm{train}=1{,}881$, $n_\mathrm{OOD}=481$) & $(a,\log M_{\mathrm{star}})$ & 5 & 0.0060 & $6.7{\times}10^{5}$ \\
    \midrule
    \rowcolor{gray!12}
    MAK $\dot x_i = F - B x_i - R x_i x_j$ (synthetic) & $(x_i,\,x_j)$ & 3 & 5.66e-14 & \textbf{2.84e-14} \\
    MAK $\dot x_i = F - B x_i - R x_i x_j$ (synthetic) & $(\tanh x_i,\,\tanh x_j)$ & 6 & 0.0013 & 0.1296 \\
    MAK $\dot x_i = F - B x_i - R x_i x_j$ (synthetic) & $(x_i)$ & 3 & 0.2714 & 0.4195 \\
    \bottomrule
  \end{tabular}
\end{table}
\endgroup

\FloatBarrier
\section{Full Versions of Table Excerpts}
\label{app:full_tables}

The main text shows excerpts of several tables to fit the page budget, retaining only the rows that anchor the headline contrast. The full grids are reproduced here with all model classes, variable subsets, and baselines.

\subsection{Exp.~1.2 full grid (DGP specificity, all models, cross-DGP OOD)}
\label{app:exp12_full}

Full version of \Cref{tab:exp12_main} (\Cref{sec:exp12}), reproduced as \Cref{tab:exp12}. The grid crosses both candidate feature maps (\texttt{log-log}, \texttt{log-}$y$) and the no-transform baseline against three model classes (OLS, MLP, SINDy); each row's predictions are evaluated against \emph{both} the $x^2$ targets (OOD $P_1$ column) and the $e^{\alpha x}$ targets (OOD $P_2$ column) on $[2, 10]$, mirroring the layout of Exp.~1.1 (\Cref{tab:exp11}). The diagonal of (DGP $\times$ correct feature map) is the only configuration that succeeds OOD against the matching target: linear-in-feature models (OLS and SINDy) on the correct transform recover the right family up to noise, while every off-diagonal cell collapses against its own DGP. MLPs collapse on every cell, including the correct feature map, because they do not exploit the residual linearity in feature space.

\begin{table}[ht]
  \centering
  \small
  \caption{Exp.~1.2: DGP specificity. $P_1(x){=}x^2$ vs.\ $P_2(x){=}e^{(\ln 4)(x-1)}$ agree on $[1, 2]$ and diverge by ${\sim}10^5$ on $[2, 10]$. The DGP column identifies the training-data source. OOD $P_1$ and OOD $P_2$ score the same predictions against $P_1$ and $P_2$ as ground truth, respectively. \colorbox{gray!12}{\strut Shaded rows} use the correct feature map for OLS. TabPFN / TimesFM are zero-shot; the feature map is applied to inputs and (where applicable) outputs. OOD relative error (\%), mean $\pm$ std over 3 seeds; best per column in bold.}
  \label{tab:exp12}
  \begin{tabular}{lllrrr}
    \toprule
    \textbf{DGP} & \textbf{Feature map} & \textbf{Model} & \textbf{ID (\%)} & \textbf{OOD $P_1$ (\%)} & \textbf{OOD $P_2$ (\%)} \\
    \midrule
    \rowcolor{gray!12}
    $P_1$: $x^2$ & log-log & OLS & 1.0 $\pm$ 0.5 & \textbf{3.2 $\pm$ 1.7} & 99.8 $\pm$ 0.0 \\
    \rowcolor{gray!12}
     & log-log & MLP & 1.4 $\pm$ 0.3 & 83.3 $\pm$ 2.2 & 100.0 $\pm$ 0.0 \\
    \rowcolor{gray!12}
     & log-log & SINDy & 0.9 $\pm$ 0.4 & 3.8 $\pm$ 1.7 & 99.8 $\pm$ 0.0 \\
    \rowcolor{gray!12}
     & log-log & TabPFN & 1.3 $\pm$ 0.3 & 85.4 $\pm$ 1.0 & 100.0 $\pm$ 0.0 \\
    \rowcolor{gray!12}
     & log-log & TimesFM & -- & 28.3 $\pm$ 0.0 & 99.8 $\pm$ 0.0 \\
    \cmidrule(lr){2-6}
     & log-$y$ & OLS & 3.2 $\pm$ 0.3 & $\gg 9{,}999$ & 13.5 $\pm$ 3.1 \\
     & log-$y$ & MLP & 1.2 $\pm$ 0.4 & 78.5 $\pm$ 1.0 & 100.0 $\pm$ 0.0 \\
     & log-$y$ & SINDy & 3.2 $\pm$ 0.3 & $\gg 9{,}999$ & 13.5 $\pm$ 3.1 \\
     & log-$y$ & TabPFN & 1.3 $\pm$ 0.3 & 85.2 $\pm$ 2.1 & 100.0 $\pm$ 0.0 \\
     & log-$y$ & TimesFM & -- & $\gg 9{,}999$ & 20.6 $\pm$ 0.0 \\
    \cmidrule(lr){2-6}
     & None & MLP & 1.2 $\pm$ 0.2 & 85.1 $\pm$ 1.0 & 100.0 $\pm$ 0.0 \\
     & None & SINDy & 2.8 $\pm$ 0.2 & 61.8 $\pm$ 0.3 & 99.9 $\pm$ 0.0 \\
     & None & TabPFN & 1.2 $\pm$ 0.1 & 80.3 $\pm$ 2.1 & 100.0 $\pm$ 0.0 \\
     & None & TimesFM & -- & 77.7 $\pm$ 0.0 & 100.0 $\pm$ 0.0 \\
    \midrule
    \rowcolor{gray!12}
    $P_2$: $e^{\alpha(x-1)}$ & log-$y$ & OLS & 1.1 $\pm$ 0.5 & $\gg 9{,}999$ & \textbf{12.8 $\pm$ 8.8} \\
    \rowcolor{gray!12}
     & log-$y$ & MLP & 1.5 $\pm$ 0.3 & 72.9 $\pm$ 5.4 & 100.0 $\pm$ 0.0 \\
    \rowcolor{gray!12}
     & log-$y$ & SINDy & 1.1 $\pm$ 0.5 & $\gg 9{,}999$ & \textbf{12.8 $\pm$ 8.8} \\
    \rowcolor{gray!12}
     & log-$y$ & TabPFN & 1.4 $\pm$ 0.3 & 80.6 $\pm$ 3.3 & 100.0 $\pm$ 0.0 \\
    \rowcolor{gray!12}
     & log-$y$ & TimesFM & -- & $\gg 9{,}999$ & 29.6 $\pm$ 0.0 \\
    \cmidrule(lr){2-6}
     & log-log & OLS & 3.1 $\pm$ 0.2 & 5.6 $\pm$ 3.1 & 99.8 $\pm$ 0.0 \\
     & log-log & MLP & 1.5 $\pm$ 0.3 & 85.6 $\pm$ 4.4 & 100.0 $\pm$ 0.0 \\
     & log-log & SINDy & 3.1 $\pm$ 0.2 & 5.6 $\pm$ 3.1 & 99.8 $\pm$ 0.0 \\
     & log-log & TabPFN & 1.3 $\pm$ 0.3 & 83.5 $\pm$ 1.6 & 100.0 $\pm$ 0.0 \\
     & log-log & TimesFM & -- & 52.4 $\pm$ 0.0 & 99.7 $\pm$ 0.0 \\
    \cmidrule(lr){2-6}
     & None & MLP & 1.4 $\pm$ 0.2 & 88.3 $\pm$ 1.2 & 100.0 $\pm$ 0.0 \\
     & None & SINDy & 6.0 $\pm$ 0.2 & 63.2 $\pm$ 0.3 & 99.9 $\pm$ 0.0 \\
     & None & TabPFN & 1.3 $\pm$ 0.1 & 74.5 $\pm$ 3.3 & 100.0 $\pm$ 0.0 \\
     & None & TimesFM & -- & 77.1 $\pm$ 0.0 & 100.0 $\pm$ 0.0 \\
    \bottomrule
  \end{tabular}
\end{table}

\subsection{Exp.~3.1 full grid (selection-protocol counts)}
\label{app:exp31_full}

Full version including the per-criterion selection-counts column. Across $300$ selections ($100$ trials $\times$ $3$ seeds), CV picks Fourier OLS $244$ times and fails on the remaining $56$, distributed as Poly-7 OLS: $37$, Poly-9 OLS: $19$, raw OLS: $0$, confirming that the wrong picks land on the observationally-equivalent polynomial confounds rather than on the raw baseline.

\begin{table}[ht]
  \centering
  \small
  \caption{Exp.~3.1: Selection protocol. 100 trials, DGP $\sin(x)$, $K=4$ candidate feature maps (all OLS; random baseline 25\%). \colorbox{gray!12}{\strut Shaded row} is the near-boundary protocol (\Cref{cor:nearboundary}), querying the $\eps_W$-divergence boundary outside the agreement zone. Selection counts are over 300 (trial $\times$ seed) draws; the correct $\varphi(x){=}(\sin x, \cos x)$ (bold header) maximises the count. Accuracy (\%), mean $\pm$ std; best in bold.}
  \label{tab:exp31}
  \begin{tabular}{lrrrrr}
    \toprule
    \textbf{Criterion} & \textbf{Accuracy (\%)} & \multicolumn{4}{c}{\textbf{Selection counts}} \\
    \cmidrule(lr){3-6}
     & & \textbf{Fourier} & Poly-7 & Poly-9 & Raw $x$ \\
    \midrule
    5-fold CV & 81.3 $\pm$ 3.4 & \textbf{244} & 37 & 19 & 0 \\
    \rowcolor{gray!12}
    Near-boundary val. & \textbf{100.0 $\pm$ 0.0} & \textbf{300} & 0 & 0 & 0 \\
    \bottomrule
  \end{tabular}
\end{table}

\subsection{Exp.~2.1 (MAK) vector-field full grid}
\label{app:exp21_full}

Full version of \Cref{tab:exp21_main} (\Cref{sec:exp21}), reproduced as \Cref{tab:exp21}, with the extended caption for the mass-action-kinetics vector-field experiment.

\begin{table}[ht]
  \centering
  \small
  \caption{Exp.~2.1 (MAK): vector field on a canonical synthetic 2-node graph \citep{vasiliauskaite2024generalization}. True process: $\dot{x}_i = F - Bx_i - Rx_ix_j$ with textbook parameter values $F=0.5,\,B=0.1,\,R=1.0$ (dimensionless, representative of the CSTR mass-action regime: $F$ is constant inflow, $B$ first-order outflow, $R$ bimolecular rate). Training: $(x_1, x_2) \in [0.4, 1.2]^2$; OOD: $[0, 2.0]^2 \setminus [0.4, 1.2]^2$. \colorbox{gray!12}{\strut Shaded rows} use the canonically-correct bilinear feature map $[1, x_i, x_i x_j]$, which linearises the target globally: OLS on these features recovers $(F, -B, -R)$ exactly. OOD relative error (\%), mean $\pm$ std over 3 seeds. Best OOD in bold.}
  \label{tab:exp21}
  \begin{tabular}{llrr}
    \toprule
    \textbf{Feature map} & \textbf{Model} & \textbf{ID (\%)} & \textbf{OOD (\%)} \\
    \midrule
    None & MLP & 0.5 $\pm$ 0.0 & 24.4 $\pm$ 0.9 \\
    \cmidrule(lr){1-4}
    \rowcolor{gray!12}
    Bilinear & OLS & 0.0 $\pm$ 0.0 & \textbf{0.0 $\pm$ 0.0} \\
    \rowcolor{gray!12}
    Bilinear & MLP & 1.4 $\pm$ 1.2 & 16.9 $\pm$ 1.3 \\
    \bottomrule
  \end{tabular}
\end{table}

\subsection{Exp.~2.2 Kepler full ablation}
\label{app:exp22_full}

Full version of \Cref{tab:exp22_main} (\Cref{sec:exp22}), reproduced as \Cref{tab:exp22}, with the complete variable-subset $\times$ feature-map $\times$ model ablation. The correct $(\log a, \log M_{\mathrm{star}})$ representation paired with OLS or SINDy is the OOD winner; MLPs and TabPFN collapse on every variable subset, including the correct one. The four-variable extension $(\log a, \log M_{\mathrm{star}}, \log R_{\mathrm{star}}, \log T_\mathrm{eff})$ matches the two-variable Kepler result, consistent with the irrelevance of $R_{\mathrm{star}}, T_\mathrm{eff}$ to the third-law DGP.

\begin{table}[ht]
  \centering
  \small
  \caption{Exp.~2.2 (Kepler): third law on NASA exoplanets. $n_\mathrm{train}=1{,}881$ ($a<0.5$~AU), $n_\mathrm{OOD}=481$ ($\leq20$~AU). OOD rel.\ error (\%), mean $\pm$ std over 3 seeds; best OOD per variable group in bold. Correct OLS recovers $b_a=1.493$, $b_{M_{\mathrm{star}}}=-0.493$ (theory $1.500, -0.500$); SINDy on the correct map gives $k=2$ active terms. \colorbox{gray!12}{\strut Shaded rows} use the correct feature map.}
  \label{tab:exp22}
  \begin{tabular}{lllrr}
    \toprule
    \textbf{Variables} & \textbf{Feature map} & \textbf{Model} & \textbf{ID (\%)} & \textbf{OOD (\%)} \\
    \midrule
    $a$ & $a$ & OLS & 28.3 $\pm$ 0.0 & 63.5 $\pm$ 0.0 \\
     &  & SINDy & 18.8 $\pm$ 0.0 & 158.1 $\pm$ 0.0 \\
     &  & MLP & 18.2 $\pm$ 1.0 & 85.8 $\pm$ 0.4 \\
     &  & TabPFN & 16.1 $\pm$ 0.0 & 69.5 $\pm$ 0.0 \\
    \cmidrule(lr){2-5}
     & $\log a$ & OLS & 16.4 $\pm$ 0.0 & \textbf{21.2 $\pm$ 0.0} \\
     &  & SINDy & 16.4 $\pm$ 0.0 & 107.9 $\pm$ 0.0 \\
     &  & MLP & 15.9 $\pm$ 0.0 & 75.0 $\pm$ 0.8 \\
     &  & TabPFN & 15.6 $\pm$ 0.0 & 48.9 $\pm$ 0.0 \\
    \midrule
    $a,\,M_{\mathrm{star}}$ & $a,\,M_{\mathrm{star}}$ & OLS & 26.0 $\pm$ 0.0 & 63.1 $\pm$ 0.0 \\
     &  & SINDy & 13.5 $\pm$ 0.0 & 165.2 $\pm$ 0.0 \\
     &  & MLP & 6.9 $\pm$ 0.2 & 79.0 $\pm$ 2.1 \\
     &  & TabPFN & 4.4 $\pm$ 0.0 & 70.3 $\pm$ 0.0 \\
    \cmidrule(lr){2-5}
     & $\log a,\,M_{\mathrm{star}}$ & OLS & 7.6 $\pm$ 0.0 & 14.1 $\pm$ 0.0 \\
     &  & SINDy & 5.3 $\pm$ 0.0 & $\gg 9{,}999$ \\
     &  & MLP & 4.7 $\pm$ 0.2 & 58.2 $\pm$ 2.1 \\
     &  & TabPFN & 4.4 $\pm$ 0.0 & 20.8 $\pm$ 0.0 \\
    \cmidrule(lr){2-5}
     & $a,\,\log M_{\mathrm{star}}$ & OLS & 83.1 $\pm$ 0.0 & $\gg 9{,}999$ \\
     &  & SINDy & 30.9 $\pm$ 0.0 & 99.9 $\pm$ 0.0 \\
     &  & MLP & 12.3 $\pm$ 5.5 & 86.1 $\pm$ 3.5 \\
     &  & TabPFN & 4.4 $\pm$ 0.0 & 7{,}188.4 $\pm$ 0.0 \\
    \cmidrule(lr){2-5}
    \rowcolor{gray!12}
     & $\log a,\,\log M_{\mathrm{star}}$ & OLS & 4.2 $\pm$ 0.0 & 3.8 $\pm$ 0.0 \\
    \rowcolor{gray!12}
     &  & SINDy & 3.9 $\pm$ 0.0 & \textbf{2.1 $\pm$ 0.0} \\
    \rowcolor{gray!12}
     &  & MLP & 4.4 $\pm$ 0.1 & 60.9 $\pm$ 2.5 \\
    \rowcolor{gray!12}
     &  & TabPFN & 4.3 $\pm$ 0.0 & 27.6 $\pm$ 0.0 \\
    \midrule
    $a,\,R_{\mathrm{star}}$ & $\log a,\,\log R_{\mathrm{star}}$ & OLS & 8.4 $\pm$ 0.0 & \textbf{16.3 $\pm$ 0.0} \\
     &  & SINDy & 8.3 $\pm$ 0.0 & 18.6 $\pm$ 0.0 \\
     &  & MLP & 6.4 $\pm$ 0.2 & 56.1 $\pm$ 3.1 \\
     &  & TabPFN & 5.6 $\pm$ 0.0 & 40.5 $\pm$ 0.0 \\
    \midrule
    $a,\,T_\mathrm{eff}$ & $\log a,\,\log T_\mathrm{eff}$ & OLS & 9.2 $\pm$ 0.0 & 17.3 $\pm$ 0.0 \\
     &  & SINDy & 8.5 $\pm$ 0.0 & \textbf{16.1 $\pm$ 0.0} \\
     &  & MLP & 10.9 $\pm$ 2.1 & 66.1 $\pm$ 6.8 \\
     &  & TabPFN & 7.0 $\pm$ 0.0 & 26.0 $\pm$ 0.0 \\
    \midrule
    $a,\,M_{\mathrm{star}},\,R_{\mathrm{star}},\,T_\mathrm{eff}$ & $\log a,\,\log M_{\mathrm{star}},\,\log R_{\mathrm{star}},\,\log T_\mathrm{eff}$ & OLS & 4.3 $\pm$ 0.0 & 4.4 $\pm$ 0.0 \\
     &  & SINDy & 4.3 $\pm$ 0.0 & \textbf{3.5 $\pm$ 0.0} \\
     &  & MLP & 5.5 $\pm$ 0.8 & 56.7 $\pm$ 2.6 \\
     &  & TabPFN & 4.0 $\pm$ 0.0 & 30.9 $\pm$ 0.0 \\
    \bottomrule
  \end{tabular}
\end{table}

\subsection{Full $\varphi$-misspecification sweep (\Cref{sec:exp45})}
\label{app:misspec}

\Cref{tab:e3_misspec} gives the three detuning axes in full: a wrong Fourier frequency, the exact map applied to only some coordinates, and additive noise on the features. The frequency and noise axes degrade smoothly between the exact commitment and the no-commitment anchor; the coordinate-coverage axis does not, and partial coverage of a multiplicatively coupled target is worse than no commitment at all.

\begin{table}[ht]
  \centering
  \small
  \caption{Exp.~4.5: $\varphi$-misspecification sweep on the torus benchmark. Latent $\theta \in \mathbb{T}^d$, training on the $k{=}0$ block, OOD by an integer block shift $x = \theta + 2\pi$; target $g_\mathrm{int}$ with $\lambda = 0.35$, MLP ($5$ Tanh layers, $256$ units), $N_\mathrm{train} = 4{,}096$, $3{,}000$ epochs. Three axes degrade the exact map $\varphi(x) = (\sin x_j, \cos x_j)_{j=1}^d$ independently. \textbf{A}: a wrong Fourier frequency, $\varphi_\omega(x) = (\sin \omega x_j, \cos \omega x_j)_{j=1}^d$. \textbf{B}: the exact map on only the first $m$ coordinates, raw pass-through for the rest, so $m = 0$ is itself the no-commitment anchor. \textbf{C}: exact frequencies but additive feature noise. Rows labelled \texttt{raw} carry no commitment at all and are the reference anchor; for scale, the predict-the-mean RMSE of the target is $0.369$ at $d{=}2$ and $0.246$ at $d{=}4$. ID and OOD RMSE, mean $\pm$ std over 3 seeds; best OOD RMSE per block in bold; \colorbox{gray!12}{\strut shaded rows} use the exact map. Degradation is graceful on the frequency and noise axes (A and C: OOD RMSE interpolates smoothly between the exact commitment and the no-commitment anchor, with no cliff), but non-monotone on the coordinate-coverage axis (B): because $g_\mathrm{int}$ couples all $d$ coordinates multiplicatively, partial coverage ($m = 1, 2, 3$) is \emph{worse} than no commitment at all. Partial alignment buys partial extrapolation; partial \emph{coverage} of a coupled mechanism buys none.}
  \label{tab:e3_misspec}
  \begin{tabular}{lrr}
    \toprule
    \textbf{Setting} & \textbf{ID RMSE} & \textbf{OOD RMSE} \\
    \midrule
    \multicolumn{3}{l}{\emph{A. Wrong Fourier frequency ($d{=}2$, $\omega_\mathrm{true}{=}1$)}} \\
    \midrule
    $\omega = 0.90$ & 0.0041 $\pm$ 0.0002 & 0.2577 $\pm$ 0.0025 \\
    $\omega = 0.95$ & 0.0055 $\pm$ 0.0007 & 0.1297 $\pm$ 0.0013 \\
    $\omega = 0.98$ & 0.0059 $\pm$ 0.0006 & 0.0515 $\pm$ 0.0003 \\
    \rowcolor{gray!12}
    $\omega = 1.00$ & 0.0078 $\pm$ 0.0058 & \textbf{0.0078 $\pm$ 0.0058} \\
    $\omega = 1.02$ & 0.0095 $\pm$ 0.0002 & 0.0540 $\pm$ 0.0013 \\
    $\omega = 1.05$ & 0.0284 $\pm$ 0.0007 & 0.1440 $\pm$ 0.0044 \\
    $\omega = 1.10$ & 0.0705 $\pm$ 0.0003 & 0.2984 $\pm$ 0.0091 \\
    $\omega = 1.25$ & 0.2063 $\pm$ 0.0032 & 0.5750 $\pm$ 0.0096 \\
    $\omega = 1.50$ & 0.3495 $\pm$ 0.0048 & 0.5037 $\pm$ 0.0087 \\
    \texttt{raw} (no $\varphi$) & 0.0123 $\pm$ 0.0065 & 0.5538 $\pm$ 0.0277 \\
    \midrule
    \multicolumn{3}{l}{\emph{B. Partial coordinate coverage ($\varphi$ on $m$ of $d{=}4$)}} \\
    \midrule
    $m = 0$ (\texttt{raw}, no $\varphi$) & 0.0186 $\pm$ 0.0023 & 0.3528 $\pm$ 0.0245 \\
    $m = 1$ & 0.0160 $\pm$ 0.0027 & 0.5027 $\pm$ 0.0143 \\
    $m = 2$ & 0.0134 $\pm$ 0.0009 & 0.5989 $\pm$ 0.0135 \\
    $m = 3$ & 0.0186 $\pm$ 0.0064 & 0.5349 $\pm$ 0.0410 \\
    \rowcolor{gray!12}
    $m = 4$ (exact $\varphi$) & 0.0093 $\pm$ 0.0000 & \textbf{0.0089 $\pm$ 0.0007} \\
    \midrule
    \multicolumn{3}{l}{\emph{C. Noisy $\varphi$ (feature noise $\mathcal{N}(0, s^2)$, $d{=}2$)}} \\
    \midrule
    \rowcolor{gray!12}
    $s = 0.00$ (exact $\varphi$) & 0.0059 $\pm$ 0.0038 & \textbf{0.0059 $\pm$ 0.0037} \\
    $s = 0.05$ & 0.0234 $\pm$ 0.0032 & 0.0232 $\pm$ 0.0031 \\
    $s = 0.10$ & 0.0427 $\pm$ 0.0009 & 0.0416 $\pm$ 0.0013 \\
    $s = 0.20$ & 0.0884 $\pm$ 0.0031 & 0.0865 $\pm$ 0.0004 \\
    $s = 0.40$ & 0.2418 $\pm$ 0.0009 & 0.2412 $\pm$ 0.0041 \\
    \bottomrule
  \end{tabular}
\end{table}

\subsection{Full Camelyon17 table (\Cref{sec:exp24})}
\label{app:histopath}

\begin{figure}[t]
  \centering
  \includegraphics[width=\linewidth]{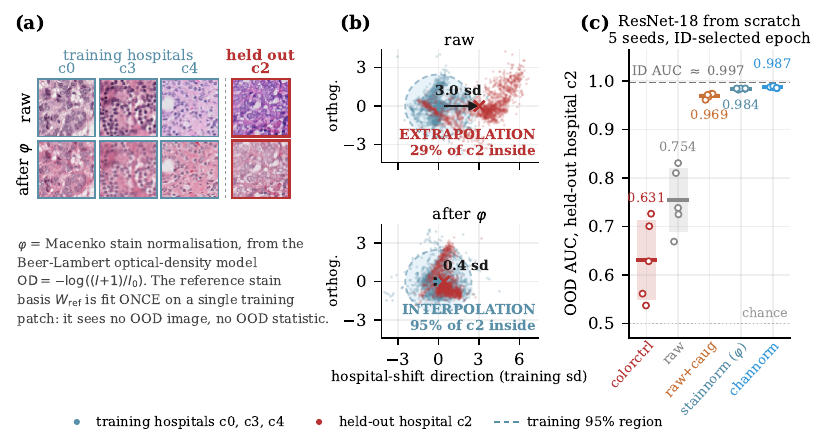}
  \caption{Exp.~2.4, an extrapolation made into an interpolation. \emph{(a)} Real Camelyon17 patches from the three training hospitals and from the held-out hospital, before (top) and after (bottom) the stain map $\varphi$; the reference basis is fit once on a single training patch. \emph{(b)} Patch-mean colour, whitened by the training cloud's own covariance. Raw, the held-out hospital sits $3.0$ training standard deviations away and only $29\%$ of its patches fall inside the training $95\%$ region; after $\varphi$ the offset is $0.4$ and coverage is $95\%$. \emph{(c)} OOD AUC per arm with all five seeds shown individually. The non-adaptive \texttt{colorctrl} control lands below \texttt{raw}, while $\varphi$ and \texttt{channorm} are tight clusters near the in-distribution ceiling.\label{fig:histopath_concept}}
\end{figure}

\Cref{tab:histopath} gives the full grid behind \Cref{sec:exp24}, with in-distribution and out-of-distribution accuracy alongside AUC for all four input maps. 

\begin{table}[ht]
  \centering
  \small
  \caption{Exp.~2.4 (Camelyon17): cross-hospital tumour classification under a fixed input feature map $\varphi$. Camelyon17-WILDS $96{\times}96$ H\&E lymph-node patches (tumour vs.\ normal); training centres $\{0,\,3,\,4\}$, held-out OOD centre $2$. $\varphi$, \texttt{stainnorm}, is Macenko stain normalisation, derived from the Beer--Lambert optical-density model of dye absorption and fit \emph{once} on a single training patch: it sees no OOD image and no OOD statistic. Arms: \texttt{raw}, no preprocessing; \texttt{stainnorm} ($\varphi_1$); \texttt{channorm} ($\varphi_2$), \texttt{raw+caug} ($\varphi_4$) per-image per-channel mean/std standardisation (adaptive but stain-agnostic, it never aligns the stain basis) ($\varphi_3$); training-time stochastic colour jitter  (the standard domain-generalisation baseline); \texttt{colorctrl} ($\varphi_c$), one fixed random RGB affine applied identically to every image and every hospital (negative control, matched in form to $\varphi$ but non-adaptive). ResNet-18 trained from scratch ($n_\mathrm{train}=60{,}000$, $n_\mathrm{ID}=10{,}000$, $n_\mathrm{OOD}=20{,}000$), reported at the \emph{ID-selected} epoch (model selection uses in-distribution validation only and never sees OOD labels), mean $\pm$ std over 5 seeds. \colorbox{gray!12}{\strut Shaded rows} are input maps $\varphi$; best OOD AUC in bold. ID AUC is $\ge 0.995$ for every arm, so the entire spread is out-of-distribution: $\varphi$ converts an OOD-AUC collapse to $0.75$ into $0.98$, while the non-adaptive \texttt{colorctrl} control deepens it to $0.63$.}
  \label{tab:histopath}
  \resizebox{\linewidth}{!}{%
  \begin{tabular}{lrrrr}
    \toprule
    \textbf{Input map} & \textbf{ID acc} & \textbf{ID AUC} & \textbf{OOD acc} & \textbf{OOD AUC} \\
    \midrule
    \multicolumn{5}{l}{\emph{ResNet-18, trained from scratch (5 seeds)}} \\
    \midrule
    \texttt{raw} & 0.9854 $\pm$ 0.0006 & 0.9984 $\pm$ 0.0001 & 0.6815 $\pm$ 0.0479 & 0.7544 $\pm$ 0.0590 \\
    \rowcolor{gray!12}
    \texttt{stainnorm} ($\varphi$) & 0.9762 $\pm$ 0.0006 & 0.9962 $\pm$ 0.0002 & 0.9410 $\pm$ 0.0021 & 0.9837 $\pm$ 0.0010 \\
    \rowcolor{gray!12}
    \texttt{channorm} & 0.9777 $\pm$ 0.0006 & 0.9965 $\pm$ 0.0002 & 0.9411 $\pm$ 0.0034 & \textbf{0.9867 $\pm$ 0.0015} \\
    \rowcolor{gray!12}
    \texttt{raw+caug} & 0.9681 $\pm$ 0.0012 & 0.9951 $\pm$ 0.0003 & 0.9136 $\pm$ 0.0098 & 0.9692 $\pm$ 0.0047 \\
    \rowcolor{gray!12}
    \texttt{colorctrl} & 0.9840 $\pm$ 0.0005 & 0.9981 $\pm$ 0.0001 & 0.5774 $\pm$ 0.0698 & 0.6306 $\pm$ 0.0743 \\
    \bottomrule
  \end{tabular}%
  }
\end{table}

\subsection{Full $d$-torus tables (\Cref{sec:scope})}
\label{app:torus}

Per-cell results for all four torus benchmarks of \Cref{fig:torus_scope}: representation alignment under fixed budget (condition (a), \Cref{tab:exp41}), capacity at full interaction order (condition (b), \Cref{tab:exp43}), coverage at full interaction order (condition (c), \Cref{tab:exp42}), and noise robustness across $d \in \{4, 6, 8\}$ (condition (d), \Cref{tab:exp44}).

\begin{table}[ht]
  \centering
  \small
  \caption{Torus benchmark, condition (a). Fixed budget $N=4096$ on the $d$-torus with target $g_\mathrm{int}$. The exact map $\varphi(x)=(\sin x_j, \cos x_j)$ dominates raw and learned-Fourier baselines. OOD RMSE, mean$\pm$std over 100 seeds; best per row in \textbf{bold}.}
  \label{tab:exp41}
  \begin{tabular}{lrrr}
    \toprule
    $d$ & raw MLP & learned-Fourier MLP & exact-$\varphi$ MLP \\
    \midrule

    1 & 0.949 $\pm$ 0.038 & 0.736 $\pm$ 0.492 & \textbf{0.004 $\pm$ 0.010} \\
    2 & 0.560 $\pm$ 0.070 & 0.509 $\pm$ 0.253 & \textbf{0.003 $\pm$ 0.003} \\
    4 & 0.360 $\pm$ 0.039 & 0.428 $\pm$ 0.110 & \textbf{0.009 $\pm$ 0.005} \\
    6 & 0.247 $\pm$ 0.021 & 0.370 $\pm$ 0.101 & \textbf{0.019 $\pm$ 0.003} \\
    8 & 0.289 $\pm$ 0.042 & 0.274 $\pm$ 0.055 & \textbf{0.028 $\pm$ 0.002} \\
    12 & 0.428 $\pm$ 0.058 & 0.249 $\pm$ 0.068 & \textbf{0.008 $\pm$ 0.001} \\
    \bottomrule
  \end{tabular}
\end{table}

\begin{table}[ht]
  \centering
  \small
  \caption{Torus benchmark, condition (b). Exact $\varphi$ held fixed; target $g_{d,q}$ has interaction order $q$. OOD $R^2$ collapses as $q$ approaches $d$, even when ID $R^2\!\approx\!1$, exposing a capacity bottleneck at finite coverage. Best OOD $R^2 \geq 0.95$ per row group in \textbf{bold}; $N=M^d$. Mean$\pm$std over 100 seeds.}
  \label{tab:exp43}
  \begin{tabular}{rrrrr}
    \toprule
    $d$ & $M$ & $q$ & ID $R^2$ & OOD $R^2$ \\
    \midrule

    4 & 4 & 1 & 1.000 $\pm$ 0.000 & \textbf{1.000 $\pm$ 0.000} \\
    4 & 4 & 2 & 0.991 $\pm$ 0.004 & \textbf{0.992 $\pm$ 0.004} \\
    4 & 4 & 3 & 0.988 $\pm$ 0.004 & \textbf{0.988 $\pm$ 0.004} \\
    4 & 4 & 4 & 0.666 $\pm$ 0.158 & 0.666 $\pm$ 0.157 \\
    4 & 5 & 1 & 1.000 $\pm$ 0.000 & \textbf{1.000 $\pm$ 0.000} \\
    4 & 5 & 2 & 0.998 $\pm$ 0.002 & \textbf{0.998 $\pm$ 0.002} \\
    4 & 5 & 3 & 0.997 $\pm$ 0.003 & \textbf{0.997 $\pm$ 0.003} \\
    4 & 5 & 4 & 0.988 $\pm$ 0.004 & \textbf{0.988 $\pm$ 0.004} \\
    4 & 6 & 1 & 1.000 $\pm$ 0.000 & \textbf{1.000 $\pm$ 0.000} \\
    4 & 6 & 2 & 0.998 $\pm$ 0.002 & \textbf{0.998 $\pm$ 0.002} \\
    4 & 6 & 3 & 0.998 $\pm$ 0.002 & \textbf{0.998 $\pm$ 0.003} \\
    4 & 6 & 4 & 0.997 $\pm$ 0.002 & \textbf{0.997 $\pm$ 0.002} \\
    \midrule
    6 & 4 & 1 & 1.000 $\pm$ 0.000 & \textbf{1.000 $\pm$ 0.000} \\
    6 & 4 & 2 & 0.997 $\pm$ 0.007 & \textbf{0.997 $\pm$ 0.007} \\
    6 & 4 & 3 & 0.994 $\pm$ 0.005 & \textbf{0.994 $\pm$ 0.005} \\
    6 & 4 & 4 & 0.982 $\pm$ 0.005 & \textbf{0.982 $\pm$ 0.005} \\
    6 & 4 & 5 & 0.985 $\pm$ 0.003 & \textbf{0.985 $\pm$ 0.003} \\
    6 & 4 & 6 & 0.699 $\pm$ 0.184 & 0.696 $\pm$ 0.194 \\
    6 & 5 & 1 & 1.000 $\pm$ 0.000 & \textbf{1.000 $\pm$ 0.000} \\
    6 & 5 & 2 & 0.998 $\pm$ 0.004 & \textbf{0.998 $\pm$ 0.004} \\
    6 & 5 & 3 & 0.995 $\pm$ 0.004 & \textbf{0.995 $\pm$ 0.004} \\
    6 & 5 & 4 & 0.993 $\pm$ 0.003 & \textbf{0.994 $\pm$ 0.003} \\
    6 & 5 & 5 & 0.996 $\pm$ 0.001 & \textbf{0.996 $\pm$ 0.001} \\
    6 & 5 & 6 & 0.991 $\pm$ 0.002 & \textbf{0.991 $\pm$ 0.002} \\
    6 & 6 & 1 & 1.000 $\pm$ 0.000 & \textbf{1.000 $\pm$ 0.000} \\
    6 & 6 & 2 & 0.998 $\pm$ 0.004 & \textbf{0.998 $\pm$ 0.004} \\
    6 & 6 & 3 & 0.996 $\pm$ 0.004 & \textbf{0.996 $\pm$ 0.004} \\
    6 & 6 & 4 & 0.994 $\pm$ 0.004 & \textbf{0.994 $\pm$ 0.004} \\
    6 & 6 & 5 & 0.997 $\pm$ 0.002 & \textbf{0.996 $\pm$ 0.002} \\
    6 & 6 & 6 & 0.996 $\pm$ 0.002 & \textbf{0.996 $\pm$ 0.002} \\
    \midrule
    8 & 4 & 1 & 1.000 $\pm$ 0.001 & \textbf{1.000 $\pm$ 0.001} \\
    8 & 4 & 2 & 0.988 $\pm$ 0.098 & \textbf{0.988 $\pm$ 0.096} \\
    8 & 4 & 3 & 0.994 $\pm$ 0.006 & \textbf{0.993 $\pm$ 0.007} \\
    8 & 4 & 4 & 0.983 $\pm$ 0.033 & \textbf{0.984 $\pm$ 0.035} \\
    8 & 4 & 5 & 0.982 $\pm$ 0.011 & \textbf{0.982 $\pm$ 0.010} \\
    8 & 4 & 6 & 0.974 $\pm$ 0.008 & \textbf{0.974 $\pm$ 0.008} \\
    8 & 4 & 7 & 0.979 $\pm$ 0.005 & \textbf{0.979 $\pm$ 0.006} \\
    8 & 4 & 8 & 0.855 $\pm$ 0.079 & 0.856 $\pm$ 0.071 \\
    \bottomrule
  \end{tabular}
\end{table}

\begin{table}[ht]
  \centering
  \small
  \caption{Torus benchmark, condition (c). Exact $\varphi$ held fixed at full interaction order $q{=}d$; the only knob is per-coordinate coverage $M$ at $N{=}M^d$. OOD $R^2$ rises with $M$ once the sample budget covers each torus cell, locating the coverage threshold $M_{\mathrm{crit}} \approx 4$ (\Cref{sec:scope}). Cells with $N$ above budget are omitted. Best OOD $R^2 \geq 0.95$ per row group in \textbf{bold}; mean$\pm$std over 100 seeds.}
  \label{tab:exp42}
  \begin{tabular}{rrrrrr}
    \toprule
    $d$ & $M$ & $N$ & ID $R^2$ & OOD $R^2$ & OOD RMSE \\
    \midrule

    4 & 2 & 16 & -1.237 $\pm$ 0.734 & -1.251 $\pm$ 0.752 & 1.474 $\pm$ 0.235 \\
    4 & 3 & 81 & -0.789 $\pm$ 0.222 & -0.776 $\pm$ 0.204 & 1.328 $\pm$ 0.076 \\
    4 & 4 & 256 & 0.660 $\pm$ 0.161 & 0.661 $\pm$ 0.159 & 0.566 $\pm$ 0.134 \\
    4 & 5 & 625 & 0.988 $\pm$ 0.004 & \textbf{0.988 $\pm$ 0.004} & 0.110 $\pm$ 0.018 \\
    4 & 6 & 1{,}296 & 0.997 $\pm$ 0.002 & \textbf{0.997 $\pm$ 0.002} & 0.050 $\pm$ 0.014 \\
    \midrule
    6 & 2 & 64 & -0.648 $\pm$ 0.290 & -0.657 $\pm$ 0.304 & 1.279 $\pm$ 0.115 \\
    6 & 3 & 729 & -0.800 $\pm$ 0.155 & -0.793 $\pm$ 0.154 & 1.331 $\pm$ 0.058 \\
    6 & 4 & 4{,}096 & 0.696 $\pm$ 0.182 & 0.695 $\pm$ 0.172 & 0.531 $\pm$ 0.132 \\
    6 & 5 & 15{,}625 & 0.991 $\pm$ 0.002 & \textbf{0.991 $\pm$ 0.002} & 0.092 $\pm$ 0.012 \\
    6 & 6 & 46{,}656 & 0.996 $\pm$ 0.002 & \textbf{0.996 $\pm$ 0.002} & 0.061 $\pm$ 0.010 \\
    \midrule
    8 & 2 & 256 & -0.751 $\pm$ 0.238 & -0.765 $\pm$ 0.227 & 1.315 $\pm$ 0.092 \\
    8 & 3 & 6{,}561 & -0.672 $\pm$ 0.115 & -0.670 $\pm$ 0.126 & 1.291 $\pm$ 0.059 \\
    8 & 4 & 65{,}536 & 0.860 $\pm$ 0.051 & 0.863 $\pm$ 0.050 & 0.364 $\pm$ 0.062 \\
    \bottomrule
  \end{tabular}
\end{table}

\begin{table}[ht]
  \centering
  \small
  \caption{Exp.~4.4: Noise robustness (d) of the DGP commitment on the $d$-torus, fixed $N{=}4{,}096$. Training labels carry Gaussian noise $\eta \sim \mathcal{N}(0, \sigma^2)$; ID and OOD targets remain clean. The canonical $\varphi(x){=}(\sin x_j, \cos x_j)_{j=1}^d$ shows approximately linear error scaling with training-noise level over $\sigma>0$; raw and learned-Fourier baselines are bias-floored below the plug-in crossover from \Cref{prop:erm} ($\sigma \lesssim \sqrt{n}\,\eps_W/\sqrt{2\pi}$, bracketed by the stored cells to $[0.71, 1.57]$) and rise with $\sigma$ above this estimated crossover, but the canonical commitment retains the OOD advantage throughout. OOD RMSE, mean$\pm$std over 100 seeds; best OOD RMSE per row in \textbf{bold}. All three arms are the same MLP under three feature maps. The two metrics answer different questions. OOD $R^2$ answers \emph{does this arm beat predicting the mean, in absolute terms}: it is positive only while the OOD RMSE stays below the standard deviation of the clean OOD target, and goes negative once the RMSE exceeds it. Its denominator is the \emph{clean} OOD target variance, which is the same for every arm and every $\sigma$, so at high $\sigma$ every arm is negative: no arm extrapolates usefully in absolute terms there, however the arms rank against each other. \Cref{tab:exp44_ratio} answers the complementary relative question. OOD $R^2 \geq 0.95$ in \textbf{bold}.}
  \label{tab:exp44}
  \resizebox{\textwidth}{!}{%
  \begin{tabular}{lrrrrrr}
    \toprule
    & \multicolumn{3}{c}{OOD RMSE} & \multicolumn{3}{c}{OOD $R^2$} \\
    \cmidrule(lr){2-4} \cmidrule(lr){5-7}
    $\sigma$ & raw & learned-Fourier & exact-$\varphi$ & raw & learned-Fourier & exact-$\varphi$ \\
    \midrule

    \multicolumn{7}{l}{\emph{$d = 4$}} \\
    \midrule
    0 & 0.361 $\pm$ 0.038 & 0.434 $\pm$ 0.118 & \textbf{0.009 $\pm$ 0.003} & -1.19 $\pm$ 0.45 & -2.34 $\pm$ 1.83 & \textbf{0.999 $\pm$ 0.001} \\
    0.01 & 0.365 $\pm$ 0.049 & 0.442 $\pm$ 0.131 & \textbf{0.010 $\pm$ 0.003} & -1.25 $\pm$ 0.65 & -2.53 $\pm$ 2.38 & \textbf{0.998 $\pm$ 0.001} \\
    0.1 & 0.354 $\pm$ 0.048 & 0.476 $\pm$ 0.121 & \textbf{0.073 $\pm$ 0.015} & -1.11 $\pm$ 0.60 & -3.00 $\pm$ 2.09 & 0.908 $\pm$ 0.038 \\
    0.3 & 0.575 $\pm$ 0.120 & 0.784 $\pm$ 0.178 & \textbf{0.407 $\pm$ 0.017} & -4.72 $\pm$ 2.74 & -9.73 $\pm$ 5.34 & -1.75 $\pm$ 0.24 \\
    1 & 1.829 $\pm$ 0.303 & 2.090 $\pm$ 0.427 & \textbf{1.365 $\pm$ 0.046} & -55.9 $\pm$ 19.9 & -74.3 $\pm$ 30.3 & -29.9 $\pm$ 2.3 \\
    3 & 5.708 $\pm$ 0.863 & 5.758 $\pm$ 1.107 & \textbf{3.956 $\pm$ 0.120} & -551 $\pm$ 175 & -568 $\pm$ 218 & -259 $\pm$ 19 \\
    10 & 19.122 $\pm$ 2.801 & 18.247 $\pm$ 3.406 & \textbf{12.786 $\pm$ 0.379} & -6{,}180 $\pm$ 1{,}864 & -5{,}698 $\pm$ 2{,}193 & -2{,}709 $\pm$ 183 \\
    30 & 41.921 $\pm$ 11.878 & 47.601 $\pm$ 11.138 & \textbf{33.017 $\pm$ 3.739} & -31{,}408 $\pm$ 17{,}517 & -39{,}550 $\pm$ 18{,}232 & -18{,}272 $\pm$ 3{,}953 \\
    \midrule
    \multicolumn{7}{l}{\emph{$d = 6$}} \\
    \midrule
    0 & 0.245 $\pm$ 0.025 & 0.358 $\pm$ 0.091 & \textbf{0.019 $\pm$ 0.003} & -0.638 $\pm$ 0.362 & -2.69 $\pm$ 1.93 & \textbf{0.990 $\pm$ 0.004} \\
    0.01 & 0.246 $\pm$ 0.021 & 0.374 $\pm$ 0.092 & \textbf{0.021 $\pm$ 0.002} & -0.645 $\pm$ 0.300 & -2.99 $\pm$ 1.94 & \textbf{0.988 $\pm$ 0.003} \\
    0.1 & 0.303 $\pm$ 0.041 & 0.386 $\pm$ 0.078 & \textbf{0.132 $\pm$ 0.004} & -1.52 $\pm$ 0.71 & -3.18 $\pm$ 1.79 & 0.528 $\pm$ 0.034 \\
    0.3 & 0.635 $\pm$ 0.103 & 0.584 $\pm$ 0.083 & \textbf{0.371 $\pm$ 0.009} & -10.2 $\pm$ 3.9 & -8.36 $\pm$ 2.71 & -2.71 $\pm$ 0.22 \\
    1 & 1.710 $\pm$ 0.205 & 1.494 $\pm$ 0.161 & \textbf{1.199 $\pm$ 0.030} & -78.9 $\pm$ 19.2 & -59.8 $\pm$ 13.8 & -37.7 $\pm$ 2.4 \\
    3 & 4.853 $\pm$ 0.529 & 4.349 $\pm$ 0.439 & \textbf{3.483 $\pm$ 0.086} & -641 $\pm$ 145 & -514 $\pm$ 105 & -326 $\pm$ 19 \\
    10 & 16.200 $\pm$ 1.728 & 14.242 $\pm$ 1.380 & \textbf{11.452 $\pm$ 0.268} & -7{,}152 $\pm$ 1{,}564 & -5{,}513 $\pm$ 1{,}114 & -3{,}533 $\pm$ 200 \\
    30 & 35.145 $\pm$ 6.632 & 39.223 $\pm$ 5.564 & \textbf{29.168 $\pm$ 3.390} & -34{,}432 $\pm$ 12{,}616 & -42{,}233 $\pm$ 11{,}693 & -23{,}240 $\pm$ 5{,}481 \\
    \midrule
    \multicolumn{7}{l}{\emph{$d = 8$}} \\
    \midrule
    0 & 0.297 $\pm$ 0.041 & 0.285 $\pm$ 0.059 & \textbf{0.028 $\pm$ 0.002} & -2.33 $\pm$ 0.94 & -2.15 $\pm$ 1.36 & \textbf{0.970 $\pm$ 0.005} \\
    0.01 & 0.297 $\pm$ 0.042 & 0.285 $\pm$ 0.057 & \textbf{0.030 $\pm$ 0.003} & -2.35 $\pm$ 0.96 & -2.14 $\pm$ 1.30 & \textbf{0.966 $\pm$ 0.006} \\
    0.1 & 0.391 $\pm$ 0.055 & 0.317 $\pm$ 0.049 & \textbf{0.120 $\pm$ 0.003} & -4.78 $\pm$ 1.69 & -2.81 $\pm$ 1.24 & 0.462 $\pm$ 0.033 \\
    0.3 & 0.961 $\pm$ 0.159 & 0.506 $\pm$ 0.063 & \textbf{0.335 $\pm$ 0.009} & -34.2 $\pm$ 11.3 & -8.66 $\pm$ 2.47 & -3.17 $\pm$ 0.25 \\
    1 & 1.926 $\pm$ 0.210 & 1.372 $\pm$ 0.137 & \textbf{1.103 $\pm$ 0.028} & -138 $\pm$ 31 & -69.6 $\pm$ 14.9 & -44.2 $\pm$ 2.8 \\
    3 & 5.109 $\pm$ 0.463 & 3.981 $\pm$ 0.331 & \textbf{3.276 $\pm$ 0.079} & -976 $\pm$ 180 & -592 $\pm$ 102 & -398 $\pm$ 23 \\
    10 & 15.555 $\pm$ 1.486 & 12.990 $\pm$ 1.110 & \textbf{10.892 $\pm$ 0.249} & -9{,}057 $\pm$ 1{,}694 & -6{,}309 $\pm$ 1{,}125 & -4{,}406 $\pm$ 243 \\
    30 & 37.122 $\pm$ 7.173 & 36.203 $\pm$ 5.387 & \textbf{26.384 $\pm$ 3.799} & -53{,}027 $\pm$ 19{,}295 & -49{,}726 $\pm$ 13{,}863 & -26{,}401 $\pm$ 7{,}737 \\
    \bottomrule
  \end{tabular}
  }
\end{table}

\begin{table}[ht]
  \centering
  \small
  \caption{Exp.~4.4, relative view: per-seed squared OOD error of the uncommitted arm over the canonical arm, $(\mathrm{RMSE}^{\mathrm{ood}}_{\mathrm{raw}} / \mathrm{RMSE}^{\mathrm{ood}}_{\text{exact-}\varphi})^2$, mean$\pm$std over 100 seeds. Where OOD $R^2$ (\Cref{tab:exp44}) answers the absolute question \emph{does this arm beat predicting the mean} and is negative for every arm once $\sigma$ is large, this ratio answers the complementary relative question \emph{how much does the commitment buy over no commitment}, and stays meaningful after every $R^2$ has gone negative. Reading down the $d{=}4$ column, the advantage compresses from 1{,}957$\times$ at $\sigma{=}0.01$ (off-table) onto a floor near $2\times$ and then stops; the same floor is reached at the higher $d$, only at larger $\sigma$. This approximately constant high-noise ratio is empirical; the evaluation targets are clean, so there is no shared irreducible test-noise term.}
  \label{tab:exp44_ratio}
  \begin{tabular}{lrrr}
    \toprule
    $\sigma$ & $d = 4$ & $d = 6$ & $d = 8$ \\
    \midrule

    0.3 & 2.1 $\pm$ 1.0 & 3.0 $\pm$ 1.1 & 8.5 $\pm$ 2.8 \\
    1 & 1.9 $\pm$ 0.7 & 2.1 $\pm$ 0.5 & 3.1 $\pm$ 0.7 \\
    3 & 2.1 $\pm$ 0.7 & 2.0 $\pm$ 0.4 & 2.4 $\pm$ 0.4 \\
    10 & 2.3 $\pm$ 0.7 & 2.0 $\pm$ 0.4 & 2.1 $\pm$ 0.4 \\
    30 & 1.8 $\pm$ 1.1 & 1.5 $\pm$ 0.6 & 2.2 $\pm$ 1.1 \\
    \bottomrule
  \end{tabular}
\end{table}

\FloatBarrier
\section{Architectural Invariance: Full Grid and Robustness Sweeps}
\label{app:exp33_sweeps}

\begin{figure}[ht]
  \centering
  \includegraphics[width=\linewidth]{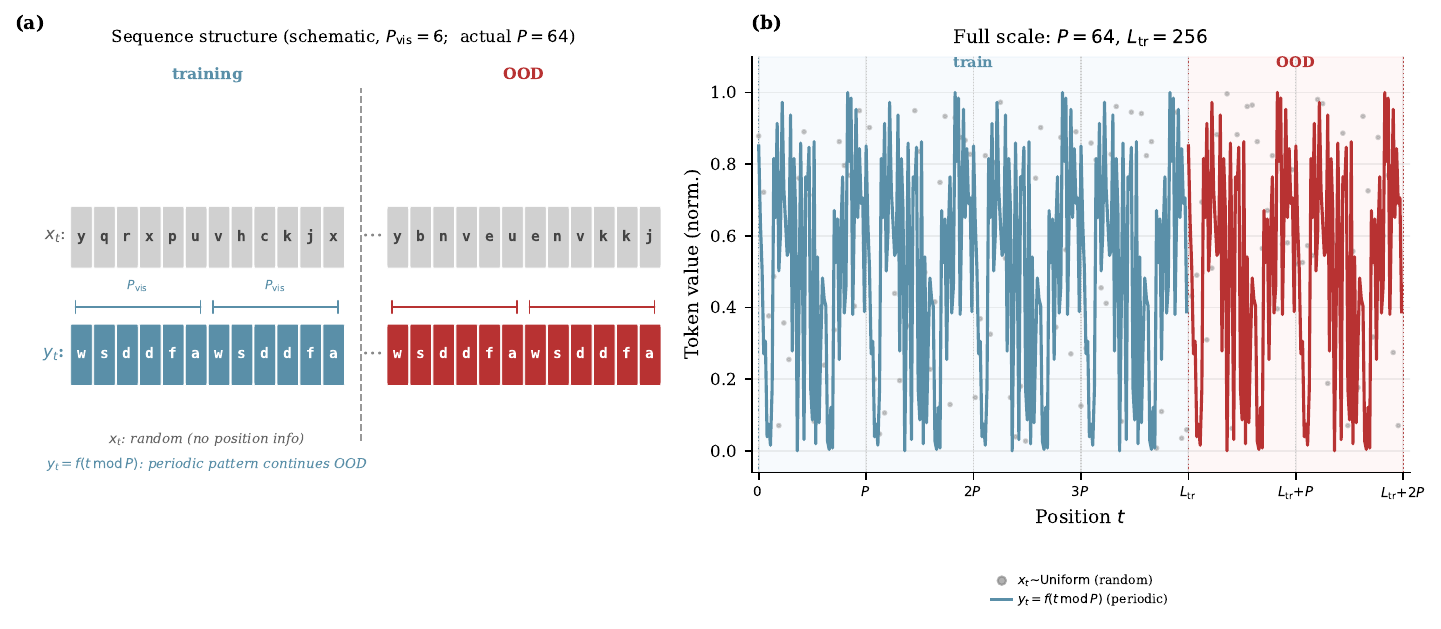}\\[4pt]
  \includegraphics[width=\linewidth]{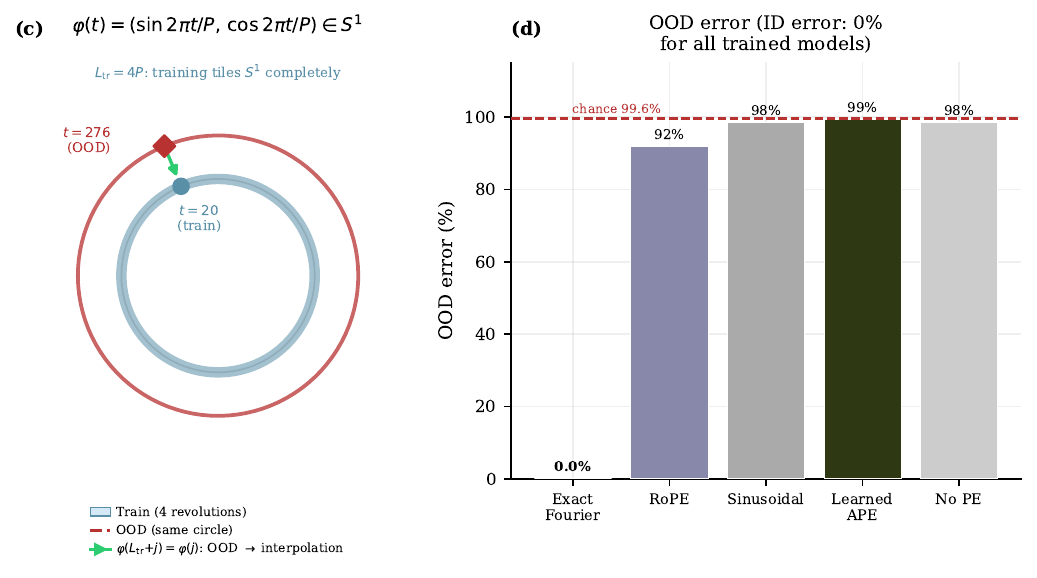}
  \caption{Architectural invariance concept (\Cref{sec:exp33}).
    \emph{(a)} schematic ($P_{\mathrm{vis}}{=}6$): random input $x_t$, periodic target $y_t=f(t\bmod P)$; OOD continues the same pattern past $L_{\mathrm{train}}$.
    \emph{(b)} full scale ($P{=}64$). Token IDs and target labels are divided by $255$ for display only. The task predicts integer class labels.
    \emph{(c)} $\mathbb{S}^1$ geometry: exact-Fourier PE $\varphi(t)=(\sin 2\pi t/P,\cos 2\pi t/P)$ with $L_{\mathrm{train}}{=}4P$ tiles $\mathbb{S}^1$, so OOD reduces to interpolation (\Cref{lem:sin}).
    \emph{(d)} only exact-Fourier achieves near-zero OOD across Transformer, Mamba, and S4D; chance $\approx 99.61\%$. RoPE \citep{su2021roformer}; Sinusoidal \citep{vaswani2017attention}.\label{fig:exp33_concept}}
\end{figure}

Full $(\text{PE} \times P \times \text{architecture})$ grid for the periodic positional-encoding task of \Cref{sec:exp33}.
\textbf{Exact-Fourier PE achieves $<0.22\%$ mean OOD error in every cell} (\Cref{fig:exp33_concept}; $0.0\%$ in 11 of 12 (model, $P$) cells, $0.2\%$ on Transformer at $P{=}512$); all other PEs sit between ${\sim}5.5\%$ (a single high-variance learned-Fourier cell on S4D) and $99\%$ (chance baseline $99.61\%$).
The result is consistent across the three tested architectures. Periodicity of the feature map alone does not guarantee periodic predictions from a causal backbone whose state or attention context can differ across periods.

\begin{table}[ht]
  \centering
  \small
  \caption{Exp.~3.3 Architectural invariance: ID/OOD error (\%) on the periodic annotation task across Transformer, Mamba, S4D and $P \in \{64, 128, 256, 512\}$ ($L_\mathrm{train}{=}4P$, $L_\mathrm{OOD}{=}2P$; chance ${\approx}99.61\%$). \colorbox{gray!12}{\strut Shaded rows} use the correct \textbf{Exact Fourier} PE, which yields $\le 0.2\%$ OOD across all three backbones (the guarantee belongs to the feature map). Mean over 3 seeds; 264 runs total.}
  \label{tab:exp33_main}
  \resizebox{\linewidth}{!}{%
  \begin{tabular}{lrrrrrrrr}
    \toprule
    \textbf{PE} & \multicolumn{2}{c}{$P{=}64$} & \multicolumn{2}{c}{$P{=}128$} & \multicolumn{2}{c}{$P{=}256$} & \multicolumn{2}{c}{$P{=}512$} \\
    \cmidrule(lr){2-3} \cmidrule(lr){4-5} \cmidrule(lr){6-7} \cmidrule(lr){8-9}
     & ID & OOD & ID & OOD & ID & OOD & ID & OOD \\
    \midrule
    \multicolumn{9}{l}{\emph{Transformer}} \\
    \rowcolor{gray!12}
    \textbf{Exact Fourier} & 0.0 $\pm$ 0.0 & \textbf{0.0 $\pm$ 0.0} & 0.0 $\pm$ 0.0 & \textbf{0.0 $\pm$ 0.0} & 0.0 $\pm$ 0.0 & \textbf{0.0 $\pm$ 0.0} & 0.6 $\pm$ 0.5 & \textbf{0.2 $\pm$ 0.3} \\
    Learned Fourier ($k{=}1$) & 5.7 $\pm$ 7.4 & 66.4 $\pm$ 45.8 & 20.9 $\pm$ 15.2 & 78.2 $\pm$ 20.3 & 30.3 $\pm$ 24.3 & 84.4 $\pm$ 13.1 & 49.3 $\pm$ 28.7 & 90.3 $\pm$ 6.3 \\
    Learned Fourier ($k{=}4$) & 0.0 $\pm$ 0.0 & 66.7 $\pm$ 47.1 & 0.0 $\pm$ 0.0 & 32.9 $\pm$ 46.6 & 0.0 $\pm$ 0.0 & 66.3 $\pm$ 43.1 & 0.0 $\pm$ 0.0 & 65.0 $\pm$ 46.0 \\
    Learnable Sinusoidal & 0.0 $\pm$ 0.0 & 93.0 $\pm$ 3.9 & 0.0 $\pm$ 0.0 & 83.0 $\pm$ 11.6 & 0.0 $\pm$ 0.0 & 78.7 $\pm$ 5.1 & 0.0 $\pm$ 0.0 & 59.7 $\pm$ 6.6 \\
    Sinusoidal & 0.0 $\pm$ 0.0 & 98.4 $\pm$ 0.6 & 0.0 $\pm$ 0.0 & 99.0 $\pm$ 0.7 & 0.0 $\pm$ 0.0 & 99.1 $\pm$ 0.4 & 0.0 $\pm$ 0.0 & 99.3 $\pm$ 0.1 \\
    RoPE & 0.0 $\pm$ 0.0 & 90.3 $\pm$ 1.6 & 0.0 $\pm$ 0.0 & 94.8 $\pm$ 0.4 & 0.0 $\pm$ 0.0 & 92.0 $\pm$ 4.6 & 0.0 $\pm$ 0.0 & 98.5 $\pm$ 0.6 \\
    Learned APE & 0.0 $\pm$ 0.0 & 97.7 $\pm$ 1.2 & 0.0 $\pm$ 0.0 & 98.4 $\pm$ 1.1 & 0.0 $\pm$ 0.0 & 99.4 $\pm$ 0.3 & 0.0 $\pm$ 0.0 & 99.2 $\pm$ 0.2 \\
    No PE & 80.3 $\pm$ 0.7 & 98.4 $\pm$ 0.0 & 87.2 $\pm$ 0.2 & 99.2 $\pm$ 0.0 & 90.5 $\pm$ 0.2 & 99.2 $\pm$ 0.0 & 93.5 $\pm$ 0.4 & 99.5 $\pm$ 0.1 \\
    \midrule
    \multicolumn{9}{l}{\emph{Mamba}} \\
    \rowcolor{gray!12}
    \textbf{Exact Fourier} & 0.0 $\pm$ 0.0 & \textbf{0.0 $\pm$ 0.0} & 0.0 $\pm$ 0.0 & \textbf{0.0 $\pm$ 0.0} & 0.0 $\pm$ 0.0 & \textbf{0.0 $\pm$ 0.0} & 0.0 $\pm$ 0.0 & \textbf{0.0 $\pm$ 0.0} \\
    Learned Fourier ($k{=}1$) & 2.5 $\pm$ 3.6 & 66.5 $\pm$ 47.0 & 14.6 $\pm$ 20.2 & 66.3 $\pm$ 46.9 & 48.4 $\pm$ 34.0 & 66.9 $\pm$ 45.0 & 63.8 $\pm$ 15.1 & 83.0 $\pm$ 17.2 \\
    Learned Fourier ($k{=}4$) & 0.0 $\pm$ 0.0 & 66.8 $\pm$ 46.9 & 0.0 $\pm$ 0.0 & 66.5 $\pm$ 47.0 & 0.0 $\pm$ 0.0 & 66.3 $\pm$ 46.9 & 0.0 $\pm$ 0.0 & 99.6 $\pm$ 0.2 \\
    Learnable Sinusoidal & 0.0 $\pm$ 0.0 & 95.8 $\pm$ 2.9 & 0.0 $\pm$ 0.0 & 86.4 $\pm$ 2.7 & 0.0 $\pm$ 0.0 & 90.1 $\pm$ 3.3 & 0.0 $\pm$ 0.0 & 16.9 $\pm$ 3.4 \\
    Sinusoidal & 0.0 $\pm$ 0.0 & 99.0 $\pm$ 1.0 & 0.0 $\pm$ 0.0 & 98.4 $\pm$ 0.6 & 0.0 $\pm$ 0.0 & 98.3 $\pm$ 0.8 & 0.0 $\pm$ 0.0 & 99.0 $\pm$ 0.2 \\
    Learned APE & 0.0 $\pm$ 0.0 & 96.4 $\pm$ 1.1 & 0.0 $\pm$ 0.0 & 97.8 $\pm$ 0.7 & 0.0 $\pm$ 0.0 & 98.8 $\pm$ 0.3 & 0.0 $\pm$ 0.0 & 99.5 $\pm$ 0.1 \\
    No PE & 49.9 $\pm$ 4.1 & 97.9 $\pm$ 0.7 & 75.8 $\pm$ 2.9 & 99.1 $\pm$ 0.5 & 87.3 $\pm$ 2.1 & 98.8 $\pm$ 0.6 & 91.7 $\pm$ 0.8 & 98.9 $\pm$ 0.1 \\
    \midrule
    \multicolumn{9}{l}{\emph{S4D}} \\
    \rowcolor{gray!12}
    \textbf{Exact Fourier} & 0.0 $\pm$ 0.0 & \textbf{0.0 $\pm$ 0.0} & 0.0 $\pm$ 0.0 & \textbf{0.0 $\pm$ 0.0} & 0.0 $\pm$ 0.0 & \textbf{0.0 $\pm$ 0.0} & 0.0 $\pm$ 0.0 & \textbf{0.0 $\pm$ 0.0} \\
    Learned Fourier ($k{=}1$) & 16.4 $\pm$ 23.2 & 55.3 $\pm$ 28.6 & 0.0 $\pm$ 0.0 & 5.5 $\pm$ 7.8 & 9.6 $\pm$ 4.1 & 22.7 $\pm$ 10.2 & 33.5 $\pm$ 32.2 & 43.3 $\pm$ 32.2 \\
    Learned Fourier ($k{=}4$) & 0.0 $\pm$ 0.0 & 59.6 $\pm$ 42.3 & 0.0 $\pm$ 0.0 & 33.2 $\pm$ 47.0 & 0.0 $\pm$ 0.0 & 32.9 $\pm$ 46.5 & 0.0 $\pm$ 0.0 & 7.4 $\pm$ 10.5 \\
    Learnable Sinusoidal & 0.0 $\pm$ 0.0 & 89.2 $\pm$ 4.2 & 0.0 $\pm$ 0.0 & 95.7 $\pm$ 1.2 & 0.0 $\pm$ 0.0 & 96.7 $\pm$ 0.9 & 0.0 $\pm$ 0.0 & 83.1 $\pm$ 1.6 \\
    Sinusoidal & 0.0 $\pm$ 0.0 & 97.9 $\pm$ 0.3 & 0.0 $\pm$ 0.0 & 97.9 $\pm$ 0.7 & 0.0 $\pm$ 0.0 & 99.4 $\pm$ 0.2 & 0.0 $\pm$ 0.0 & 99.2 $\pm$ 0.2 \\
    Learned APE & 0.0 $\pm$ 0.0 & 98.4 $\pm$ 1.1 & 0.0 $\pm$ 0.0 & 99.0 $\pm$ 0.7 & 0.0 $\pm$ 0.0 & 98.7 $\pm$ 0.8 & 0.0 $\pm$ 0.0 & 99.4 $\pm$ 0.0 \\
    No PE & 1.9 $\pm$ 2.7 & 82.7 $\pm$ 2.0 & 53.1 $\pm$ 10.4 & 59.6 $\pm$ 5.4 & 19.8 $\pm$ 25.6 & 22.1 $\pm$ 28.2 & 79.4 $\pm$ 22.2 & 81.5 $\pm$ 21.8 \\
    \midrule
    Chance & --- & 99.61 & --- & 99.61 & --- & 99.61 & --- & 99.61 \\
    \bottomrule
  \end{tabular}%
  }
\end{table}

The headline is robust under independent variations of period at fixed context length ($L_{\mathrm{train}}=2048$, \Cref{app:exp33_sweeps_a}), OOD horizon ($L_{\mathrm{OOD}} \in \{256, 1024, 4096\}$ at $P{=}128$, \Cref{app:exp33_sweeps_b}), SSM state dimension ($d_{\mathrm{state}} \in \{16, 64, 256\}$ at $P{=}512$, \Cref{app:exp33_sweeps_c}), and learned-PE behaviour across $P$ (\Cref{app:exp33_sweeps_l}); per-cell sweep tables follow.

\subsection{Period sweep at fixed $L_{\mathrm{train}}$.}
\label{app:exp33_sweeps_a}
\Cref{tab:exp33_sweepA} reports ID and OOD error when $L_{\mathrm{train}}$ is held fixed at $2048$ and the period $P$ varies; the headline contrast is preserved when the train length is decoupled from the $4P$ baseline.
\begin{table}[ht]
  \centering
  \small
  \caption{Sweep A: $L_\mathrm{train}{=}2048$ held fixed, $P$ varied; decouples the result from the L/P ratio of the headline table.}
  \label{tab:exp33_sweepA}
  \resizebox{\textwidth}{!}{%
  \begin{tabular}{llrrrrrrrr}
    \toprule
    \textbf{Model} & \textbf{PE} & \multicolumn{2}{c}{$P{=}32$} & \multicolumn{2}{c}{$P{=}64$} & \multicolumn{2}{c}{$P{=}128$} & \multicolumn{2}{c}{$P{=}256$} \\
    \cmidrule(lr){3-4} \cmidrule(lr){5-6} \cmidrule(lr){7-8} \cmidrule(lr){9-10}
     & & ID & OOD & ID & OOD & ID & OOD & ID & OOD \\
    \midrule
    \rowcolor{gray!12}
    Transformer & \textbf{Exact Fourier} & 0.0 $\pm$ 0.0 & 0.0 $\pm$ 0.0 & 0.0 $\pm$ 0.0 & 0.0 $\pm$ 0.0 & 0.0 $\pm$ 0.0 & 0.0 $\pm$ 0.0 & 0.1 $\pm$ 0.1 & 0.0 $\pm$ 0.0 \\
    Transformer & Sinusoidal & 0.0 $\pm$ 0.0 & 90.4 $\pm$ 3.2 & 0.0 $\pm$ 0.0 & 90.1 $\pm$ 1.6 & 0.0 $\pm$ 0.0 & 96.6 $\pm$ 0.6 & 0.0 $\pm$ 0.0 & 98.6 $\pm$ 0.2 \\
    Transformer & RoPE & 0.0 $\pm$ 0.0 & 65.3 $\pm$ 8.1 & 0.0 $\pm$ 0.0 & 45.6 $\pm$ 9.4 & 0.0 $\pm$ 0.0 & 74.6 $\pm$ 5.9 & 0.0 $\pm$ 0.0 & 84.3 $\pm$ 5.9 \\
    Transformer & Learned APE & 0.0 $\pm$ 0.0 & 96.2 $\pm$ 0.6 & 0.0 $\pm$ 0.0 & 97.8 $\pm$ 0.5 & 0.0 $\pm$ 0.0 & 98.7 $\pm$ 0.6 & 0.0 $\pm$ 0.0 & 98.5 $\pm$ 0.3 \\
    Transformer & No PE & 90.8 $\pm$ 0.1 & 93.7 $\pm$ 0.0 & 93.6 $\pm$ 0.5 & 96.9 $\pm$ 0.0 & 93.8 $\pm$ 0.1 & 97.7 $\pm$ 0.0 & 94.0 $\pm$ 0.3 & 98.6 $\pm$ 0.4 \\
    \rowcolor{gray!12}
    Mamba & \textbf{Exact Fourier} & 0.0 $\pm$ 0.0 & 0.0 $\pm$ 0.0 & 0.0 $\pm$ 0.0 & 0.0 $\pm$ 0.0 & 0.0 $\pm$ 0.1 & 0.0 $\pm$ 0.1 & 0.1 $\pm$ 0.2 & 0.1 $\pm$ 0.2 \\
    Mamba & Sinusoidal & 0.0 $\pm$ 0.0 & 88.4 $\pm$ 1.3 & 0.0 $\pm$ 0.0 & 87.4 $\pm$ 1.1 & 0.0 $\pm$ 0.0 & 94.4 $\pm$ 1.0 & 0.0 $\pm$ 0.0 & 97.8 $\pm$ 0.3 \\
    Mamba & Learned APE & 0.0 $\pm$ 0.0 & 94.9 $\pm$ 1.3 & 0.0 $\pm$ 0.0 & 97.4 $\pm$ 1.2 & 0.0 $\pm$ 0.0 & 98.3 $\pm$ 0.6 & 0.0 $\pm$ 0.0 & 99.0 $\pm$ 0.1 \\
    Mamba & No PE & 92.8 $\pm$ 0.7 & 95.8 $\pm$ 0.3 & 92.5 $\pm$ 1.0 & 97.5 $\pm$ 0.4 & 92.0 $\pm$ 0.9 & 98.1 $\pm$ 0.4 & 90.1 $\pm$ 2.5 & 98.0 $\pm$ 0.0 \\
    \rowcolor{gray!12}
    S4D & \textbf{Exact Fourier} & 0.0 $\pm$ 0.0 & 0.0 $\pm$ 0.0 & 0.0 $\pm$ 0.0 & 0.0 $\pm$ 0.0 & 0.0 $\pm$ 0.0 & 0.0 $\pm$ 0.0 & 0.0 $\pm$ 0.0 & 0.0 $\pm$ 0.0 \\
    S4D & Sinusoidal & 0.0 $\pm$ 0.0 & 90.9 $\pm$ 0.8 & 0.0 $\pm$ 0.0 & 89.9 $\pm$ 1.9 & 0.0 $\pm$ 0.0 & 96.9 $\pm$ 0.7 & 0.0 $\pm$ 0.0 & 98.9 $\pm$ 0.5 \\
    S4D & Learned APE & 0.0 $\pm$ 0.0 & 96.2 $\pm$ 0.5 & 0.0 $\pm$ 0.0 & 98.3 $\pm$ 0.5 & 0.0 $\pm$ 0.0 & 99.0 $\pm$ 0.2 & 0.0 $\pm$ 0.0 & 99.3 $\pm$ 0.2 \\
    S4D & No PE & 19.2 $\pm$ 22.5 & 24.8 $\pm$ 21.0 & 57.0 $\pm$ 7.2 & 72.9 $\pm$ 13.7 & 67.2 $\pm$ 31.0 & 77.3 $\pm$ 27.0 & 89.0 $\pm$ 3.0 & 93.4 $\pm$ 2.5 \\
    \bottomrule
  \end{tabular}
  }
\end{table}

\subsection{OOD horizon sweep ($L_{\mathrm{OOD}}$).}
\label{app:exp33_sweeps_b}
\Cref{tab:exp33_sweepC} reports OOD error at increasing extrapolation horizons ($P{=}128$, $L_{\mathrm{OOD}}$ varied); exact-Fourier stays flat as predicted by \Cref{lem:sin}.
\begin{table}[ht]
  \centering
  \small
  \caption{Sweep C: extrapolation horizon ($P{=}128$, $L_\mathrm{OOD}$ varied). Exact Fourier stays flat (Lemma~\ref{lem:sin}); learned APE stays at chance.}
  \label{tab:exp33_sweepC}
  \begin{tabular}{llrrr}
    \toprule
    \textbf{Model} & \textbf{PE} & \multicolumn{1}{c}{$L_\mathrm{OOD}{=}256$} & \multicolumn{1}{c}{$L_\mathrm{OOD}{=}1024$} & \multicolumn{1}{c}{$L_\mathrm{OOD}{=}4096$} \\
    \cmidrule(lr){3-3} \cmidrule(lr){4-4} \cmidrule(lr){5-5}
     & & OOD & OOD & OOD \\
    \midrule
    \rowcolor{gray!12}
    Transformer & \textbf{Exact Fourier} & 0.0 $\pm$ 0.0 & 0.0 $\pm$ 0.0 & 0.0 $\pm$ 0.0 \\
    Transformer & Learned APE & 98.4 $\pm$ 1.1 & 98.6 $\pm$ 0.1 & 98.6 $\pm$ 0.1 \\
    \rowcolor{gray!12}
    Mamba & \textbf{Exact Fourier} & 0.0 $\pm$ 0.0 & 0.2 $\pm$ 0.4 & 0.0 $\pm$ 0.0 \\
    Mamba & Learned APE & 98.6 $\pm$ 0.3 & 98.4 $\pm$ 0.1 & 98.5 $\pm$ 0.3 \\
    \rowcolor{gray!12}
    S4D & \textbf{Exact Fourier} & 0.0 $\pm$ 0.0 & 0.0 $\pm$ 0.0 & 0.0 $\pm$ 0.0 \\
    S4D & Learned APE & 98.8 $\pm$ 0.8 & 98.6 $\pm$ 0.2 & 98.6 $\pm$ 0.0 \\
    \bottomrule
  \end{tabular}
\end{table}

\subsection{SSM state-dimension sweep ($d_{\mathrm{state}}$).}
\label{app:exp33_sweeps_c}
\Cref{tab:exp33_sweepD} sweeps $d_{\mathrm{state}}$ at $P{=}512$; the OOD result is independent of state size for the exact-Fourier rows.
\begin{table}[ht]
  \centering
  \small
  \caption{Sweep D: SSM capacity at $P{=}512$, varying $d_\mathrm{state}$ with all else fixed. Closes the ``Mamba was crippled by $d_\mathrm{state}{=}16$'' criticism.}
  \label{tab:exp33_sweepD}
  \resizebox{\textwidth}{!}{%
  \begin{tabular}{llrrrrrr}
    \toprule
    \textbf{Model} & \textbf{PE} & \multicolumn{2}{c}{$d_\mathrm{state}{=}16$} & \multicolumn{2}{c}{$d_\mathrm{state}{=}64$} & \multicolumn{2}{c}{$d_\mathrm{state}{=}256$} \\
    \cmidrule(lr){3-4} \cmidrule(lr){5-6} \cmidrule(lr){7-8}
     & & ID & OOD & ID & OOD & ID & OOD \\
    \midrule
    \rowcolor{gray!12}
    Mamba & \textbf{Exact Fourier} & 0.2 $\pm$ 0.3 & 0.2 $\pm$ 0.3 & 0.1 $\pm$ 0.1 & 0.1 $\pm$ 0.1 & 0.2 $\pm$ 0.1 & 0.1 $\pm$ 0.1 \\
    Mamba & Learned APE & 0.0 $\pm$ 0.0 & 99.5 $\pm$ 0.1 & 0.0 $\pm$ 0.0 & 99.3 $\pm$ 0.1 & 0.0 $\pm$ 0.0 & 99.2 $\pm$ 0.2 \\
    Mamba & No PE & 92.1 $\pm$ 0.8 & 98.9 $\pm$ 0.1 & 92.4 $\pm$ 0.9 & 99.2 $\pm$ 0.3 & 91.4 $\pm$ 0.5 & 99.0 $\pm$ 0.3 \\
    \rowcolor{gray!12}
    S4D & \textbf{Exact Fourier} & 0.0 $\pm$ 0.0 & 0.0 $\pm$ 0.0 & 0.1 $\pm$ 0.1 & 0.2 $\pm$ 0.1 & 1.5 $\pm$ 2.1 & 1.9 $\pm$ 2.7 \\
    S4D & Learned APE & 0.0 $\pm$ 0.0 & 99.4 $\pm$ 0.2 & 0.0 $\pm$ 0.0 & 99.4 $\pm$ 0.3 & 0.0 $\pm$ 0.0 & 99.2 $\pm$ 0.1 \\
    S4D & No PE & 0.0 $\pm$ 0.0 & 0.0 $\pm$ 0.0 & 79.4 $\pm$ 22.3 & 81.5 $\pm$ 21.8 & 98.0 $\pm$ 0.6 & 98.9 $\pm$ 0.0 \\
    \bottomrule
  \end{tabular}
  }
\end{table}

\subsection{Period-length sweep ($P$) for learned PE variants.}
\label{app:exp33_sweeps_l}
\Cref{tab:exp33_sweepL} reports the learned PE variants across $P\in\{64,128,256,512\}$ at the baseline configuration; SGD does not recover the correct frequency from the specified initialisations, so the learned-Fourier and learnable-sinusoidal arms remain far from exact-Fourier across all four periods.
\begin{table}[ht]
  \centering
  \small
  \caption{Sweep L: learned-PE variants at baseline ($L_\mathrm{train}{=}4P$, $L_\mathrm{OOD}{=}2P$, $d_\mathrm{state}{=}64$); tests whether SGD recovers the correct frequency from the specified initialisations and matches \textbf{exact Fourier}.}
  \label{tab:exp33_sweepL}
  \resizebox{\textwidth}{!}{%
  \begin{tabular}{llrrrrrrrr}
    \toprule
    \textbf{Model} & \textbf{PE} & \multicolumn{2}{c}{$P{=}64$} & \multicolumn{2}{c}{$P{=}128$} & \multicolumn{2}{c}{$P{=}256$} & \multicolumn{2}{c}{$P{=}512$} \\
    \cmidrule(lr){3-4} \cmidrule(lr){5-6} \cmidrule(lr){7-8} \cmidrule(lr){9-10}
     & & ID & OOD & ID & OOD & ID & OOD & ID & OOD \\
    \midrule
    Transformer & Learned Fourier ($k{=}1$) & 5.7 $\pm$ 7.4 & 66.4 $\pm$ 45.8 & 20.9 $\pm$ 15.2 & 78.2 $\pm$ 20.3 & 30.3 $\pm$ 24.3 & 84.4 $\pm$ 13.1 & 49.3 $\pm$ 28.7 & 90.3 $\pm$ 6.3 \\
    Transformer & Learned Fourier ($k{=}4$) & 0.0 $\pm$ 0.0 & 66.7 $\pm$ 47.1 & 0.0 $\pm$ 0.0 & 32.9 $\pm$ 46.6 & 0.0 $\pm$ 0.0 & 66.3 $\pm$ 43.1 & 0.0 $\pm$ 0.0 & 65.0 $\pm$ 46.0 \\
    Transformer & Learnable Sinusoidal & 0.0 $\pm$ 0.0 & 93.0 $\pm$ 3.9 & 0.0 $\pm$ 0.0 & 83.0 $\pm$ 11.6 & 0.0 $\pm$ 0.0 & 78.7 $\pm$ 5.1 & 0.0 $\pm$ 0.0 & 59.7 $\pm$ 6.6 \\
    Mamba & Learned Fourier ($k{=}1$) & 2.5 $\pm$ 3.6 & 66.5 $\pm$ 47.0 & 14.6 $\pm$ 20.2 & 66.3 $\pm$ 46.9 & 48.4 $\pm$ 34.0 & 66.9 $\pm$ 45.0 & 63.8 $\pm$ 15.1 & 83.0 $\pm$ 17.2 \\
    Mamba & Learned Fourier ($k{=}4$) & 0.0 $\pm$ 0.0 & 66.8 $\pm$ 46.9 & 0.0 $\pm$ 0.0 & 66.5 $\pm$ 47.0 & 0.0 $\pm$ 0.0 & 66.3 $\pm$ 46.9 & 0.0 $\pm$ 0.0 & 99.6 $\pm$ 0.2 \\
    Mamba & Learnable Sinusoidal & 0.0 $\pm$ 0.0 & 95.8 $\pm$ 2.9 & 0.0 $\pm$ 0.0 & 86.4 $\pm$ 2.7 & 0.0 $\pm$ 0.0 & 90.1 $\pm$ 3.3 & 0.0 $\pm$ 0.0 & 15.1 $\pm$ 4.3 \\
    S4D & Learned Fourier ($k{=}1$) & 16.4 $\pm$ 23.2 & 55.3 $\pm$ 28.6 & 0.0 $\pm$ 0.0 & 5.5 $\pm$ 7.8 & 9.6 $\pm$ 4.1 & 22.7 $\pm$ 10.2 & 33.5 $\pm$ 32.2 & 43.3 $\pm$ 32.2 \\
    S4D & Learned Fourier ($k{=}4$) & 0.0 $\pm$ 0.0 & 59.6 $\pm$ 42.3 & 0.0 $\pm$ 0.0 & 33.2 $\pm$ 47.0 & 0.0 $\pm$ 0.0 & 32.9 $\pm$ 46.5 & 0.0 $\pm$ 0.0 & 7.4 $\pm$ 10.5 \\
    S4D & Learnable Sinusoidal & 0.0 $\pm$ 0.0 & 89.2 $\pm$ 4.2 & 0.0 $\pm$ 0.0 & 95.7 $\pm$ 1.2 & 0.0 $\pm$ 0.0 & 96.7 $\pm$ 0.9 & 0.0 $\pm$ 0.0 & 83.1 $\pm$ 1.6 \\
    \bottomrule
  \end{tabular}
  }
\end{table}

\FloatBarrier
\section{Experimental Details}
\label{app:experiments}

\paragraph{Common configuration.}
Unless stated otherwise, results are mean $\pm$ std over $3$ random seeds.
The default MLP backbone is $5$ hidden layers with $256$ Tanh units, trained with Adam (lr~$10^{-3}$) on MSE; per-experiment epoch counts and any deviations are listed below.
OLS uses the closed-form normal equations with no regularisation.
SINDy uses \texttt{pysindy} v$2.1.0$ with STLSQ; per-experiment polynomial / Fourier libraries and sparsity thresholds are stated with each experiment.
The Tier-1 toy battery, Tier-3 selection-protocol study, and Tier-4 torus benchmark draw inputs from synthetic distributions on $\mathbb{R}$, $\mathbb{R}^+$, and $S^1\!\times\!\cdots\!\times\!S^1$ respectively; no external dataset is used.

\paragraph{Exp.~1.1: $\sin(x)$ vs.\ Taylor-9 (\Cref{sec:exp11}).}
\textbf{Dataset:} $f(x)=\sin(x)$ on $W=[-\pi, \pi]$ vs.\ the degree-$9$ Taylor expansion as the observationally-equivalent confound; OOD $[\pi, 3\pi]$. $N_{\mathrm{train}}=N_{\mathrm{OOD}}=128$, no observation noise (clean grid).
\textbf{Methods:} raw / Fourier $\varphi(x)=(\sin x, \cos x)$ feature maps crossed with $\{$OLS, MLP, SINDy$\}$. Foundation-model arms: TabPFN (cloud zero-shot tabular) and TimesFM-1.0-200M via \texttt{forecast\_with\_covariates} (context length $=$ horizon $=128$).
\textbf{Optimisation:} MLPs trained for $3{,}000$ epochs. SINDy uses a single library across both feature maps (\texttt{ps.PolynomialLibrary(degree=9, include\_bias=True)}, STLSQ threshold~$10^{-6}$, \texttt{max\_iter}~$=50$), applied identically on the raw input $x$ and on the Fourier-feature input $(\sin x, \cos x)$, matching the polynomial-library convention used by Exp.~1.2, 2.1, 2.2 (Exp.~1.3 is the only place SINDy ranges over Fourier-family inputs).

\paragraph{Exp.~1.2: power-law vs.\ exponential (\Cref{sec:exp12}).}
\textbf{Dataset:} $P_1(x)=x^2$ and $P_2(x)=e^{\alpha(x-1)}$ with $\alpha = 2\ln 2 \approx 1.386$ (chosen so $P_1, P_2$ agree at $x=1, 2$); $W=[1, 2]$, OOD $[2, 10]$. $n_{\mathrm{train}}=150$, $n_{\mathrm{test}}=300$ (both ID and OOD), Gaussian observation noise $\sigma=0.2$ on $y$.
\textbf{Methods:} three feature maps (\texttt{log-log}, \texttt{log-}$y$, \texttt{raw}) crossed with three model classes (OLS, MLP, SINDy), plus TabPFN and TimesFM in three input-format variants each. TimesFM uses a $128$-point context on $W$ and forecasts the full OOD window at the matching step ($\approx 1{,}016$ points).
\textbf{Optimisation:} MLPs trained for $2{,}000$ epochs. SINDy polynomial library degree~$1$, threshold~$0.05$ (the canonically-correct DGP is affine in transformed coordinates).

\paragraph{SINDy Diagnostic: Full SINDy diagnostic Battery (\Cref{app:exp13}).}
\textbf{Dataset:} eight univariate DGPs ($\sin x$, $\sin 2x$, $x^2$, $x^3+x$, $G_1=\sin x \cos 2x$, $G_2=\tanh(5\sin x \cos 2x)$, $I_1=\exp(2\log x - \log^2 x)$, $I_2=\exp(1 - 2 e^{-x})$) crossed with a candidate set of feature maps in the Fourier and log-polynomial families, yielding $35$ (DGP, $\varphi$) pairs. $N_{\mathrm{train}}=200$ on the per-DGP training range, $N_{\mathrm{OOD}}=400$ on a contiguous OOD interval.
\textbf{Method:} for each pair, SINDy fits a sparse model with PolynomialLibrary at pair-specific degree~$1$ or $2$, threshold~$0.05$, \texttt{max\_iter}~$=50$. The diagnostic $\delta_{\mathrm{OOD}}$ is the squared-NMSE of the OOD derivative residual, with derivatives evaluated by finite differences (\texttt{np.gradient}).

\paragraph{Exp.~2.1: MAK vector field (\Cref{sec:exp21}).}
\textbf{Dataset:} synthetic 2-node mass-action ODE $\dot{x}_i = F - B x_i - R x_i x_j$ with $F=0.5, B=0.1, R=1.0$~\citep{vasiliauskaite2024generalization}; training samples drawn uniformly from $(x_1, x_2)\in[0.4, 1.2]^2$, OOD from $[0, 2.0]^2 \setminus [0.4, 1.2]^2$.
\textbf{Method:} OLS / SINDy on the bilinear feature map $\varphi(x)=[1, x_i, x_i x_j]$ vs.\ raw / wrong baselines.

\paragraph{Exp.~2.2: Kepler (\Cref{sec:exp22}).}
\textbf{Dataset:} NASA Exoplanet Archive snapshot mirrored as the Hugging Face dataset \texttt{juliensimon/nasa-exoplanets} ($6{,}158$ confirmed exoplanets); we keep rows where $(a, T, M_{\mathrm{star}}, R_{\mathrm{star}}, T_{\mathrm{eff}})$ are all positive-valued and split by semi-major axis: $n_{\mathrm{train}}=1{,}881$ at $a<0.5$~AU, $n_{\mathrm{OOD}}=481$ at $0.5\le a<20$~AU (total $n=2{,}362$ used in the headline two-variable Kepler experiment).
\textbf{Method:} no ODE integration is involved; the data are catalogued $(T, a, M_{\mathrm{star}})$ tuples, and feature-conditioned OLS / SINDy fit the candidate coordinate transforms (correct: $(\log a, \log M_{\mathrm{star}}) \to \log T$; wrong-direction baselines as in \Cref{tab:exp22_main}).

\paragraph{Exp.~2.3: Cross-species CDS (\Cref{sec:exp23}).}
\textbf{Datasets:}
\textit{Train}: \textit{Saccharomyces cerevisiae} S288C reference genome R64-1-1 (Ensembl release-112), chromosomes I--IX; chr~XI held out as the matched-organism ID baseline.
\textit{Zero-shot OOD eval} ($5$ organisms, NCBI RefSeq, taxonomy in parentheses): \textit{Schizosaccharomyces pombe} ASM294v3 (\texttt{GCF\_000002945.2}, $36\%$ GC), \textit{Escherichia coli} K-12 MG1655 ASM584v2 (\texttt{GCF\_000005845.2}, $50\%$ GC), \textit{Bacillus subtilis} str.~168 ASM904v1 (\texttt{GCF\_000009045.1}, $44\%$ GC), \textit{Mycobacterium tuberculosis} H37Rv ASM19595v2 (\texttt{GCF\_000195955.2}, $65\%$ GC), \textit{Plasmodium falciparum} 3D7 (\texttt{GCF\_000002765.6}, $19\%$ GC); CDS coordinates from each organism's GFF3.
Reads are $300$\,bp, with $5{,}000$ training reads per class (CDS / non-CDS) and $1{,}000$ eval reads per class per organism, max token length $400$, with clean any-strand labels (reverse-complement of minus-strand positives). Positive read centres lie in CDS; negative read spans contain no CDS on either strand. Genome annotations determine the orientation of positive reads.
\textbf{Architecture:} all four variants use a random-init TinyTransformer encoder ($d_{\mathrm{model}}=128$, $2$ layers, $4$ attention heads, FFN $256$, vocabulary $\{A, C, G, T, N\}$ + special tokens).
Each variant has its own parameters and is trained separately.
The pseudocode below covers the four columns of \Cref{tab:exp23_species_main}.
Here $x$ includes [CLS], [SEP], and padding, and $t$ indexes token positions.
\texttt{Encode} includes input normalisation and dropout, the Transformer with its padding mask, and output normalisation.
\texttt{MLP} denotes Linear$(d_{\mathrm{in}},128)$ $\to$ Tanh $\to$ Dropout$(0.1)$ $\to$ Linear$(128,2)$, producing CDS/non-CDS logits.
\texttt{APE} is learned absolute positional encoding, \texttt{SinPE} is fixed sinusoidal encoding, and $\varphi_3(t)=(\sin(2\pi t/3),\cos(2\pi t/3))$.
The codon variants learn a projection $W\in\mathbb{R}^{128\times2}$, initialised to zero, while keeping the period fixed at $3$.

\smallskip
{\small
\noindent\begin{minipage}{\linewidth}
\textbf{Forward pass: shared computation and variant-specific additions.}
\begin{verbatim}
if variant == "APE":
    PE = APE(t)
elif variant == "sin":
    PE = SinPE(t)
else:                                        # codon, codon+biology
    PE = APE(t) + W * phi3(t)

H = Encode(Embed(x) + PE)
z = [H[CLS]]                                 # all four variants
if variant in {"codon", "codon+biology"}:
    z += [FFT3(H[1:])]                       # add 128 features
if variant == "codon+biology":
    b = FFT3(ACGT_indicators(x[1:]))          #   4 features
    s = sort(stop_counts(x[1:], frames=[0,1,2]))  # 3 features
    q = AA_frequencies(x[1:], frames=[0,1,2]) # 3 x 21
    q = flatten(sort_rows(q, by=stop_fraction))  # 63 features
    z += [b, s, q]
logits = MLP(concat(z))
\end{verbatim}
\end{minipage}
}
The classifier input has $128$ features for APE and sin, $256$ for codon, and $326$ for codon+biology.

\textbf{Feature operators:} \texttt{FFT3} subtracts each channel's mean over non-padding positions, zeros padding, and takes the Fourier magnitude at bin $\operatorname{round}(T/3)$.
Here $T$ is the padded length after removing [CLS]. The transform runs along the sequence dimension.
\texttt{ACGT\_indicators} supplies four binary channels, one per nucleotide.
\texttt{stop\_counts} counts TAA, TAG, and TGA in each reading frame. \texttt{sort} orders these three counts in ascending order.
\texttt{AA\_frequencies} translates codons with the standard genetic code and returns one normalised histogram per frame (20 amino-acid classes plus stop).
Each histogram is divided by its valid-codon count. Both codon operators exclude triplets containing N, padding, or special tokens.
\texttt{sort\_rows} orders the three histograms by increasing stop fraction before concatenation.
Thus \texttt{FFT3(H[1:])} comes from learned hidden states, while $b$, $s$, and $q$ come directly from the DNA sequence.
Fourier magnitudes are invariant to cyclic shifts of the transformed signal. This property does not imply invariance of the complete models, which retain [CLS] and absolute PE.

\textbf{Optimisation:} AdamW (lr~$10^{-3}$, weight decay $0.01$, batch $64$, $30$ epochs, gradient clip $1.0$).

\paragraph{Exp.~3.1: Near-boundary selection protocol (\Cref{sec:exp31}).}
\textbf{Dataset:} $f(x)=\sin(x)$ with Gaussian observation noise (default $\sigma=10^{-2}$). Training window $[-\pi, 0]$ ($n=200$), near-boundary validation window $[0, \pi]$ ($n=100$, \ie $n_{\mathrm{train}}/2$), OOD window $[\pi, 3\pi]$ ($n=200$).
\textbf{Method:} four OLS candidates compete: $\{$\texttt{fourier\_ols}, \texttt{poly7\_ols}, \texttt{poly9\_ols}, \texttt{raw\_ols}$\}$ (random-baseline accuracy $0.25$). The CV baseline is $5$-fold KFold scored on training data; near-boundary scores each candidate on the held-out $[0, \pi]$ window. Headline numbers average $100$ trials $\times$ $3$ seeds. The robustness sweep cycles $\sigma \in \{10^{-3}, 3{\cdot}10^{-3}, 10^{-2}, 3{\cdot}10^{-2}, 10^{-1}, 3{\cdot}10^{-1}, 1\}$ at $100$ trials each.

\paragraph{Exp.~3.1: near-boundary width sweep and the $t_V$ collapse (\Cref{sec:exp31}).}
\textbf{Design:} the headline protocol gives the near-boundary segment the full width of the training window, which leaves open how thin the segment may be before selection fails.
We therefore sweep the segment $V=[0, \delta_V]$ over $\delta_V \in \{0.01, 0.02, 0.04, 0.08, 0.16, 0.32, 0.64, 1.28, \pi\}$, \ie from well below to exactly the training width $|W|=\pi$, crossed with segment sizes $n_V \in \{10, 25, 50, 100\}$ and noise levels $\sigma \in \{3{\cdot}10^{-3}, 10^{-2}, 3{\cdot}10^{-2}, 10^{-1}\}$, holding the DGP, the four-candidate library, the training window $[-\pi, 0]$ with $n_{\mathrm{train}}=200$, and the OOD window $[\pi, 3\pi]$ fixed at their Exp.~3.1 values.
Every cell is scored over $100$ trials $\times$ $3$ seeds ($42, 43, 44$), and accuracy is agreement with the oracle (lowest-OOD-error) candidate, as in the headline table.
Both label regimes were run, a clean segment and a segment carrying the same observation noise $\sigma$ as the training data; \Cref{tab:exp31_width_sweep} reports the noisy one, which is the like-for-like comparison against cross-validation.

\textbf{Behaviour at the geometric floor:} fitting the polynomial candidates on $W$ and measuring their in-window agreement level $\eps_W$ and their initial divergence rate $r$ just outside it gives the floor $\sqrt{3}\,\eps_W/r = 0.017$ for the degree-$7$ candidate (binding) and $0.009$ for degree-$9$, which we use as empirical reference widths. The fitted rates do not verify the uniform divergence assumption of \Cref{cor:nearboundary}. ($r$ is fitted on $u \in (0, 0.02]$; the floor is sensitive to that range.)
The sweep shows that these reference widths are insufficient for reliable selection.
At the floor itself ($\delta_V = 0.02$, only $0.6\%$ of the training window) accuracy is $0.57$ to $0.69$, and even at $\delta_V = 0.04$, more than twice the floor, it reaches only $0.72$, still short of the $5$-fold cross-validation baseline of $0.813$.
The first swept width to overtake cross-validation is $\delta_V=0.16$, $5.1\%$ of the training window ($0.82$ at $n_V=25$, $0.90$ at $n_V=100$), and $\delta_V=0.64$, $20.4\%$ of it, reaches $0.99$.
The thin segment of \Cref{sec:exp31} therefore survives, but at a width the geometry alone does not predict: $\sqrt{3}\,\eps_W/r$ guarantees a population gap under the corollary's assumptions, not that a finite noisy sample can \emph{resolve} the difference.

\textbf{Collapse onto $t_V$:} \Cref{prop:strip} supplies the missing quantity, and the sweep is governed by it.
Computing $t_V = \sqrt{n_V}\,\norm{\Delta}_{L^2(\mathcal{D}_V)}/\sigma$ per cell, with $\Delta$ measured between the oracle candidate's fit and its hardest wrong competitor on $V$ (so the four-way selection is scored through its hardest two-point subproblem), the $144$ cells collapse onto $t_V$ alone: bins that pool very different $(\delta_V, n_V, \sigma)$ combinations agree to within $0.065$, and holding $\delta_V=0.16$ fixed while moving only $n_V$ and $\sigma$ carries $t_V$ from $2.6$ to $8.8$ and accuracy from $0.75$ to $0.90$ (correlation $0.90$), so none of the three factors predicts the outcome on its own.
All $71$ cells with $t_V \ge 6$ beat cross-validation and, of the $73$ below, $69$ do not; the best single threshold, $t_V = 5.67$, classifies $141$ of $144$ cells, and the split survives scoring each cell against its own $\sigma$'s cross-validation baseline.
The constant is rule-dependent, not universal: \Cref{prop:strip} makes the likelihood-ratio test correct with probability exactly $\Phi(t_V/2)$, so setting $\Phi(t_V/2) = 0.813$ gives $t_V = 1.8$ for the optimal test between two \emph{known} candidates, whereas the relative-error rule actually used in Exp.~3.1 selects among four \emph{fitted} ones and crosses the same baseline later, at $t_V \approx 6$.

\begin{table}[ht]
  \centering
  \small
  \caption{Exp.~3.1 (near-boundary width sweep). Block~(i): selection accuracy on a near-boundary segment $V=[0, w]$ carrying $n_V$ points, at $\sigma=10^{-2}$ with segment labels carrying the same noise as the training data. Training window $W=[-\pi, 0]$, $|W|=\pi$, $n_{\mathrm{train}}=200$; accuracy is agreement with the oracle (lowest-OOD-error) candidate over 100 trials $\times$ 3 seeds (42, 43, 44) $=300$ draws per cell. The geometric floor $\sqrt{3}\,\eps_W/r=0.017$ of \Cref{cor:nearboundary} (poly-7 binding, poly-9 gives 0.009) is marked; entries above the $5$-fold CV baseline $0.813$ are bold. Block~(ii): the same sweep re-run over $9$~widths $\times$ $4$~segment sizes $\times$ $4$~noise levels $\sigma \in \{0.003, 0.01, 0.03, 0.1\}$ $=144$ cells, binned by the Prop.~\ref{prop:strip} statistic $t_V=\sqrt{n_V}\,\norm{\Delta}_{L^2(\mathcal{D}_V)}/\sigma$ alone ($\Delta$ measured per cell against the hardest wrong candidate). $\#w$, $\#n_V$, $\#\sigma$ count the distinct values pooled in each bin: cells sharing a $t_V$ agree to within 0.065 however they are built. All 71 cells with $t_V \ge 6$ beat cross-validation; of the 73 below, 69 do not.}
  \label{tab:exp31_width_sweep}
  \begin{tabular}{lrrrrr}
    \toprule
    \multicolumn{6}{l}{\textbf{(i)} Accuracy vs.\ segment width and size ($\sigma=10^{-2}$)} \\
    \midrule
    $w$ & $w/|W|$ (\%) & $n_V{=}10$ & $n_V{=}25$ & $n_V{=}50$ & $n_V{=}100$ \\
    \cmidrule(lr){1-6}
    0.01 & 0.32 & 0.543 & 0.623 & 0.633 & 0.717 \\
    \multicolumn{6}{l}{\quad\itshape $\sqrt{3}\,\eps_W/r = 0.017$: \Cref{cor:nearboundary} geometric floor} \\
    0.02 & 0.64 & 0.567 & 0.677 & 0.693 & 0.660 \\
    0.04 & 1.27 & 0.600 & 0.663 & 0.720 & 0.723 \\
    0.08 & 2.55 & 0.647 & 0.687 & 0.747 & 0.780 \\
    0.16 & 5.09 & 0.767 & \textbf{0.817} & \textbf{0.890} & \textbf{0.900} \\
    0.32 & 10.19 & \textbf{0.910} & \textbf{0.943} & \textbf{0.970} & \textbf{0.967} \\
    0.64 & 20.37 & \textbf{0.990} & \textbf{0.993} & \textbf{0.987} & \textbf{0.987} \\
    1.28 & 40.74 & \textbf{1.000} & \textbf{1.000} & \textbf{1.000} & \textbf{1.000} \\
    3.14 & 100.00 & \textbf{1.000} & \textbf{1.000} & \textbf{1.000} & \textbf{1.000} \\
    \cmidrule(lr){1-6}
    $5$-fold CV & --- & \multicolumn{4}{c}{0.813} \\
    \midrule
    \multicolumn{6}{l}{\textbf{(ii)} The same 144 cells binned by $t_V$ alone} \\
    \cmidrule(lr){1-6}
    $t_V$ & cells & $\#w$ & $\#n_V$ & $\#\sigma$ & accuracy \\
    \cmidrule(lr){1-6}
    $<2$ & 26 & 4 & 2 & 4 & 0.622 $\pm$ 0.042 \\
    $2$--$4$ & 36 & 5 & 4 & 4 & 0.698 $\pm$ 0.044 \\
    $4$--$8$ & 19 & 4 & 4 & 4 & 0.829 $\pm$ 0.065 \\
    $8$--$16$ & 11 & 2 & 3 & 4 & 0.930 $\pm$ 0.028 \\
    $16$--$32$ & 7 & 2 & 2 & 4 & 0.973 $\pm$ 0.014 \\
    $32$--$64$ & 10 & 3 & 3 & 4 & 0.993 $\pm$ 0.006 \\
    $\ge 64$ & 35 & 3 & 4 & 4 & 0.999 $\pm$ 0.003 \\
    \bottomrule
  \end{tabular}
\end{table}

\paragraph{Exp.~3.2: $\delta_{\mathrm{OOD}}$ correlation diagnostic (\Cref{sec:exp32}).}
\textbf{Inputs:} the $35$ (DGP, $\varphi$) pairs of Exp.~1.3 (Fourier and log-polynomial families), augmented in the consolidated table by every SINDy-regime configuration of Exp.~1.1, Exp.~1.2, Exp.~2.1 (MAK), and Exp.~2.2 (Kepler), each documented with its own (DGP, $\varphi$) grid above.
\textbf{Method:} $\delta_{\mathrm{OOD}}$ is the squared-NMSE residual in the fitted target space after STLSQ fitting. The toy battery uses finite-difference derivative targets, threshold~$0.05$, and \texttt{max\_iter}~$=50$; the consolidated table also includes static regression targets, with per-experiment settings given above. The diagnostic requires labelled OOD observations or derivatives. The reported correlation is Spearman $\rho$ between $\delta_{\mathrm{OOD}}$ and a binary correctness label (canonically-correct vs.\ wrong) over the $35$-pair toy battery.

\paragraph{Exp.~3.3: Architectural invariance (\Cref{sec:exp33}).}
\textbf{Dataset:} synthetic periodic-token task with vocabulary size $256$, target $y_t = f(t \bmod P)$ for the fixed random lookup table $f$ defined in \Cref{sec:exp33} and i.i.d.\ random input tokens; $10{,}000$ training sequences of length $L_{\mathrm{train}}=4P$ and $2{,}000$ ID + $2{,}000$ OOD eval sequences of length $L_{\mathrm{train}}+L_{\mathrm{OOD}}$ with $L_{\mathrm{OOD}}=2P$, periods $P\in\{64, 128, 256, 512\}$.
We generate $f$ with seed $0$ and share it across data splits, backbones, PE choices, and training seeds at each $P$.
\textbf{Architectures (three backbones, all $4$-layer, $n_{\mathrm{embd}}=256$):}
(i) \textit{Causal Transformer}: $4$ attention heads, learnable per-PE injection at the input;
(ii) \textit{Mamba} (selective SSM, \citealp{gu2023mamba}): default \texttt{mamba-ssm} block stack at $d_{\mathrm{model}}=256$;
(iii) \textit{S4D} \citep{gu2022s4d}: diagonal SSM with state dimension $d_{\mathrm{state}}=64$.
The eight PE arms (sinusoidal, RoPE, learned APE, learnable sinusoidal, learned-Fourier $k\in\{1, 4\}$, exact-Fourier $\varphi(t)=(\sin 2\pi t/P, \cos 2\pi t/P)$, no-PE) plug into a uniform input-injection slot.
\textbf{Optimisation:} AdamW (lr~$3{\times}10^{-4}$, weight decay $0.1$, batch $64$, $5{,}000$ steps with $200$-step warmup and cosine decay to $10\%$, early stopping at patience $10$ evaluations).
The learnable sinusoidal arm starts at $\omega_k=10000^{-k/d_{\mathrm{model}}}$, while fixed sinusoidal PE uses $\omega_k=10000^{-2k/d_{\mathrm{model}}}$, for $k=0,\ldots,d_{\mathrm{model}}/2-1$.
The $(\text{backbone} \times \text{PE} \times P \times \text{seed})$ grid (with RoPE only on Transformer) yields the $264$-run headline grid (= $156$ primary block + $108$ Sweep~L runs that populate the learned-Fourier and learnable-sinusoidal rows of \Cref{tab:exp33_main}); \Cref{app:exp33_sweeps} also reports the $L_{\mathrm{train}}$, $L_{\mathrm{OOD}}$, $d_{\mathrm{state}}$, and PE-init sweeps.

\paragraph{Exp.~4: Torus benchmark (\Cref{sec:scope}).}
\textbf{Common setup.} Periodic $d$-torus $\mathbb{T}^d = [0, 2\pi)^d$ with exact feature map $\varphi(x) = (\sin x_j, \cos x_j)_{j=1}^d$; OOD samples are obtained by an integer $2\pi$-block shift of every coordinate ($k_{\mathrm{block}}=1$), so OOD points map under $\varphi$ to the same $\mathbb{S}^1$ images as training points by construction. Backbone is the default MLP ($5$ Tanh layers, $256$ hidden units), Adam lr $10^{-3}$, MSE, $3{,}000$ epochs. Held-out test sets: $N_{\mathrm{ID}} = N_{\mathrm{OOD}} = 2{,}048$, $100$-seeds. $N_{\mathrm{train}}=4{,}096$, $d \in \{1, 2, 4, 6, 8, 12\}$. Target $g_{\mathrm{int}}(\theta) = (1-\lambda)\,\frac{1}{d}\sum_j \sin\theta_j + \lambda \prod_j \sin\theta_j$ with $\lambda = 0.35$. Baselines: raw MLP (input dim $d$), exact-$\varphi$ MLP (input dim $2d$), and a learned-Fourier MLP whose per-coordinate log-frequencies are initialised log-uniformly in $[\log 0.1, \log 10]$ (deliberately misspecified).
\textbf{critical coverage}: $d \in \{4, 6, 8\}$, coverage sweep $M \in \{2, 3, 4, 5, 6\}$ with $N_{\mathrm{train}} = \mathrm{round}(M^d)$ (up to ${\sim}2{\cdot}10^5$ samples). Target $g_{d, d}(\theta) = \binom{d}{d}^{-1/2} \sum_{|S|=d} \prod_{j \in S} \sqrt{2}\sin\theta_j$ (variance-normalised). Exact-$\varphi$ MLP only.
\textbf{residual interaction order}: $d \in \{4, 6, 8\}$, $M \in \{4, 5, 6\}$, interaction-order sweep $q \in \{1, \ldots, d\}$. Target $g_{d, q}(\theta) = \binom{d}{q}^{-1/2} \sum_{|S|=q} \prod_{j \in S} \sqrt{2}\sin\theta_j$. Exact-$\varphi$ MLP only; the feature map is held correct, so degradation as $q \!\to\! d$ at low $M$ is attributable to model-class capacity rather than representation.
\textbf{noise robustness}: $d \in \{4, 6, 8\}$ at $N_{\mathrm{train}} = 4{,}096$ ($M_{\mathrm{eff}}=N^{1/d}=8,4,\sqrt{8}\approx2.83$, respectively); target $g_{\mathrm{int}}$ as in the common setup. Training labels are corrupted with $\eta \sim \mathcal{N}(0, \sigma^2)$ per the \Cref{sec:exp31} convention; ID and OOD targets remain clean (oracle reference). $\sigma$ sweeps $\{0, 10^{-2}, 10^{-1}, 3{\cdot}10^{-1}, 1, 3, 10, 30\}$, spanning the estimated crossover described below. All three baselines (raw, learned-Fourier, exact-$\varphi$) are evaluated on one clean ID grid per cell.
The regime boundary is $\sigma_\star = \sqrt{n}\,\eps_W/\sqrt{2\pi}$, the noise level at which the linear form of \Cref{prop:erm} (\Cref{app:proof_prop_erm}) stops constraining any test. Its plug-in input $\eps_W$ is the empirical in-window $L^2$ distance between the process the raw commitment settles on and the one the exact-$\varphi$ commitment settles on, and each cell records enough to bracket this empirical distance exactly: with $\Delta$ the raw-versus-exact-$\varphi$ prediction gap on that grid, Jensen and the triangle inequality give
\begin{equation*}
  \max\bigl(\overline{|\Delta|},\; |\mathrm{RMSE}_{\mathrm{raw}} - \mathrm{RMSE}_{\varphi}|\bigr)
  \;\le\; \eps_W \;\le\; \mathrm{RMSE}_{\mathrm{raw}} + \mathrm{RMSE}_{\varphi}.
\end{equation*}
Over the $300$ cells with $0 < \sigma \le 0.05$, where the gap measures the commitment rather than the noise the two models are fitting, the median bounds give $\eps_W \in [0.028, 0.062]$ and hence $\sigma_\star \in [0.71, 1.57]$: this plug-in comparison places $\sigma \le 0.3$ below the crossover, $\sigma \ge 3$ above it, and $\sigma = 1$ inside it. The proposition concerns fixed candidates and a population gap; these data-dependent fitted models and empirical gaps do not establish a testing impossibility.

\FloatBarrier
\paragraph{Compute resources.}
\Cref{tab:resources} reports the per-experiment run count, median wall time per run, device, peak GPU memory, and total wall time, extracted directly from the saved JSON results.
The Tier-1 toy experiments (Exp.~1.1--1.3) and the cross-species CDS sweep (Exp.~2.3) run on a single laptop GPU (RTX 3000 Ada, $8$\,GB) within minutes; this includes the foundation-model conditions (TabPFN, TimesFM) of Exp.~1.1--1.2, the SINDy battery of Exp.~1.3 (CPU-only, NumPy/pysindy), and the $4$~PE/head variants $\times\,6$~organisms $\times\,3$~seeds CDS grid at peak ${\sim}0.45$\,GB. The near-boundary selection-protocol study (Exp.~3.1) is also CPU-only and runs in seconds. The remaining experiments, including the MAK vector field, the torus benchmarks, and the architecture-invariance primary block plus its sweeps, were submitted to a shared HPC cluster (RTX 4090, $24$\,GB), where the architecture-invariance experiment accounts for the bulk of the compute budget (peak ${\sim}22$\,GB on sweeps~A--D).

\begin{table}[ht]
  \centering
  \footnotesize
  \caption{Compute resources per experiment, split by venue: laptop (RTX 3000 Ada, 8\,GB) vs.\ HPC cluster (RTX 4090 24\,GB GPU or CPU node). Wall time, device, and peak GPU memory (max across runs of \texttt{torch.cuda.max\_memory\_allocated}) are read from each run's saved JSON; rows that never moved a tensor onto the GPU are reported as CPU even when CUDA was available. The Exp.~3.3 headline grid (\Cref{tab:exp33_main}) is the primary block plus Sweep~L (rows~1 and~2 of the Exp.~3.3 block below); their run counts sum to the $264$-run figure cited in the main text. Sweeps A--D are independent robustness checks (\Cref{app:exp33_sweeps}) and are not part of the headline.}
  \label{tab:resources}
  \resizebox{\textwidth}{!}{%
  \begin{tabular}{l r r l r r}
    \toprule
    \textbf{Experiment} & \textbf{\#runs} & \textbf{wall/run (med.)} & \textbf{Device} & \textbf{Peak GPU mem} & \textbf{Wall (total)} \\
    \midrule
    \multicolumn{6}{l}{\textit{Local laptop (RTX 3000 Ada Generation Laptop GPU)}} \\
    \cmidrule(l){1-6}
    Exp.~1.1: $\sin(x)$ vs.\ Taylor & 3 & 8 s & RTX 3000 Ada Laptop & 25\,MB & 23 s \\
    Exp.~1.2: power law vs.\ exponential & 3 & 36 s & RTX 3000 Ada Laptop & 28\,MB & 2.0 min \\
    Exp.~1.3: Correctness and SINDy battery & 3 & $<$1 s & CPU & -- & $<$1 s \\
    Exp.~2.2: Kepler multivariate & 3 & 51 s & RTX 3000 Ada Laptop & 57\,MB & 3.2 min \\
    Exp.~2.3 (CDS): cross-species CDS detection & 12 & 2.2 min & RTX 3000 Ada Laptop & 446\,MB & 26.9 min \\
    Exp.~3.1: near-boundary val.\ vs.\ CV & 3 & $<$1 s & CPU & -- & $<$1 s \\
    \cmidrule(l){1-6}
    \textit{Subtotal} & 27 &  &  &  & 32.5 min \\
    \midrule
    \multicolumn{6}{l}{\textit{Shared HPC cluster (RTX 4090 / CPU)}} \\
    \cmidrule(l){1-6}
    Exp.~2.1: mass-action kinetics ODE / vector field & 3 & 7 s & RTX 4090 & 21\,MB & 19 s \\
    Exp.~3.3 Architectural invariance: PE $\times$ architecture, primary block & 156 & 4.5 min & RTX 4090 & 11.2\,GB & 1.1 d \\
    \quad\,$\hookrightarrow$ Sweep~L: period-length sweep, learned PEs & 108 & 14.4 min & RTX 4090 & 11.2\,GB & 2.2 d \\
    \quad\,$\hookrightarrow$ sweeps A--D: $L_\mathrm{train}$ / $L_\mathrm{OOD}$ / $d_\mathrm{state}$ / PE-init & 317 & 27.9 min & RTX 4090 & 22.3\,GB & 13.3 d \\
    Exp.~4.1 (Tor-Repr): representation alignment & 100 & 1.0 min & RTX 4090 & 62\,MB & 1.7 h \\
    Exp.~4.2 (Tor-Cov): critical coverage scaling & 100 & 1.1 min & RTX 4090 & 542\,MB & 1.8 h \\
    Exp.~4.3 (Tor-Cap): capacity at varied coverage & 100 & 4.8 min & RTX 4090 & 542\,MB & 8.1 h \\
    Exp.~4.4 (Tor-Noise): noise robustness across $\sigma$ & 100 & 4.1 min & RTX 4090 & 61\,MB & 6.9 h \\
    \cmidrule(l){1-6}
    \textit{Subtotal} & 984 &  &  &  & 17.3 d \\
    \midrule
    \textbf{Total} & 1\,011 &  &  &  & 17.4 d \\
    \bottomrule
  \end{tabular}%
  }
\end{table}

\FloatBarrier

\end{document}